\documentclass{article}

\usepackage{amsfonts,amsmath,amssymb,amsthm,bm}
\usepackage{algorithm}
\usepackage{algorithmic}
\usepackage{mathrsfs}
\usepackage[colorlinks,
            linkcolor=red,
            anchorcolor=blue,
            citecolor=blue
            ]{hyperref}
\usepackage{natbib}
\usepackage{tabularx}
\usepackage{multirow}
\usepackage{threeparttable}
\usepackage{makecell}
\usepackage{booktabs} 
\usepackage{fullpage}
\usepackage{appendix}
\usepackage{authblk}

\newtheorem{lemma}{Lemma}
\newtheorem{theorem}{Theorem}
\newtheorem{proposition}{Proposition}
\newtheorem{definition}{Definition}
\newtheorem{example}{Example}
\newtheorem{corollary}[theorem]{Corollary}
\newtheorem{remark}{Remark}
\newtheorem{condition}{Condition}
\ifx\assumption\undefined
\newtheorem{assumption}{Assumption}
\fi

\newcommand{\rR}{{\mathbb{R}}}
\newcommand{\E}{{\mathbb{E}}}

\renewcommand\bm[1]{{#1}}
\renewcommand\parallel{\|}

\usepackage{xcolor}         % colors
\newcommand\T{T} % transpose
\newcommand\truth{\theta} % true coefficients
\newcommand\est{\hat\theta} % arbitrary estimator
\newcommand\ridge{\hat\theta_\lambda} % ridge estimator
\newcommand\ols{\hat\theta_0} % least norm estimator
\newcommand\srisk{\mathcal{R}^{\operatorname{std}}} % standard risk
\newcommand\arisk{\mathcal{R}^{\operatorname{adv}}_{\budget}} % adversarial risk
\newcommand\budget{\alpha} % adversarial budget
\newcommand\sriskb{\mathcal{B}^{\operatorname{std}}}
\newcommand\sriskv{\mathcal{V}^{\operatorname{std}}}
\newcommand\normb{\mathcal{B}^{\operatorname{norm}}}
\newcommand\normv{\mathcal{V}^{\operatorname{norm}}}

\newcommand{\trace}{\operatorname{tr}}
\newcommand{\fracl}[2]{{#1}/{#2}}

\newcommand{\inprobto}{\overset{\text{pr.}}{\to}}

\title{The Surprising Harmfulness of Benign Overfitting for\\ Adversarial Robustness
}

\author{Yifan Hao\thanks{The Hong Kong University of Science and Technology. Email: yhaoah@connect.ust.hk} \quad \quad \quad \quad Tong Zhang\thanks{University of Illinois Urbana-Champaign. Email: tongzhang@tongzhang-ml.org}}

\date{}

\begin{document}

\maketitle

\begin{abstract}
Recent empirical and theoretical studies have established the generalization capabilities of large machine learning models that are trained to (approximately or exactly) fit noisy data.
In this work, we prove a surprising result that even if the ground truth itself is robust to adversarial examples, and  the benignly overfitted model  is benign in terms of the ``standard'' out-of-sample risk objective, this benign overfitting process can be harmful when out-of-sample data are subject to adversarial manipulation.
More specifically, our main results contain two parts: (i) the min-norm estimator in overparameterized linear model always leads to adversarial vulnerability in the ``benign overfitting'' setting; (ii) we verify an asymptotic trade-off result between the standard risk and the ``adversarial'' risk of every ridge regression estimator, implying that under suitable conditions these two items cannot both be small at the same time by any single choice of the ridge regularization parameter. Furthermore, under the lazy training regime, we demonstrate parallel results on two-layer neural tangent kernel (NTK) model, which align with empirical observations in deep neural networks.  
Our finding provides theoretical insights into the puzzling phenomenon observed in practice, where the true target function (e.g., human) is robust against adverasrial attack, while beginly overfitted neural networks lead to models that are not robust. 
\end{abstract}

\section{Introduction}

The ``benign overfitting'' phenomenon~\citep{bartlett2020benign} refers to the ability of large (and typically ``overparameterized'') machine learning models to achieve near-optimal prediction performance despite being trained to exactly, or almost exactly, fit noisy training data.
Its key ingredients include the inductive biases of the fitting method, such as the least norm bias in linear regression, as well as favorable data properties that are compatible with the inductive bias.
When these pieces are in place, ``overfitted'' models have high out-of-sample accuracy, which runs counter to the conventional advice that cautions against exactly fitting training data and instead recommends the use of regularization to balance training error and model complexity. These estimators without any regularization have found widespread application in real-world scenarios and garnered considerable attentions owing to their surprising generalization performance~\citep{zhang2016understanding, belkin2019reconciling,bartlett2020benign, shamir2022implicit}. Besides generalization performance, another much anticipated feature of machine learning models is the adversarial robustness. Some recent works~\citep{raghunathan2019adversarial, rice2020overfitting, huang2021exploring, wu2021wider} empirically verified that an increased model capacity deteriorates the robustness of neural networks. However, corresponding theoretical understandings are still lacking. 
%However, some recent works~\citep{gao2019convergence, raghunathan2020understanding, huang2021exploring, wu2021wider, hassani2022curse} suggested that overparameterization could hurt robustness, they verified an increased model capacity (i.e, wider width or deeper depth) deteriorates the robustness of neural networks both theoretically and empirically.

For standard risk, \citet{belkin2019reconciling} illustrated the advantages of improving generalization performance by incorporating more parameters into the prediction model, and \citet{bartlett2020benign} verified the consistency of the ``ridgeless'' estimator in ``benign overfitting'' phase. In this work, we continue our exploration in the same setting, and reveal a surprising finding:  ``benign overfitting'' estimators may become {\em overly sensitive to adversarial attacks~\citep{szegedy2013intriguing, goodfellow2014explaining} even when the ground truth target is robust to such attacks}. This result is unexpected, especially in light of the adversarial robustness of the ground truth target and the established consistency of the generalization performance in \citet{bartlett2020benign}, along with seemingly conflicting finding from earlier studies~\citep{bubeck2021law, bubeck2023universal}, which would have led to the conjecture that overparameterization with benign overfitting could also benefit adversarial robustness. 
This work disproves this seemingly natural conjecture from~\citep{bubeck2021law, bubeck2023universal}  by characterizing the precise impact of data noise on adversarial vulnerability through two performance metrics of an estimator:
one is the standard risk---the difference between mean squared error of the predictor and that of the conditional mean function;
the other is the adversarial risk---which is the same as the standard excess risk, except the input to the predictor is perturbed by an adversary so as to maximize the squared error.
In this paper, we limit the power of the adversary by constraining the perturbation to be bounded in $\ell_2$ norm.

We take explorations in a canonical linear regression context and a two-layer neural tangent kernel (NTK) framework~\citep{jacot2018neural}.
In the linear regression setup, the ``ridgeless'' regression estimator will have vanishing standard risk as sample size $n$ grows if overfitting is benign (in the sense of \citet{bartlett2020benign}).
Furthermore, we investigate ridge regression, which can be regarded as a variant of adversarial training in the benign overfitting setting. In previous studies, it is not clear how these estimators behave in terms of the adversarial risk.
In Section \ref{sec:linear_result}, we focus on the adversarial robustness for this setting, and tackle a general regime in which adversarial vulnerability is an inevitable by-product of overfitting the noisy data, even if the ground truth model has a bounded Lipschitz norm and is robust to adversarial attacks. In addition, we extend our result to the neural tangent kernel (NTK) \citep{jacot2018neural} regime in Section \ref{sec:ntk}, and it is consistent with the empirical results which reveals that ``benign overfitting'' and ``double-descent'' phenomena \citep{belkin2019reconciling, nakkiran2021deep} coexist with the vulnerability of neural networks to adversarial perturbations \citep{biggio2013evasion, szegedy2013intriguing}.

\paragraph{Main contribution.} Our main results could be summarized below.
\begin{itemize}
\item For linear model in the benign overfitting regime, as the sample size $n$ grows, the adversarial risk of the ridgeless estimator ($\lambda\to0$) diverges to infinity, even when the ground truth model is robust to adversarial attacks and the consistency in generalization performance, indicated by the convergence of standard risk to zero, is affirmed.
Furthermore, the same conclusions hold for well-bahaved gradient descent solution in the neural tangent kernel (NTK) regime.
\item For linear model in the benign overfitting regime, there is a trade-off between the standard risk and the  adversarial risk of every ridge regression estimator, in that the standard risk and the adversarial risk cannot be simultaneously small  with any choice of regularization parameter $\lambda$. Consequently, employing ridge regression does not offer a resolution to this trade-off.
\end{itemize}
At the technical level, our analysis involves the study of non-asymptotic standard risk and adversarial risk, and it provides the following insights.
\begin{itemize}
\item Our characterization on adversarial risk captures the effect of the data noise: it turns out that the benign overfitting of noise induces an exploded adversarial risk.
\item The analysis relies on a more refined lower bound technique than that of \citet{tsigler2020benign}, which makes the impact of $\lambda$ on standard risk and adversarial risk more explicit. %To the best of our knowledge, this is the first time that attention is raised over such an impact non-asymptotically, which we deem relevant to the trade-off phenomena.
\end{itemize}

\section{Related works}

Our paper draws on, and contributes to, the literature on implicit bias, benign overfitting and adversarial robustness.
We review the most relevant works below.

\paragraph{Implicit bias.}

The ability of large overparameterized models to generalize despite fitting noisy data has been empirically observed in many prior works~\citep{neyshabur2015search,zhang2016understanding,wyner2017explaining,belkin2018understand,belkin2019reconciling,liang2020just}.
As mentioned above, this is made possible by the implicit bias of optimization algorithms (and other fitting procedures) towards solutions that have favorable generalization properties; such implicit biases are well-documented and studied in the literature~\citep[e.g.,][]{telgarsky2013margins,neyshabur2015path,keskar2016large,neyshabur2017exploring,wilson2017marginal}.

\paragraph{Benign overfitting.}

When these implicit biases are accounted for, very sharp analyses of interpolating models can be obtained in these so-called benign overfitting regimes for regression problems~ \citep{bartlett2020benign,belkin2020two,muthukumar2020harmless,liang2020just,hastie2022surprises,shamir2022implicit,tsigler2020benign, simon2023more}. Our work partly builds on the setup and analyses of \citet{bartlett2020benign} and \citet{tsigler2020benign}. Another line of work focuses on the analysis of benign overfitting on classification problems~\citep{chatterji2021finite, muthukumar2021classification, wang2022binary, wang2023benign}.
However, these and other previous works do not make an explicit connection to the adversarial robustness of the interpolating models in the benign overfitting regime.

\paragraph{Adversarial robustness.}
The detrimental sensitivity of machine learning models to small but adversarially chosen input perturbations has been observed by \citet{dalvi2004adversarial} in linear classifiers and also by \citet{szegedy2013intriguing} in deep networks.
Many works have posited explanations for the susceptibility of deep networks to such ``adversarial attacks''~\citep{shafahi2018adversarial,schmidt2018adversarially,ilyas2019adversarial, gao2019convergence, dan2020sharp, sanyal2020benign, hassani2022curse} without delving into their near-optimal generalization performance, and many alternative training objectives have been proposed to guard against such attacks~\citep{madry2017towards,wang2019improving,zhang2019theoretically, lai2020adversarial, zou2021provable}. Another line of research~\citep{bubeck2021law, bubeck2023universal} proposed that overparameterization is needed for enhancing the adversarial robustness of neural networks, however, their works do not conclusively demonstrate its effectiveness. 
Even in a complementary but related work of \citet{chen2023benign}, the authors demonstrated that benign overfitting  can occur in adversarially robust linear classifiers when the data noise level is low, it has been widely observed that robustness to adversarial attacks may come at the cost of predictive accuracy~\citep{madry2017towards, raghunathan2019adversarial, rice2020overfitting, huang2021exploring, wu2021wider} in many practical datasets.

Recently, some works also focus on studying the trade-off between adversarial robustness and generalization on overparameterized models. For linear classification problems, \citet{tsipras2018robustness} attempted to verified the inevitability of this trade-off, but any classifier that can separate their data is not robust. 
\citet{dobriban2023provable} highlighted the influence of data class imbalance on the trade-off, however, their ground truth model itself is not robust and it is not unexpected to obtain a non-robust estimator; this issue limits insights into the influence of the overfitting process on adversarial vulnerability, but our work addresses this limitation by utilizing a robust ground truth model.
In the domain of linear regression problems, \citet{javanmard2020precise} characterized an asymptotic trade-off between standard risk and adversarial risk, yet their ground truth is also not robust, and the adversarial effect of estimators is mild, matching the effect of general Lipschitz-bounded target functions, thus falling short of revealing the substantial vulnerability of overfitted estimators. In comparison, our work presents a significant adversarial vulnerability, with an exploded adversarial risk corresponding to  unbounded Lipschitz functions, even though the true target function itself has bounded Lipschitz condition. 
We show that the reason this surprising pheneomon  can happen is due to the overfitting of noise. 
\citet{donhauser2021interpolation} also characterizes the precise asymptotic behavior of the adversarial risk under isotropic normal designs on both regression setting and 
 classification setting, but the adversarial effect in their work is similarly mild and matches that of the target function. 

In summary, our main results differ from previous works in that (i) we consider the case where the ground truth model itself is robust to adversarial attacks, and it is highly unexpected that benign overfitting exhibits significant vulnerability to adversarial examples, leading to exploded adversarial risk corresponding to non-robust targets. This is especially surprising since results of~\citet{bubeck2021law} and \citet{bubeck2023universal} would have suggested that overparameterization could be helpful when target itself is robust;
(ii) in comparison to previous results on regression problems~\citep{javanmard2020precise, donhauser2021interpolation}, we present more precise non-asymptotic analyses on non-isotropic designs, with both upper and lower bounds; (iii) we also investigate the Neural Tangent Kernel (NTK) regime.
Our finding can better explain the puzzling phenomenon observed in practice, where human (true target) is robust, while beginly overfitted neural networks still lead to models that are not robust under adversarial attack.

\section{Preliminaries}

\textbf{Notation.} For any matrix $A$, we use $\| A \|_2$ to denote its $\mathcal{L}_2$ operator norm, use $\text{tr}\{ A \}$ to denote its trace, and use $\| A \|_F$ to denote its Frobenius norm. The $j-$th row of $A$ is denoted as $A_{j \cdot}$, and the $j-$th column of $A$ is denoted as $A_{\cdot j}$. The $i-$th largest eigenvalue of $A$ is denoted as $\mu_i(A)$. The transposed matrix of $A$ is denoted as $A^T$. And the inverse matrix of $A$ is denoted as $A^{-1}$. The notation $a = o(b)$ means that $a/b \to 0$; similarly, $a = \omega(b)$ means that $a/b \to \infty$. For a sequence of random variables $\{ v_s \}$, $v_s = o_p(1)$ refers to $v_s \inprobto 0$ as $s \to \infty$, and the notation $\gamma_s v_s = o_p(1)$ is equivalent to $v_s = o_p(1/\gamma_s)$; $v_s = O_p(1)$ refers to $\lim_{M \to \infty} \sup_s \mathbb{P}(| v_s | \ge M) = 0$, similarly, $\gamma_s v_s = O_p(1)$ is equivalent to $v_s = O_p(1/\gamma_s)$.

\subsection{Ridge regression in linear model}

We study a regression setting where $n$ i.i.d.~training examples $(x_1,y_1),\dotsc,(x_n,y_n)$ take values in $\rR^p \times \rR$ and obey the following linear model with parameter $\truth \in \rR^p$:
\begin{equation}
  \E[y_i \mid x_i] = x_i^\T \truth
  \label{eq:linear_model}
\end{equation}
We consider the \emph{ridge regression estimator} $\ridge$  of $\truth$ with regularization parameter $\lambda \geq 0$:
\begin{equation}\label{eq:estpara}
  \ridge
  := X^\T (XX^\T + n\lambda I)^\dag y
\end{equation}
where $X = [x_1,\dotsc,x_n]^\T \in \rR^{n \times p}$ and $y = [y_1,\dotsc,y_n]^\T \in \rR^n$.
The symbol $^\dag$ denotes the Moore-Penrose pseudoinverse, so $\ridge$ is well-defined even for $\lambda=0$ (giving the ``ridgeless'' estimator).

\subsection{Performance measures}

In our setting, the standard performance measure for an estimator $\est$ is the excess mean squared error:
\begin{equation*}
  \E_{(x_\star,y_\star)}\bigl[(x_\star^\T \est - y_\star)^2\bigr] - \E_{(x_\star,y_\star)}\bigl[(x_\star^\T \truth - y_\star)^2\bigr]
  =
  \E_{x_\star}\bigl[(x_\star^\T \est - x_\star^\T \truth)^2\bigr]
  ,
\end{equation*}
where the expectation is taken with respect to $(x_\star,y_\star)$, an independent copy of $(x_1,y_1)$.
Following \citet{tsigler2020benign}, we consider an average case performance measure in which the excess mean squared error is averaged over the choice of the finite $\ell_2$ norm parameter $\truth$ according to a symmetrical distribution, independent of the training examples:
\begin{equation}\label{eq:std}
  \srisk(\est) :=
  \E_{\theta}\left[
    \E_y\Bigl[
      \E_{x_\star}\bigl[(x_\star^\T \est - x_\star^\T \truth)^2\bigr]
    \Bigr]
  \right]
  .
\end{equation}
(We also take expectation with respect to the labels $y$ in the training data.)
We refer to $\srisk(\est)$ as the \emph{standard risk} of $\est$.

The \emph{adversarial risk} of $\est$ is defined by
\begin{equation}\label{eq:adv}
  \arisk(\est) :=
  \E_{\theta}\left[
    \E_y\Biggl[
      \E_{x_\star}\biggl[\sup_{\|\delta\|_2 \leq \budget}((x_\star + \delta)^\T \est - x_\star^\T \truth)^2\biggr]
    \Biggr]
  \right]
  .
\end{equation}
The supremum is taken over vectors $\delta \in \rR^p$ of $\ell_2$-norm at most $\budget$, where $\budget\geq0$ is the \emph{perturbation budget} of the adversary.
(Observe that $\arisk$ with $\budget=0$ is the same as $\srisk$.)

\begin{remark}
  \label{remark:arisk}
  The supremum expression in the definition of $\arisk$ evaluates to
  \begin{equation*}
    \budget^2 \|\est\|_2^2 + 2\budget \|\est\|_2 |x_\star^\T (\est - \truth)| + (x_\star^\T(\est - \truth))^2 ,
  \end{equation*}
  which is bounded above by
  \begin{equation*}
    2\bigl( \budget^2 \|\est\|_2^2 + (x_\star^\T(\est - \truth))^2 \bigr) ,
  \end{equation*}
  (by the inequality of arithmetic and geometric means), and bounded below by
  \begin{equation*}
    \budget^2 \|\est\|_2^2 + (x_\star^\T(\est - \truth))^2 .
  \end{equation*}
  These inequalities imply the following relationship between $\arisk$ and $\srisk$:
  \begin{equation}
    \budget^2 \E_{\theta,y} \|\est\|_2^2 + \srisk(\est)
    \leq \arisk(\est) \leq
    2\bigl( \budget^2 \E_{\theta,y} \|\est\|_2^2 + \srisk(\est) \bigr)
    ,
    \label{eq:arisk_bound}
  \end{equation}
  where the expectation $\E_{\truth,y}$ is taken over the randomness in true parameters $\theta$ and training labels $y$. This can be seen as motivation for the ridge regression estimator $\ridge$ (with appropriately chosen $\lambda$) when $\arisk$ is the primary performance measure of interest.
\end{remark}

\subsection{Data assumptions and effective ranks}\label{set_pre}

We adopt the following data assumptions from \citet{bartlett2020benign} on the distribution of each of the i.i.d.~training examples $(x_1,y_1),\dotsc,(x_n,y_n)$:
\begin{enumerate}
  \item $x_i = V\Lambda^{1/2} \eta_i$, where $V\Lambda V^\T = \sum_{i\geq1} \lambda_i v_i v_i^\T$ is the spectral decomposition of $\varSigma := \E[x_ix_i^\T]$, and the components of $\eta_i$ are independent $\sigma_x$-subgaussian random variables with mean zero and unit variance;
  \item $\E[y_i \mid x_i] = x_i^\T \truth$ (as already stated in \eqref{eq:linear_model});
  \item $\E[(y_i - x_i^\T\truth)^2 \mid x_i] = \sigma^2 > 0$;
  \item $0 < \| \truth \|_2 < \infty$ and flips signs of all its coordinates with probability $0.5$ independently.
\end{enumerate}
The second moment matrix $\varSigma$ is permitted to depend on $n$.
Without loss of generality, assume $\lambda_1 \geq \lambda_2 \geq \dotsb > 0$.
Define the following \emph{effective ranks} for each nonnegative integer $k$:
\begin{equation}\label{eq:rank}
  r_k := \frac{\sum_{i>k} \lambda_i}{\lambda_{k+1}} , \qquad
  R_k := \frac{\bigl( \sum_{i>k} \lambda_i \bigr)^2}{\sum_{i>k} \lambda_i^2} ,
\end{equation}
as well as the \emph{critical index} $k^*(b)$ for a given $b>0$:
\begin{equation}\label{eq:defk}
  k^*(b) := \inf \{ k \geq 0 : r_k \geq bn \} .
\end{equation}
Note that each of $r_k, R_k, k^*(b)$ depends (implicitly) on $\varSigma$ and hence also may depend on $n$.

\section{Main results for linear model}\label{sec:linear_result}

Our main results are two theorems: Theorem~\ref{thm:main_explode} verifies the poor adversarial robustness of min-norm estimator in benign overfitting regime, under two mild conditions; Theorem~\ref{thm:main_tradeoff} characterize a lower bound on a particular trade-off between $\arisk$ and the convergence rate on $\srisk$ for $\ridge$, with an additional condition. Here we introduce them respectively.

\subsection{Adversarial vulnerability of min-norm estimator}
The following condition ensures ``benign overfitting'' in the sense of \citet{bartlett2020benign,tsigler2020benign}.
\begin{condition}[Benign overfitting condition]
  \label{cond:benign}
  There exists a constant $b>0$ such that, for $k^* := k^*(b)$,
  \begin{equation*}
    \lim_{n\to\infty} \| \truth_{k^{*}:\infty} \|^2_{\Sigma_{k^* : \infty}}
    = \lim_{n\to\infty} \biggl( \frac{\lambda_{k^*+1} r_{k^*}}{n} \biggr)^2 \cdot \| \truth_{0:k^*} \|^2_{\Sigma_{0:k^*}^{-1}}
    = \lim_{n\to\infty} \frac{k^*}{n}
    = \lim_{n\to\infty} \frac{n}{R_{k^*}}
    = 0 .
  \end{equation*}
\end{condition}
\citet{tsigler2020benign} showed that under Condition~\ref{cond:benign}, we have $\srisk(\ols) \inprobto 0$ in probability.

Our first main result can be informally stated below, which is a direct consequence of Theorem~\ref{thm_ridge} and Corollary~\ref{cor_1}, means that while the noise is not sufficiently small, an exploded adversarial risk will be induced on min-norm estimator and ridge estimator with small regularization parameter $\lambda$.
\begin{theorem}\label{thm:main_explode}
Assume Conditions~\ref{cond:benign} holds with constant $b>0$ and data noise $\sigma^2 = \omega(\lambda_{k^*+1} r_{k^*} / n)$. For min-norm estimator ( regularization parameter $\lambda = 0$), and budget $\budget > 0$, we have
\begin{equation*}
\srisk (\ridge |_{\lambda = 0}) \inprobto 0, \quad \frac{\arisk(\ridge |_{\lambda = 0})}{\alpha^2} \inprobto \infty, 
\end{equation*}
as $n \to \infty$.
\end{theorem}
The result above shows that even leading to a near-optimal standard risk, the min-norm estimator always implies an exploded adversarial risk, which means its non-robustness to adversarial attacks. 

\begin{remark}
The constrain on noise level reveals the pivotal factor in triggering an exploded adversarial risk is the presence of data noise. \cite{shamir2022implicit} has proposed that in benign overfitting regime, the ``tail features'' are orthogonal to each other, which is essential to the near-zero standard risk. However, adversarial risk always tends to find the ``worst'' disturbation direction given any observer $x$, which would hurt the orthogonality among these ``tail features''. Then while the noise is not small, overfitting on training data will cause a sufficient large Lipschitz norm on the estimator, as well as a large adversarial risk. 
\end{remark}

\begin{remark}
%We could also consider another aspect:
If the model has zero noise, and
we overfit the training data, then the resulting estimator is a projection of the true parameter on a subspace, which spans on the training observations. Therefore the resulting estimate is always robust to adversarial attacks. This means that the adversarial non-robustness is due to the overfitting of noise. 
\end{remark}

\subsection{Trade-off phenomena in ridge regression}
Before investigating the behaviors of other estimators, we need to introduce the third condition, which relies on a new definition \emph{cross effective rank}. The new definition is similar to \emph{effective rank} in \eqref{eq:rank} and characterizes the decrease rate of both eigenvalues $\{ \lambda_i \}_{i \ge 1}$ of $\varSigma$ and parameter weights $\{ \theta_i \}_{i \ge 1} $ on each dimension.
\begin{definition}[Cross Effective Rank]\label{def:cr}
For the covariance matrix $\Sigma = V^T \Lambda V$, denote $\lambda_i = \mu_i(\Sigma)$, and the ground truth parameter as $\truth$. If we have $\| \truth \|_2 < \infty$ and $\sum_i \lambda_i < \infty$, define the cross effective rank as:
\begin{equation*}
    s_k = \frac{\| \truth_{k:\infty} \|^2_{\Sigma_{k:\infty}}   \sum_{i>k}\lambda_i}{ \| \truth \|_2^2 \lambda_{k+1}^2}.
\end{equation*}
\end{definition}
If we denote $\tilde{\truth} = V^T \truth$, the term $\| \theta_{k : \infty} \|^2_{\varSigma_{k : \infty}}$ could be expressed as $\sum_{j > k} \lambda_j \tilde{\theta}_j^2$, then the following condition ensures both the eigenvalues $\{ \lambda_i \}$ and the corresponding parameter weights $\{ \tilde{\truth}_i^2 \}$ do not drop too quickly, as well as constraining the signal-to-noise ratio not too small, i.e,

\begin{condition}[Trade-off condition]\label{cond:trade}The condition consists of three parts:
\begin{enumerate}
\item (slow decay rate in parameter norm) Considering the cross effective rank $s_k$ in Definition \ref{def:cr}, with the definition of $k^* \doteq k^*(b)$, define
\begin{equation}\label{eq:cross_rk}
    w^* = \inf\{ w \ge 0: s_w \ge n \sqrt{\max\{ k^* / n , n / R_{k^*}\}}  \},
\end{equation}
we have $w^* < k^*$.
\item (slow decay rate in parameter norm) For any index $1 \le i \le k^*$ satisfying that
$\lim_{n \to \infty} \lambda_i / \lambda_{k^*+1} = \infty$, we have
\begin{equation*}
    \lim_{n \to \infty} \frac{\lambda_i^2 \| \truth_{0:i-1} \|_{\Sigma_{0:i-1}^{-1}}^2 + \| \truth_{i-1:\infty} \|^2_{\Sigma_{i-1:\infty}}}{\lambda_{k^*+1}^2 \| \truth_{0:k^*} \|^2_{\Sigma_{0:k^*}} + \| \truth_{k^*:\infty} \|^2_{\Sigma_{k^*:\infty}} } = \infty.
\end{equation*}
\item (appropriate signal-to-noise ratio) The noise should not be too large to cover up the information in observers. To be specific,
\begin{equation*}
    \lim_{n \to \infty} \frac{\sigma^2 \sum_{i > w^*} \lambda_i}{n \lambda_{w^*}^2} = \lim_{n \to \infty} \sum_{i \le w^*} \frac{\sigma^2}{n \lambda_i} = 0,
\end{equation*}
where $w^*$ is defined in \eqref{eq:cross_rk}.
\end{enumerate}
\end{condition}

To show the compatibility of Condition \ref{cond:trade} and standard benign overfitting settings, we verify it on two examples from \citet{bartlett2020benign}.

\begin{example}
Suppose the eigenvalues as
\begin{equation*}
    \lambda_{i} = i^{-(1 + 1 / \sqrt{n})}, i = 1, \dots
\end{equation*}
the parameters and noise level are as
\begin{equation*}
    \tilde{\truth}_i^2 = \frac{1}{i \log^{2}(i+1)}, i = 1, \dots \quad \sigma^2 = \frac{1}{n^{1/4}},
\end{equation*}
it is obvious that the effective rank $k^* = \sqrt{n}$, as well as $R_{k^*} = n^{3/2}$. 

To verify the first item in Condition~\ref{cond:trade}, we suppose there is an index $u$ satisfying 
\begin{equation*}
   u^{(2+2/\sqrt{n})} \sum_{i>u} \frac{1}{i^{2+1/\sqrt{n}} \log^2(i)} \cdot \sum_{i>u} \frac{1}{i^{1+1/\sqrt{n}}} = \frac{u\sqrt{n}}{\log^2(u)} \ge n^{3/4} \Rightarrow \frac{u}{\log^2(u)} = O(n^{1/4}),
\end{equation*}
it implies that $w^* / \log^2(w^*) = O(n^{1/4})$ and $w^* < k^*$. 

For the second item, for any index $e$, if $\lambda_e / \lambda_{k^* + 1} \to \infty$, we have $k^* / e \to \infty$, which implies that 
\begin{equation*}
    \lambda_e^2 \sum_{i=1}^{e-1} \tilde{\truth}_i^2 / \lambda_i + \sum_{i \ge e} \tilde{\truth}_i \lambda_i = O(\frac{1}{e^{1 + 1/\sqrt{n}} \log^2(e)}), 
\end{equation*}
which is far larger than $\lambda_{k^*+1}^2 \sum_{i=1}^{k^*} \tilde{\truth}_i^2 / \lambda_i + \sum_{i \ge k^*} \tilde{\truth}_i \lambda_i = 1/(k^{* 1+ 1/\sqrt{n}} \log^2(k^*))$.

Further, the third item in Condition \ref{cond:trade} can be verified as
\begin{equation*}
    \frac{\sigma^2}{n} \max\{ \frac{\sum_{i>w^*} \lambda_i}{\lambda_{w^*}^2}, \sum_{i=1}^{w^*} \frac{1}{\lambda_i} \} = \frac{1}{n^{5/4}} \max\{ \sqrt{n} w^{* 2 + 1/\sqrt{n}}, w^{* 2 + 1/\sqrt{n}} \} = \frac{w^{* 2 + 1/\sqrt{n}}}{n^{3/4}} \to 0.
\end{equation*} 
\end{example}

\begin{example}
Suppose the eigenvalues as
\begin{equation*}
    \lambda_i = i^{-1}, i = 1, \dots, e^{n^{3/4}},
\end{equation*}
and the parameters and noise level are
\begin{equation*}
    \tilde{\truth}_i^2 = \frac{1}{i \log^3(i)},i = 1, \dots, e^{n^{3/4}}, \quad \sigma^2 = \frac{1}{\log(n)}.
\end{equation*}
the effective rank is $k^* = n^{1/4}$,and $R_{k^*} = n^{3/2}$. 

Then similarly, for the first item, we can consider
\begin{equation*}
    u^2 \sum_{i > u} \frac{1}{i} \sum_{i>u} \frac{1}{i^2 \log^3(i)} = \frac{n^{3/4} u}{\log^3(3)} \ge n \sqrt{ \max\{ k^* / n, \frac{n}{R_{k^*}} \}} = n^{3/4} \Rightarrow u = O(1),
\end{equation*}
which implies that $w^* = O(1)$.

For the second item, considering $\lambda_e / \lambda_{k^*+1} \to \infty$, we have $k^* / e \to \infty$, and
\begin{equation*}
    \lambda_e^2 \sum_{i=1}^{e-1} \tilde{\truth}_i^2 / \lambda_i + \sum_{i \ge e} \tilde{\truth}_i \lambda_i = O(\frac{1}{e \log^3(e)}) \ge O(\frac{1}{k^* \log^3(k^*)}) = \lambda_{k^*+1}^2 \sum_{i=1}^{k^*} \tilde{\truth}_i^2 / \lambda_i + \sum_{i \ge k^*} \tilde{\truth}_i \lambda_i.
\end{equation*}
The third item in Condition \ref{cond:trade} also meets naturally as $w^* = O(1)$:
\begin{equation*}
 \frac{\sigma^2}{n} \max\{ \frac{\sum_{i>w^*} \lambda_i}{\lambda_{w^*}^2}, \sum_{i=1}^{w^*} \frac{1}{\lambda_i} \} = O(\sigma^2 / n) = O(\frac{1}{n \log(n)}) \to 0. 
\end{equation*}
\end{example}
Based on the three conditions above, we could verify the trade-off between standard risk convergence rate and adversarial risk as follows:
\begin{theorem} \label{thm:main_tradeoff}
Assume Conditions~\ref{cond:benign} and~\ref{cond:trade} hold with constant $b>0$ and data noise $\sigma^2 = \omega(\lambda_{k^*+1} r_{k^*} / n)$. For all regularization parameter $\lambda \ge 0$ and budget $\budget > 0$, we have
\begin{equation*}
  \frac{\srisk(\ridge)}{\| \truth \|_2^2 \srisk(\ridge |_{\lambda=0})} + \frac{\arisk(\ridge)}{\alpha^2} \inprobto \infty,
\end{equation*}    
as $n \to \infty$.
\end{theorem}
This result immediately implies the following consequence. As the sample size $n \to \infty$, if benign overfitting occurs with
a well-behaved convergence rate on $\srisk(\ridge)$, to be specific, $\srisk(\ridge) / (\| \truth \|_2^2\srisk (\ridge |_{\lambda = 0}) )\nrightarrow \infty$, 
then $\arisk(\ridge) \to \infty$ (in probability) for any fixed budget $\budget$.
This implies the trade-off phenomena between forecasting accuracy and adversarial robustness.

\begin{remark}
 \citet{chen2023benign} has proposed that benign overfitting phenomenon can occur in adversarially robust linear classifiers while the data noise level is low, which is consistent with our result from a different perspective and reveals the crutial role of noise level ($\sigma^2 = \omega(\lambda_{k^*+1} r_{k^*} / n)$) in the trade-off phenomena.
\end{remark}

\begin{remark}
The trade-off phenomena in Theorem \ref{thm:main_tradeoff} can also be influenced by parameter norm $\| \truth \|_2$ (as is shown in Condition \ref{cond:trade}). To be specific, according to Theorem \ref{thm_ridge}, we would always obtain an exploded adversarial risk while taking small regularization parameter $\lambda$; but if we consider a situation that the parameter norm is small (e.g, $\| \truth \|_2 \to 0$), increasing the regularization parameter $\lambda$ will not cause the standard risk convergence rate worse obviously. It implies that in this specific situation, increasing $\lambda$ will induce a decrease in adversarial risk, and do not hurt the convergence rate on standard risk, so there is no trade-off phenomena.  
\end{remark}
\textbf{Discussion on the trade-off condition.} Condition \ref{cond:trade} indicate a specific function class, in which there exists a trade-off phenomenon between $\srisk$ and $\arisk$.  This function class is general as it does not put any restrictions on specific eigenvalue structures, but just provide some sufficient conditions for the trade-off. If both the eigenvalues of covariance matrix $\Sigma$ and the parameter weights corresponding to each dimension decrease slowly enough, at the same time the signal-to-noise ratio is properly large, we can never find an approximate solution, which leads to both good convergence rate in standard risk and good adversarial robustness at the same time. 

But on the other hand, if either the eigenvalues of $\Sigma$ or the parameter weights decreases rapidly, we can always truncate the high-dimensional data $x \in \mathbb{R}^p$ and just capture the first ``important'' $d$ dimension observed data to predict target variable $y$, for a specific integer $d \ll n \ll p$, then the corresponding estimator will lead to both well-behaved standard risk convergence rate and robustness to adversarial attacks. However, the choice of truncation integer $d$ is not natural, as we do not know enough information about the eigenvalues of covariance matrix $\Sigma$ in general situations. So in the practical training of large machine learning models, the non-truncated estimator is more commonly employed.

\section{Extension to two-layer neural tangent kernel networks}\label{sec:ntk}

In this section, we will extend Theorem \ref{thm_ridge} to general two-layer neural networks in neural tangent kernel (NTK) regime. Given that the benign overfitting of standard risk in NTK regime has been well-verified by \cite{cao2019generalization, adlam2020neural,li2021towards,zhu2023benign}, we treat it as a side result; instead, our primary focus is to illustrate the adversarial vulnerability of gradient descent solution in scenarios characterized by benign overfitting. Given observer $x$ and target $y$, we assume the ground truth model is
\begin{equation}\label{eq:func_tru}
    y = g(x) + \xi,
\end{equation}
where function $g(\cdot) : \mathbb{R} \to \mathbb{R}$ and the noise $\xi$ is independent of $x$. Then we aim to fit $y$ on the two-layer network function $f_{NN}(w, x)$, i.e, 
\begin{equation*}
     f_{NN}(w,x) = \frac{1}{\sqrt{mp}} \sum_{j=1}^m u_j h(\theta_j^T x), \quad [ \theta_j, u_j] \in \mathbb{R}^{p+1} ,
\end{equation*}
in which $w = [\theta_j^T, u_j, j = 1, \dots, m] \in \mathbb{R}^{m(p+1)}$ is the vectorized parameter of $[\Theta, U] \in \mathbb{R}^{m \times (p+1)}$ ($[\Theta, U]_{j \cdot} = [\theta_j^T, u_j]$ for $j = 1, \dots, m$) and the ReLU activation function $h(\cdot)$ is defined as $h(z) = \max\{ 0, z\}$. While training in neural tangent kernel (NTK) regime with a random initial parameter $w_0$, we restate the definition in \cite{cao2019generalization}, which characterizes the small distance between $[\Theta_0, U_0]$ and some parameter $[\Theta, U]$:
\begin{definition}[R-neighborhood]
 For $[\Theta_0, U_0] \in \mathbb{R}^{m \times (p+1)}$, we define its R-neighborhood as
 \begin{equation*}
     \mathcal{B}([\Theta_0, U_0], R) := \left\{ [\Theta, U] \in \mathbb{R}^{m \times (p+1)} : \| [\Theta_0, U_0] - [\Theta, U] \|_F \le R \right\}.     
\end{equation*}
\end{definition}
Then within NTK regime, we can truncate $f_{NN}(w,x)$ on its first order Taylor expansion around initial point $w_0$:
\begin{equation*}
\begin{aligned}
& \quad f_{NTK} (w, x) = f_{NN}(w_0, x) + \nabla_w f_{NN}(w_0, x)^T (w - w_0)\\
&= \frac{1}{\sqrt{mp}} \sum_{j=1}^m u_{0,j} h(\theta_{0,j}^T x) + \frac{1}{\sqrt{mp}} \sum_{j=1}^m \left[ (u_j - u_{0,j}) h(\theta_{0,j}^T x) + u_{0,j} h'(\theta_{0,j}^T x) (\theta_j - \theta_{0,j})^T x \right],
\end{aligned}
\end{equation*}
where $w = [\Theta, U] \in \mathcal{B}([\Theta_0, U_0], R)$ and $R > 0$ is some constant. 
In general training process, we prefer to utilize a small learning rate $\eta$, which will induce a convergence point $\hat{w}$ as step size $t$ large enough:
\begin{proposition}
Initialize $w_0$, and consider running gradient descent on least squares loss, yielding iterates:
\begin{equation*}
    w_{t+1} = w_t - \gamma \frac{1}{n} \sum_{i=1}^n (f_{NTK} (w_t,x_i) - y_i ) \nabla_w f_{NTK} (w_t, x_i), \quad t = 0, 1, \dots
\end{equation*}
Then we can obtain
\begin{equation}\label{eq:ntk_para}
    \lim_{t \to \infty} w_t = \hat{w} = w_0 + \nabla F^T (\nabla F \nabla F^T)^{-1} (y - F),
\end{equation}
where $X = [x_1,\dotsc,x_n]^\T \in \rR^{n \times p}$, $\nabla F = [\nabla_w f_{NN}(w_0,x_1), \dots, \nabla_w f_{NN} (w_0, x_n) ]^T \in \rR^{n \times m(p+1)}$, $F = [f_{NN}(w_0,x_1), \dots, f_{NN} (w_0, x_n)]^T \in \rR^n$ and $y = [y_1,\dotsc,y_n]^\T \in \rR^n$, if the learning rate $\gamma < 1 / \lambda_{\max} (\nabla F^T \nabla F)$.
\end{proposition}
\begin{proof}
The proof is similar to Proposition $1$ in \citet{hastie2022surprises}. As all $w_t - w_0, t = 1, \dots$ lie in the row space of $\nabla F$, the choice of step size guarantees that $w_t - w_0$ converges to a min-norm solution.  
\end{proof}
Similar to the settings in linear model, here we consider the excess standard risk and adversarial risk as:
\begin{align*}
 & \srisk(\hat{w}) := \E_{x,y} \left[( f_{NTK}(\hat{w}, x) - f_{NTK}(w_*, x) )^2 \right],\\
 & \arisk(\hat{w}) := \E_{x,y} \left[ \sup_{\| \delta \|_2 \le \alpha} ( f_{NTK}(\hat{w}, x+ \delta) - f_{NTK}(w_*, x) )^2   \right].
\end{align*}
Notice that even if the kernel matrix $K = \nabla F \nabla F^T$ converges to a fixed kernel in NTK regime \citep{jacot2018neural}, the initial parameters $w_0 = [\Theta_0, U_0]$ are chosen randomly. Here we study on the setting of \citet{jacot2018neural}, where all of the initial parameters are i.i.d. sampled from standard gaussian distribution $\mathcal{N}(0,1)$.
As for observers, we can take the following assumptions on the i.i.d.~training data $(x_1,y_1),\dotsc,(x_n,y_n)$, which are similar to the assumptions in linear model,
\begin{enumerate}
  \item $x_i =  V\Lambda^{1/2} \eta_i$, where $V\Lambda V^\T = \sum_{i\geq1} \lambda_i v_i v_i^\T$ is the spectral decomposition of $\varSigma :=  \E[x_ix_i^\T]$ ($\lambda_1 > 0$ is a constant which does not change with the increase on $n$), and the components of $\eta_i$ are independent $\sigma_x$-subgaussian random variables with mean zero and unit variance;
  \item $\E[y_i \mid x_i] = f_{NTK}(w_*, x_i) $ for some $w_* = [\Theta_*, U_*] \in \mathcal{B}([\Theta_0, U_0], R)$;
  \item $\E[(y_i - f_{NTK}(w_*, x_i))^2 \mid x_i] = \E[\epsilon_i^2 | x_i] = \sigma^2 > 0$, where $\sigma > 0$ is a constant and does not change with the increase on $n$.
\end{enumerate}
The assumption on target $y$ means that we prefer to approximate the ground truth function Eq.\eqref{eq:func_tru} on a function class as:
\begin{equation*}
    \mathcal{F}_{NTK}(w_0) := \{ f_{NN}(w_0, x) + \nabla_w f_{NN}(w_0, x)^T (w - w_0) \mid w = [\Theta, U] \in \mathcal{B}([\Theta_0, U_0], R) \},
\end{equation*}
which may lead to additional error, i.e,
\begin{equation*}
    \sigma^2 = \mathbb{E}[y - f_{NTK}(w_*, x)]^2 \ge \mathbb{E}[y - g(x)]^2 = \mathbb{E} [\xi]^2.
\end{equation*}
Here we still utilize the same definition in Eq.~\eqref{eq:rank} and \eqref{eq:defk} as on linear model, the following two conditions are required in further analysis:
\begin{condition}[benign overfitting condition in NTK regime]\label{cond:ntk_benign}
\begin{align*}
&  \lim_{n \to \infty} \frac{k^*}{n} = \lim_{n \to \infty} \frac{n \sum_{j > k^*} \lambda_j^2}{l^2} = \lim_{n \to \infty} \frac{l^2}{n \sum_{j > k^*} \lambda_j} = 0,
\end{align*}
where we denote $l = r_0(\varSigma) = \sum_{j=1}^{p} \lambda_j$.
\end{condition}
Condition \ref{cond:ntk_benign} is compatible to Condition \ref{cond:benign} in linear models, which characterizes the slow decreasing rate on covariance eigenvalues $\{ \lambda_j \}$.

\begin{condition}[high-dimension condition in NTK regime]\label{cond:ntk_high-dim}
\begin{equation*}
p = o(m^{1/2}),  \quad n = o(l^{4/3}), \quad \max\{ n, l \} = o(p).
\end{equation*} 
\end{condition}
The first condition in Condition \ref{cond:ntk_high-dim} implies
the large number of neurons, which is compatible with NTK setting \citep{jacot2018neural}; the second condition characterizes the large scale of $l = r_0(\varSigma)$, which is consistent with the slow decay rate on eigenvalues $\{ \lambda_j \}$; 
and the third condition induces a high-dimension structure of input data $x$, and relaxing this condition could be left as a further exploration question; Here is also an example from \citet{bartlett2020benign} to verify these two conditions above:
\begin{example}\label{eg:ntk}
Suppose the eigenvalues as
\begin{equation*}
\lambda_k = \left\{
\begin{aligned}
& 1, k = 1, \\
& \frac{1}{n^{6/5}} \frac{1 + s^2 - 2 s \cos(k \pi / (p_n + 1))}{1 + s^2 - 2 s \cos(\pi / (p_n + 1))}, 2 \le k \le p_n, \\
& 0, \text{otherwise},
\end{aligned}
\right.
\end{equation*}
where $p_n = n^2$ and $m_n = e^n$. As it has been verified that $k^* = 1$, we could obtain 
\begin{align*}
& \frac{k^*}{n} = \frac{1}{n} \to 0,\\
& \frac{n \sum_{j > k^*} \lambda_j^2}{l^2} \le \frac{n p_n (1 + s)^4 / n^{12/5}}{(1 + p_n (1 - s)^2 / n^{6/5})^2} \le \frac{2 n}{p_n} = \frac{2}{n} \to 0,\\
& \frac{l^2}{n \sum_{j > k^*} \lambda_j} \le \frac{ (1 + p_n (1 + s)^2 / n^{6/5})^2}{ n p_n (1 - s)^2 / n^{6/5}} \le \frac{2 p_n}{n^{11/5}} = \frac{2}{n^{1/5}} \to 0,\\
& \frac{p_n}{m_n} \to 0, \quad \frac{n}{p_n} = \frac{1}{n} \to 0, \quad \frac{l}{p_n} \le \frac{2 p_n (1 + s)^2 / n^{6/5} }{p_n} \le \frac{4}{n^{6/5}} \to 0, \quad \frac{n^{3/4}}{l} \le  \frac{2 n^{3/4}}{p_n / n^{6/5}} = \frac{2}{n^{1/20}} \to 0,
\end{align*}
which induces that Condition \ref{cond:ntk_benign} and \ref{cond:ntk_high-dim} are both satisfied.
\end{example}
Considering the estimator in Eq.\eqref{eq:ntk_para}, while estimating $\srisk(\hat{w})$ and $\arisk(\hat{w})$ in a high probability phase, we can obtain the following results, which are consistent with the results in linear model:
\begin{theorem}\label{thm_ntk}
For any $b, \sigma_x > 0$, there exist constants $C_{10}, C_{11} > 0$ depending on $b, \sigma_x$, such that the following holds. 
Assume Condition~\ref{cond:ntk_benign} and \ref{cond:ntk_high-dim} are satisfied, there exists a constant $c>1$ such that for $\delta \in (0,1)$ and $\ln(1 / \delta) < n^{1/8} / c$,
with probability at least $1 - \delta$ over $X$ and $w_0$,
\begin{align*}
\srisk (\hat{w}) / C_{10} & \le R^2 \frac{l^{1/2}}{ p ^{1/2} n^{1/4}} +  \sigma^2 \left( \frac{1}{n^{1/8}} + \frac{k^*}{n} + \frac{n \sum_{j > k^*} \lambda_j^2}{l^2} \right),
\end{align*}
and
\begin{align*}
\arisk(\hat{w}) / C_{11}
& \ge \alpha^2 \sigma^2 \frac{n \lambda_{k^*+1} r_{k^*}}{l^2}.
\end{align*}
\end{theorem}
The detailed proof is in Appendix \ref{pf:ntk}. Theorem \ref{thm_ntk} induces that considering a two-layer neural networks wide enough, while input data $x$ is high-dimensional with slow decreasing rate on covariance matrix eigenvalues, gradient descent with a small learning rate will lead to good performance on standard risk, but poor robustness to adversarial attacks. And it is consistent with the results in linear models.  

It is obvious to induce the following corollary:
\begin{corollary}\label{coro:ntk}
Assume Conditions~\ref{cond:ntk_benign} and~\ref{cond:ntk_high-dim} hold with constant $b, \sigma_x > 0$. For the gradient descent solution $\hat{w}$, and budget $\budget > 0$, we have
\begin{equation*}
\srisk (\hat{w}) \inprobto 0, \quad \frac{\arisk(\hat{w})}{\alpha^2} \inprobto \infty, 
\end{equation*}
as $n \to \infty$.
\end{corollary}
\begin{remark}\label{rmk_widerclass}
 We could also consider $f_{NTK}(w_*, x)$ within a ``wider'' function class:
 \begin{align*}
  & \mathcal{F}'_{NTK}(w_0) := \{ f_{NN}(w_0, x) + \nabla_w f_{NN}(w_0, x)^T (w - w_0) \mid w \in \mathcal{C} \}, \\
  & \mathcal{C} := \{ w \in \mathbb{R}^{m(p+1)} \mid \| \mathbb{E} (w_0 - w) (w_0 - w)^T \|_2 \le \epsilon_p \},
 \end{align*}
 where $\epsilon_p = o(1 / p)$. It implies that in a high probability regime, we could just obtain
 \begin{equation*}
     \| w_* - w_0 \|_{\infty} \le \epsilon_p^{1/2},
 \end{equation*}
and there is no any restriction on $\| w_0 - w_* \|_2$, i.e, $\| [\Theta_0, U_0] - [\Theta_*, U_*] \|_F$, which means that $\mathcal{F}'_{NTK}(w_0)$ is a ``wider'' function class comparing with $\mathcal{F}_{NTK}(w_0)$. 

Then within this regime, we could also get the same result as Corollary \ref{coro:ntk}. To be specific, assume Condition~\ref{cond:ntk_benign} and \ref{cond:ntk_high-dim} are satisfied, there exist constants $C_{12}, C_{13} > 0$ depending on $l, \sigma_x$, such that as sample size $n$ increases, for the corresponding gradient descent solution $\hat{w}$ and budget $\alpha > 0$, we have
\begin{align*}
& \srisk (\hat{w}) / C_{12} \le \frac{p \cdot \epsilon_p}{\sqrt{n}} + \sigma^2 \left( \frac{1}{n^{1/8}} + \frac{k^*}{n} + \frac{n \sum_{j > k^*} \lambda_j^2}{l^2} \right) \inprobto 0, \\
& \arisk(\hat{w}) / C_{13} \ge \alpha^2 \sigma^2 \frac{n \lambda_{k^*+1} r_{k^*}}{l^2} \inprobto \infty.
\end{align*}
The detailed proof is in Appendix \ref{pf:widerclass}.
\end{remark}

\section{Outline of the argument for linear model}

In this section, we provide the technical theorems and corollaries on linear model, as well as corresponding proofs. These theorems and corollaries below induce the results in Theorem \ref{thm:main_explode} and Theorem \ref{thm:main_tradeoff}.

\subsection{Technical theorems and corollaries}

First, the following theorem gives an upper bound for standard risk and a lower bound for the expected squared norm of $\ridge$ (conditioned on the empirical input data matrix $X$) for all $\lambda \geq 0$. 
\begin{theorem}\label{thm_ridge}
For any $b > 1$ and $\sigma_x > 0$, there exist constants $C_1, C_2 > 0$ depending only on $b, \sigma_x$, such that the following holds.
Assume Condition~\ref{cond:benign} is satisfied, and set $k^* = k^*(b)$.
There exists a constant $c>1$ such that for $\delta \in (0,1)$ and $\ln(1 / \delta) < n / c$,
for any $\lambda\geq0$,
with probability at least $1 - \delta$ over $X$,
\begin{align*}
\srisk (\ridge) / C_1 & \le  \| \truth_{k^* : \infty} \|^2_{\Sigma_{k^*:\infty}} + \frac{ \lambda_{k^{*} + 1}^2 r_{k^{*}}^2 + n^2 \lambda^2}{n^2} \| \truth_{0:k^*} \|^2_{\Sigma_{0:k^*}^{-1}} \\
& \qquad + \sigma^2 \left(\frac{k^{*}}{n} + \frac{ n \sum_{i > k^{*}} \lambda_i^2}{(\lambda_{k^{*}+1} r_{k^{*}} + n \lambda)^2} \right), 
\end{align*}
and
\begin{align*}
\E_{\truth,y} \parallel \ridge \parallel^2 / C_2
& \ge  \sum_{i \le k^{*}} \left( \frac{\sigma^2}{\lambda_i} + n \tilde{\truth}_i^2 \right) \min \left\{ \frac{1}{n}, \frac{ n  \lambda_i^2}{( \lambda_{k^{*} + 1}r_{k^{*}} + n \lambda)^2}  \right\} \\
& \qquad + \frac{n \sigma^2 \lambda_{k^{*}+1} r_{k^{*}} + n^2 \sum_{j > k^{*}} \tilde{\truth}_j^2 \lambda_j^2}{ \lambda_{k^{*} + 1}^2 r_{k^{*}}^2 + n^2 \lambda^2} .
\end{align*}
\end{theorem}
(The upper bound on $\srisk(\ridge)$ is due to \citet{tsigler2020benign}.\footnote{Notice that in this paper the regularization parameter is scaled with $n$ (see Eq.~(\ref{eq:estpara})). Thus, to obtain comparable results with \citet{tsigler2020benign} one should replace $\lambda$ with $n\lambda$ in that paper.})

The result is useful when we choose the regularization parameter $\lambda$ small enough.
In this case, it implies that the standard risk converges to zero fast as sample size $n$ increases, but the norm of the estimated parameter is large.
Specifically, we have the following corollary.
\begin{corollary}\label{coro_smallreg}
There exist constants $C_3, C_4 > 0$ depending only on $b, \sigma_x$, such that the following holds.
Assume Condition~\ref{cond:benign} is satisfied, and set $k^* = k^*(b)$.
There exists a constant $c>1$ such that for $\delta \in (0,1)$ and $\ln(1 / \delta) < n / c$,
for any $\lambda \le \lambda_{k^{*} + 1} r_{k^{*}} / n$,
with probability at least $1 - \delta$ over $X$,
\begin{align*}
\srisk(\ridge) / C_3 
& \le \| \truth_{k^*:\infty} \|_{\Sigma_{k^*:\infty}}^2 + \left( \frac{ \lambda_{k^{*}+1} r_{k^{*}}}{n} \right)^2 \| \truth_{0:k^*} \|^2_{\Sigma_{0:k^*}^{-1}}
%\\ & \qquad
+ \sigma^2 \left(  \frac{k^{*}}{n} + \frac{ n }{ R_{k^{*}} }\right),\\
\E_{\truth,y} \parallel \ridge \parallel^2 / C_4
& \ge  \sum_{i \le k^{*}} \left( \frac{\sigma^2}{\lambda_i} + n \tilde{\truth}_i^2 \right) \min \left\{ \frac{1}{n}, \frac{ n  \lambda_i^2}{r_{k^{*}}^2 \lambda_{k^{*} + 1}^2 }  \right\}
%\\ & \qquad
+ \frac{n \sigma^2 }{ r_{k^{*}} \lambda_{k^{*}+1} } + \frac{n^2 \sum_{j > k^*} \tilde{\truth}_j^2 \lambda_j^2}{\lambda_{k^*+1}^2 r_{k^*}^2}.
\end{align*}
\end{corollary}
% I think we can omit 'b^2' since we have the C_4 on the left-hand side
%yes, that's right, 

From Corollary~\ref{coro_smallreg}, we can see that the standard risk is near optimal under Condition \ref{cond:benign} with the choice of $\lambda \leq \lambda_{k^*+1}r_{k^*} / n$: as $n \to \infty$, $\srisk(\ridge) \inprobto 0$.
In this sense, overfitting is benign.
However, in this case, the expected squared parameter norm is bounded below by
\[
\frac{n \sigma^2}{\lambda_{k^{*} + 1} r_{k^{*}}},
\]
which grows superlinearly in $n \sigma^2$ (on account of Condition~\ref{cond:benign} and $\sigma^2 = \omega(\lambda_{k^*+1} r_{k^*} / n)$).
The small standard risk and large adversarial risk imply the near optimal estimating accuracy and high vulnerability to adversarial attack of estimators with small $\lambda$. All these analysis in Theorem \ref{thm_ridge} and Corollary \ref{coro_smallreg} finish the proof of Theorem \ref{thm:main_explode}.

One may further ask the question that whether it is possible to use a larger $\lambda$ so that both standard risk and adversarial risk are well behaved. The answer is negative under Condition~\ref{cond:trade}.
Specifically, we can get the following lower bounds for standard risk and parameter norm in Theorem \ref{thm_r2} when $\lambda$ is larger than what is considered in Corollary~\ref{coro_smallreg}.

\begin{theorem}\label{thm_r2}
For any $b > 1, \sigma_x > 0$, there exist $C_5, C_6 > 0$ depending only on $b, \sigma_x$, such that the following holds.
Assume Condition~\ref{cond:benign} is satisfied, and set $k^* = k^*(b)$.
Suppose that $\delta \in (0,1)$ with $\ln(1 / \delta) < n / c$, where $c$ is defined in Theorem~\ref{thm_ridge},
for any $\lambda \ge \lambda_{k^{*}+1} r_{k^{*}} / n$, with probability at least $1 - \delta$ over $X$,
\begin{equation*}
\begin{aligned}
\srisk (\ridge)/ C_5
& \ge \sum_{\lambda_i \ge \lambda} \tilde{\truth}_i^2 \frac{\lambda^2}{\lambda_i} + \sum_{\lambda_i < \lambda} \tilde{\truth}_i^2 \lambda_i  + \frac{\sigma^2}{n} \left( \sum_{\lambda_i \ge \lambda} 1 + \sum_{\lambda_i < \lambda}  \frac{\lambda_i^2}{\lambda^2} \right),\\
\E_{\truth,y} \parallel \ridge \parallel^2  / C_6
& \ge \sum_{\lambda_i \ge \lambda} \tilde{\truth}_i^2 + \sum_{\lambda_i < \lambda} \frac{\tilde{\truth}_i^2 \lambda_i^2}{\lambda^2}  + \frac{\sigma^2}{n} \left( \sum_{\lambda_i \ge \lambda} \frac{1}{\lambda_i} + \sum_{\lambda_i < \lambda} \frac{\lambda_i}{\lambda^2} \right) .
\end{aligned}
\end{equation*}
\end{theorem}
Note that the lower bound on the standard risk can be derived from the results of \citet[Section 7.2]{tsigler2020benign}.

In the following corollary, we explicitly analyze different situations with respect to regularization parameter $\lambda$, which reveals the universal trade-off between estimation accuracy and adversarial robustness when benign overfitting occurs. 

\begin{corollary}\label{cor_1}
For any $b > 1, \sigma_x > 0$ and data noise $\sigma^2 = \omega(\lambda_{k^*+1} r_{k^*} / n)$, there exist constants $C_7, C_8, C_9 > 0$ depending only on $b, \sigma_x$ such that the following holds. 
Set $k^* := k^*(b)$ and suppose that $\delta \in (0,1)$ with $\ln(1 / \delta) \le n / c$, where $c$ is defined in Theorem~\ref{thm_ridge}, assume Condition \ref{cond:benign} holds, then with probability at least $1 - \delta$ over $X$,
\begin{equation*}
\begin{aligned}
& \arisk (\ridge) / C_7 \ge \frac{ n \alpha^2 \sigma^2}{\lambda_{k^{*}+1} r_{k^{*}}} \quad & &\text{if} \quad \lambda \le \frac{\lambda_{k^{*} + 1} r_{k^{*}}}{n},\\
& \srisk (\ridge)  / C_8 \ge  \| \truth \|^2_{\Sigma} \ge \srisk(\ridge |_{\lambda = 0}) \quad & &\text{if} \quad \lambda \ge \lambda_1.
\end{aligned}
\end{equation*}
Moreover, if Condition \ref{cond:benign} and \ref{cond:trade} hold, with probability at least $1 - \delta$ over $X$, we also obtain
\begin{equation*}
\begin{aligned}
& \srisk (\ridge) \ge \| \truth \|_2^2 \srisk(\ridge |_{\lambda = 0}) \quad & &\text{if} \quad \lambda_{w^*} \le \lambda < \lambda_1,\\
& \srisk(\ridge) / (C_9 \| \truth \|_2^2 \srisk(\ridge |_{\lambda=0})) \ge \Delta(\lambda) \quad & &\text{if} \quad  \frac{\lambda_{k^{*} + 1} r_{k^{*}}}{n} < \lambda < \lambda_{w^*},
\end{aligned}
\end{equation*}
in which
\begin{equation*}
    \Delta(\lambda) = \min \left\{ \frac{\lambda^2 \sum_{\lambda_i > \lambda} \tilde{\truth}_i^2 / \lambda_i + \sum_{\lambda_i \le \lambda} \tilde{\truth}_i^2 \lambda_i}{\| \truth \|_2^2 (\lambda_{k^*+1}^2 \| \truth_{0:k^*} \|^2_{\Sigma_{0:k^*}^{-1}} + \| \truth_{k^*:\infty} \|^2_{\Sigma_{k^*:\infty}} )} , \frac{\budget^2}{\arisk(\ridge) \sqrt{\max\{ k^*/n, n/R_{k^*} \}}} \right\}.
\end{equation*}
\end{corollary}

From the results of Corollary~\ref{cor_1}, we observe that with conditions above, no regularization parameter $\lambda \ge 0$ can achieve near optimal $\srisk$ convergence rate and small $\arisk$ at the same time.
A small regularization $\lambda$ will lead to diverging parameter norm,
while a large $\lambda$ will lead to an inferior standard risk.
Even when we choose $\lambda$ in the intermediate regime, either adversarial risk goes to infinity or the standard excess risk does not achieve good convergence rate. With Theorem \ref{thm_r2} and Corollary \ref{cor_1}, we finish the proof of Theorem \ref{thm:main_tradeoff}.

\subsection{Proof sketches for the technical theorems and corollaries}
 
In this part, we sketch the proofs of our main theorems on linear model; detailed proofs are in Appendix \ref{pf:ridge} and \ref{pf:trade}. For simplicity, we use $c_i^{\prime}$ to denote positive constants that only depend on $b, \sigma_x$ (which defines $k^* = k^*(b)$).

Recall the expression for the ridge regression estimate \eqref{eq:linear_model}:
\begin{equation*}
    \ridge = (X^T X + \lambda n I)^{-1} X^T y = X^T (XX^T + \lambda n I)^{-1} (X \truth + \epsilon).
\end{equation*}
We take expectation with respect to the choice of $\truth$ and the labels $y$ in the training data:
\begin{equation}\label{eq_calcu}
\begin{aligned}
 \srisk (\ridge) &= \underbrace{\mathbb{E}_{\truth} \truth^T [I - X^T(XX^T + n \lambda I)^{-1} X]\Sigma [I - X^T (XX^T + n \lambda I)^{-1} X] \truth}_{\sriskb}\\
& \quad + \sigma^2 \underbrace{ \trace\{ X \varSigma X^T (XX^T + \lambda n I)^{-2} \}}_{\sriskv},\\
\E_{\truth,y} \parallel \ridge \parallel_2^2 &= \underbrace{\mathbb{E}_{\truth} \truth^T X^T (XX^T + n \lambda I)^{-1} XX^T (XX^T + n \lambda I)^{-1} X \truth}_{\normb}\\
& \quad + \sigma^2 \underbrace{\trace \{ XX^T (XX^T + \lambda n I)^{-2} \}}_{\normv}.
\end{aligned}
\end{equation}
Then recalling the decomposition $\varSigma = \sum_i \lambda_i v_i v_i^T$, we have
\begin{equation}\label{eq:note1}
    XX^T = \sum_i \lambda_i z_i z_i^T, \quad X \varSigma X^T = \sum_i \lambda_i^2 z_i z_i^T,
\end{equation}
in which
\begin{equation}
  z_i := \frac1{\sqrt{\lambda_i}} X v_i
  \label{eq:z_i}
\end{equation}
are independent $\sigma_x$-subgaussian random vectors in $\rR^n$ with mean $0$ and covariance $I$. Then by denoting 
\begin{equation}\label{eq:note2}
    A = XX^T, \quad A_k = \sum_{i > k} \lambda_i z_i z_i^T, \quad A_{-k} = \sum_{i \ne k} \lambda_i z_i z_i^T,
\end{equation}
we can use Woodbury identity to decompose the terms in Eq.~\eqref{eq_calcu} as follows:
\begin{equation}\label{eq_term1}
\begin{aligned}
\sriskv 
&= \sum_i \lambda_i^2 z_i^T (\sum_j \lambda_j z_j z_j^T + n \lambda I)^{-2} z_i = \sum_{i} \frac{\lambda_i^2 z_i^T (A_{-i} + n \lambda I)^{-2} z_i}{[1 + \lambda_i z_i^T (A_{-i} + n \lambda I)^{-1} z_i]^2},\\
\sriskb
&\ge \sum_i \tilde{\truth}_i^2 \lambda_i (1 - \lambda_i z_i^T (XX^T + n \lambda I)^{-1} z_i)^2\\
\normv 
&= \sum_i \lambda_i z_i^T (\sum_j \lambda_j z_j z_j^T + n \lambda I)^{-2} z_i  = \sum_{i} \frac{\lambda_i z_i^T (A_{-i} + n \lambda I)^{-2} z_i}{[1 + \lambda_i z_i^T (A_{-i} + n \lambda I)^{-1} z_i]^2}\\
\normb
&\ge \sum_{i} \tilde{\truth}_i^2 \lambda_i^2 \parallel z_i \parallel_2^2 z_i^T (A + n \lambda I)^{-2} z_i.
\end{aligned}
\end{equation}
Using Lemma~\ref{lem_eigen}, \ref{lem_subspacenorm} and \ref{lem_ridgeeigen}, with a high probability, we are able to control the eigenvalues of the matrices in \eqref{eq:note1} and \eqref{eq:note2}, as well as the norms of the $z_i$, which provide an important characterization for Eq.~\eqref{eq_term1}, and induce our main proof sketches as follows.

\subsubsection{Proof sketch for Theorem \ref{thm_ridge}}

\paragraph{Upper bound for standard risk.} This follows directly from results of \citet{bartlett2020benign} and \citet{tsigler2020benign}.

%For its variance term $1.1$ , we can take Woodbury identity results as in \eqref{eq_term1}, with a high probability, the eigenvalues of matrix in \eqref{eq:note2} can be controlled with respect to $\Sigma_n$, as well as norms for i.i.d. sub-Gaussian random variables $z_i$, via calculations, the variance term can be upper bounded by
% \begin{equation*}
%\begin{aligned}
%& \quad \trace \{ X \Sigma_n X^T (XX^T + \lambda n I)^{-2} \}  = \sum_i \lambda_i^2 z_i^T (\sum_j \lambda_j z_j z_j^T + n \lambda I)^{-2} z_i\\
%& = \sum_{i=1}^{k^{*}} \frac{\lambda_i^2 z_i^T (A_{-i} + n \lambda I)^{-2} z_i}{[1 + \lambda_i z_i^T (A_{-i} + n \lambda I)^{-1} z_i]^2} + \sum_{i > k^{*}} \lambda_i^2 z_i^T (A + n \lambda I)^{-2} z_i\\
%& \le \frac{c_1^{\prime} k^{*}}{n} + \frac{c_2^{\prime} n \sum_{i > k^{*}} \lambda_i^2}{(\lambda_{k^{*}+1} r_{k^{*}}(\Sigma_n) + n \lambda)^2},
%\end{aligned}
%\end{equation*}
%in which $c_1^{\prime}$ is a constant. The critical step is to split the index set at $k^{*}$
%(defined in Condition \ref{cond:benign}), and take upper bound estimations separately. 
%As for the bias term $1.2$, we use the result of Lemma~\ref{lem_bart} 
% of \citet{tsigler2020benign} with respect to Condition~\ref{cond:benign}.

\paragraph{Lower bound for parameter norm.} We start with the variance term $\normv$ in Eq.~\eqref{eq_term1}, and using Cauchy-Schwarz,
\begin{equation*}
\begin{aligned}
\quad \trace \{ XX^T (XX^T + \lambda n I)^{-2} \} &= \sum_i \frac{1}{\lambda_i} \frac{\lambda_i^2 z_i^T(A_{-i} +\lambda n I)^{-2} z_i}{(1 + \lambda_i z_i^T(A_{-i} + \lambda n I)^{-1}z_i)^2} \\
&\geq \sum_i \frac{1}{\lambda_i \|z_i\|^2} \frac{(\lambda_i z_i^T(A_{-i} +\lambda n I)^{-1} z_i)^2}{(1 + \lambda_i z_i^T(A_{-i} + \lambda n I)^{-1}z_i)^2} \\
&\geq  \frac{c_1^{\prime}}{n} \sum_i \frac{1}{\lambda_i} \big( \frac{1}{ \lambda_i z_i^\T (A_{-i} + \lambda n I)^{-1} z_i } + 1 \big)^{-2}, 
\end{aligned}
\end{equation*}
where the last inequality is from controlling $\| z_i \|_2$ in Lemma~\ref{lem_subspacenorm}. 
We can further control the eigenvalues values of $A_{-i}$ using Lemma~\ref{lem_eigen} and \ref{lem_ridgeeigen}:
\begin{align*}
    \normv &\geq \frac{c_2^{\prime}}{n} \sum_i \frac{1}{\lambda_i} \left(\frac{\sum_{j > k^{*}} \lambda_j  + n\lambda}{n\lambda_i} + 1\right)^{-2} \\
    & \ge c_2^{\prime} \sum_{i < k^{*}} \frac{1}{\lambda_i} \min \left\{ \frac{1}{n}, \frac{b^2 n \lambda_i^2}{(\sum_{j > k^*} \lambda_j + n \lambda)^2} \right\}+  \frac{c_2^{\prime} n \lambda_{k^{*}+1} r_{k^{*}}}{(\sum_{j > k^*} \lambda_j + n \lambda)^2}, 
\end{align*} 
where the last step is followed from splitting the summation up to and after the critical index $k^{*}$ and maintaining the dominant terms.

Similarly, for the bias term $\normb$ in Eq.~\eqref{eq_term1}, by bounding the eigenvalues of matrix $A = XX^T$ with Lemma~\ref{lem_eigen}, we can show the following lower bound as (see more details in appendix),
\begin{equation*}
\begin{aligned}
\quad \normb &\geq \sum_i \tilde{\truth}_i^2 \lambda_i^2 \parallel z_i \parallel_2^2 z_i^T (A + \lambda n I)^{-2} z_i \\
&\ge c_3^{\prime} \sum_{i < k^{*}} \tilde{\truth}_i^2 \min \left\{ 1, \frac{b^2 n^2 \lambda_i^2}{(\sum_{j > k^*} \lambda_j)^2 + n^2 \lambda^2}  \right\} +  \frac{c_3^{\prime} n^2 \sum_{i > k^{*}} \tilde{\truth}_i^2 \lambda_i^2}{(\sum_{j > k^*} \lambda_j)^2 + n^2 \lambda^2} .
\end{aligned}
\end{equation*}

\subsubsection{Proof sketch for Theorem \ref{thm_r2}}

\paragraph{Lower bound for standard risk.} We need a refinement of the lower bounds from \citet{tsigler2020benign}.
By Eq.~\eqref{eq_term1}, we have the following lower bound for the variance term: 
\begin{align*}
    \sriskv = \text{tr}\{ X \Sigma X^T (XX^T + n \lambda I)^{-2} \} &\geq \sum_i \frac{1}{\|z_i\|^2} \frac{(\lambda_i z_i^T(A_{-i} +\lambda n I)^{-1} z_i)^2}{(1 + \lambda_i z_i^T(A_{-i} + \lambda n I)^{-1}z_i)^2} 
\end{align*}
where the inequality is via Cauchy-Schwarz. We further control the norm of $\| z_i \|_2$ using Lemma~\ref{lem_subspacenorm} and the eigenvalues of $A_{-i}$ using Lemma~\ref{lem_eigen} and \ref{lem_ridgeeigen}:
$$ \sriskv \geq \frac{c_4^{\prime}}{n} \sum_i \big( \frac{1}{ \lambda_i z_i^\T (A_{-i} + \lambda n I)^{-1} z_i } + 1 \big)^{-2} \geq \frac{c_5'}{n} \sum_i \left(\frac{\sum_{j > k^{*}} \lambda_j  + n\lambda}{n\lambda_i} + 1\right)^{-2}
.
$$
Splitting the summation term into eigenvalues smaller and greater than the regularization parameter, combined with the fact that $\lambda \geq \fracl{\lambda_{k^{*} + 1} r_{k^{*}}}{n} \geq b \lambda_{k^{*} + 1}$ yields,
$$\sriskv \geq \frac{c_6'}{n} \left( \sum_{\lambda_i > \lambda} 1 + \sum_{\lambda_i \leq \lambda} \frac{\lambda_i^2}{\lambda^2} \right).$$

Next, we turn to the bias term.
Writing $X = Z \Lambda^{1/2} V^T$ and $\varSigma = V \Lambda V^T$, we obtain
\begin{equation*}
\begin{aligned}
&\quad \mathbb{E} \truth^T [I - X^T(XX^T + n \lambda I)^{-1}X] \Sigma [I - X^T(XX^T + n \lambda I)^{-1} X] \truth \\
&= \mathbb{E} \tilde{\truth}^T[I - \Lambda^{1/2} Z^T(XX^T + n \lambda I)^{-1}Z \Lambda^{1/2}] \Lambda [I - \Lambda^{1/2} Z^T(XX^T + n \lambda I)^{-1}Z \Lambda^{1/2}] \tilde{\truth}\\ 
&= \sum_i \tilde{\truth}_i^2 \left\{ \lambda_i (1 - \lambda_i z_i^T (XX^T + n \lambda I)^{-1} z_i)^2 + \sum_{j \ne i} \lambda_j^2 \lambda_i (z_i^T (XX^T + n \lambda I)^{-1}z_j)^2 \right\}\\
& \ge \sum_i \tilde{\truth}_i^2 \lambda_i (1 - \lambda_i z_i^T (XX^T + n \lambda I)^{-1} z_i)^2 = \sum_i \frac{\tilde{\truth}_i^2 \lambda_i}{(1 + \lambda_i z_i^T (A_{-i} + n \lambda I)^{-1} z_i)^2},
\end{aligned}
\end{equation*}
where the last equality is by the Woodbury identity.
Moreover, the eigenvalues of matrices $A_{-i}$ are dominated by $n\lambda$ since $n \lambda \ge \lambda_{k^{*} + 1} r_{k^{*}}$ which implies the desired lower bound for the bias term:
$$ \sriskb \geq c_7' \sum_{i} \frac{\tilde{\truth}_i^2 \lambda_i}{ (1 + \frac{\lambda_i}{\lambda})^2 } . $$

\paragraph{Lower bound for parameter norm.}
Based on the condition $n \lambda \ge \lambda_{k^{*} + 1} r_{k^{*}}$, we have
\begin{equation*}
    %n^2 \lambda^2 \le \lambda_{k^{*} + 1}^2 r_{k^{*}}^2 \le 2 n^2 \lambda^2, \quad 
    n \lambda \le n\lambda + \lambda_{k^{*} + 1} r_{k^{*}} \le 2 n \lambda.
\end{equation*}
Thus, substituting the terms in the results of Theorem~\ref{thm_ridge} with the dominant term $n\lambda$, we get the final expressions in Theorem~\ref{thm_r2}.
%to bound $n\lambda$ and $n^2 \lambda^2$ in results of Theorem~\ref{thm_ridge},  we get 
%the final expressions in Theorem~\ref{thm_r2}.

\subsubsection{Proof sketch of Corollary \ref{cor_1}, Theorem \ref{thm:main_explode} and Theorem \ref{thm:main_tradeoff}}

To prove the trade-off results, we  need to analyze the standard risk bias term $\sriskb$, as well as the parameter variance term $\normv$ in \eqref{eq_calcu}. We analyze three separate regimes for $\lambda$:

\paragraph{Small regularization:} 

If $\lambda \le (\lambda_{k^{*} + 1} r_{k^{*}})/n$, then from Condition \ref{cond:benign}, we have
\begin{equation*}
    r_{k^{*}} \ge bn, \Rightarrow \lambda \le \lambda_{k^{*} + 1} \le \frac{\lambda_{k^{*} + 1} r_{k^{*}}}{bn}, \quad n \lambda \le \lambda_{k^{*} + 1} r_{k^{*}} = \sum_{j > k^{*}} \lambda_j,
\end{equation*}
then upper bounding $r_{k^{*}}^2 \lambda_{k^{*} + 1}^2 + n^2 \lambda^2$ by $2 r_{k^{*}}^2 \lambda_{k^{*} + 1}^2$, and obtain
\begin{equation*}
\normv \ge \frac{ n \sigma^2 \lambda_{k^{*}+1} r_{k^{*}}}{c^{\prime}_8 \lambda_{k^{*}+1}^2 r_{k^{*}}^2} = \frac{ n \sigma^2}{ c^{\prime}_8 \lambda_{k^{*}+1} r_{k^{*}}},
\end{equation*}
While data noise $\sigma^2 = \omega(\lambda_{k^*+1} r_{k^*} / n)$, the parameter norm diverges to infinity.
 
\paragraph{Large regularization:} 

If $\lambda \geq \lambda_1$,
we can consider the bias term $\sriskb$ in standard risk, to be specific
\begin{equation*}
\sriskb \ge \frac{1}{c_{9}^{\prime}} \left( \sum_{\lambda_i >\lambda} \frac{\lambda^2 \tilde{\truth}_i^2}{\lambda_i} + \sum_{\lambda_i \leq \lambda} \tilde{\truth}_i^2 \lambda_i \right) = \frac{1}{c_{9}^{\prime}} \| \truth \|^2_{\Sigma},
\end{equation*}
By Condition \ref{cond:benign}, the standard risk is a constant.

\paragraph{Intermediate regularization:} 

If $(\lambda_{k^{*} + 1} r_{k^{*}}) / n \le \lambda \le \lambda_1$, with Condition \ref{cond:benign},
\begin{equation*}
\begin{aligned}
& n^2 \lambda^2 \le \lambda_{k^{*} + 1}^2 r_{k^{*}}^2 + n^2 \lambda^2 \le 2 n^2 \lambda^2,\\
& \lambda \ge \frac{\lambda_{k^{*} + 1} r_{k^{*}}}{n} \ge b \lambda_{k^{*} + 1} \ge \lambda_{k^{*} + 1},
\end{aligned}
\end{equation*}
then we can upper bound $r_{k^{*}}^2 \lambda_{k^{*} + 1}^2 + n^2 \lambda^2$ by $2 n^2 \lambda^2$. 
We lower bound the bias term $\sriskb$ in the standard risk as
\begin{equation*}
 \sriskb \ge \sum_i \frac{\tilde{\truth}_i^2 \lambda_i }{(1 + \frac{\lambda_i \parallel z_i \parallel_2^2 }{\mu_n(A_{-i}) + n \lambda})^2} \ge c_{10}^{\prime} \left( \sum_{\lambda_i \ge \lambda} \frac{\lambda^2 \tilde{\truth}_i^2 }{\lambda_i} + \sum_{\lambda_i < \lambda} \tilde{\truth}_i^2 \lambda_i \right),
 \end{equation*}
and we lower bound the variance term $\normv$ in the parameter norm as
\begin{equation*}
 \sigma^2 \normv \ge c_{11}^{\prime} \frac{\sigma^2}{n} \left( \sum_{\lambda_i > \lambda} \frac{1}{\lambda_i} + \sum_{\lambda_i < \lambda} \frac{\lambda_i}{\lambda^2} \right).
\end{equation*}
With Condition~\ref{cond:trade}, if we have $(\lambda_{k^*+1} r_{k^*})/n \le \lambda \le \lambda_{w^*}$, then
\begin{equation*}
\begin{aligned}
 \sriskb(\ridge) \mathbb{E} \| \ridge \|_2^2 &\ge c_{12}^{\prime} \frac{\sigma^2}{n \lambda^2} \sum_{\lambda_i \le \lambda} \tilde{\truth}_i^2 \lambda_i \sum_{\lambda_i \le \lambda} \lambda_i \ge c_{12}^{\prime} \sigma^2 \| \truth \|_2^2 \sqrt{\max \{ \frac{k^*}{n}, \frac{n}{R_{k^*}} \}},  
\end{aligned}
\end{equation*}
which leads to 
\begin{equation*}
\frac{\sriskb(\ridge)}{\| \truth \|_2^2 \srisk(\ridge |_{\lambda=0})} \ge c_{13}^{\prime} \min \left\{ \frac{\lambda^2 \sum_{\lambda_i > \lambda} \tilde{\truth}_i^2 / \lambda_i + \sum_{\lambda_i \le \lambda} \tilde{\truth}_i^2 \lambda_i}{\| \truth \|_2^2 (\lambda_{k^*+1}^2 \| \truth_{0:k^*} \|^2_{\Sigma_{0:k^*}^{-1}} + \| \truth_{k^*:\infty} \|^2_{\Sigma_{k^*:\infty}} )} , \frac{1}{\mathbb{E} \| \ridge \|_2^2 \sqrt{\max\{ k^*/n, n/R_{k^*} \}}} \right\}.
\end{equation*}
And while $\lambda_{w^*} \le \lambda < \lambda_1$, considering the fact that $\lambda_{w^*} / \lambda_{k^*+1}$ tends to infinity and $\sriskb$ will increase with $\lambda$, we can get
\begin{equation*}
    \srisk(\ridge) \ge \sriskb(\ridge |_{\lambda = \lambda_{w^*}}) \ge \| \truth \|_2^2 \srisk(\ridge |_{\lambda = 0}).
\end{equation*}
Combining all the results above, we can obtain the corresponding result in Corollary~\ref{cor_1}.
 
So in this regime, we reveal that under large enough sample size $n$, with a high probability, the near optimal standard risk convergence rate and stable adversarial risk can not be obtained at the same time.
By considering the results for all regimes together, we obtain the conclusion stated in Theorem~\ref{thm:main_tradeoff}.

\section{Outline of the argument for NTK framework}

The proof sketches for Theorem \ref{thm_ntk} is summarized in this section, and detailed proofs could be found in Appendix \ref{pf:ntk}. Similarly, we still use $c'_i$ to denote positive constants that only depend on $b, \sigma_x$.

Recalling the expression of gradient descent solution $\hat{w}$ \eqref{eq:ntk_para}:
\begin{equation*}
    \hat{w} = w_0 + \nabla F^T (\nabla F \nabla F^T)^{-1}(y - F),
\end{equation*}
the proof of Theorem \ref{thm_ntk} mainly contains three steps: linearizing the kernel matrix $K = \nabla F \nabla F^T$, upper bounding the standard risk $\srisk(\hat{w})$ and lower bounding the adversarial risk $\arisk(\hat{w})$. Comparing with the analysis on linear model, the primary technical challenge in NTK framework is to linearize the kernel matrix with a high probability regime. Once we have the linearized approximation for the kernel matrix, we can proceed with a similar process as in the linear model.

\paragraph{Step 1: kernel matrix linearization.} By Lemma \ref{lem:ntk1} and \ref{lem:ntk2} in \citet{jacot2018neural}, the components of $K = \nabla F \nabla F^T$ in two-layer neural network can be expressed as
\begin{align*}
K_{i,j} = K(x_i, x_j) &= \nabla_w f_{NTK}(w_0, x_i)^T \nabla_w f_{NTK} (w_0, x_j)\\
&= \frac{x_i^T x_j}{\pi p} \arccos\left( - \frac{x_i^T x_j}{\| x_i \| \| x_j \|} \right) + \frac{\| x_i \| \| x_j \|}{2 \pi p} \sqrt{1 - \left( \frac{x_i^T x_j}{\| x_i \| \| x_j \|} \right)^2} + o_p(\frac{1}{\sqrt{m}}),
\end{align*}
then with Condition \ref{cond:ntk_benign} and \ref{cond:ntk_high-dim}, using a refinement of Theorem $2.1$ in \citet{el2010spectrum}, i.e, Lemma \ref{lem:spectrum}, we can approximate $K$ as a linearized matrix $\tilde{K}$:
\begin{equation*}
    \tilde{K} = \frac{l}{p} (\frac{1}{2 \pi} + \frac{3 r_0(\varSigma^2)}{4 \pi l^2}) 11^T + \frac{1}{2 p} XX^T + \frac{l}{p} (\frac{1}{2} - \frac{1}{2 \pi}) I_n.
\end{equation*}

\paragraph{Step 2: standard risk upper bound estimation.} With the solution in Eq.\eqref{eq:ntk_para}, the expected standard risk could be decomposed into bias term and variance term:
\begin{align*}
\srisk(\hat{w})  &\le \underbrace {R^2 \mathbb{E}_{ x} \| \nabla_w f_{NTK} (w_0,x) \nabla_w f_{NTK} (w_0, x)^T - \frac{1}{n} \nabla F^T \nabla F \|_2}_{\sriskb}\\
& \quad + \underbrace{\sigma^2 \mathbb{E}_{ x } \text{trace} \{ K^{-2} \nabla F \nabla_w f_{NTK}(w_0,x) \nabla_w f_{NTK}(w_0,x)^T \nabla F^T \}  }_{\sriskv}.
\end{align*}
For the bias term $\sriskb$, with Lemma \ref{lem:lip_gaussian} and \ref{lem:matrix_mean}, we can verify the sub-gaussian property of $\nabla_w f_{NTK}(w_0, x)$, which implies that $\mathbb{E}_{ x} \| \nabla_w f_{NTK} (w_0,x) \nabla_w f_{NTK} (w_0, x)^T -  \nabla F^T \nabla F / n \|_2$ converges as sampe size $n$ grows,  then we could get the concentration inequality as:
\begin{equation*}
    \sriskb \le c'_{14} R^2 \frac{l^{1/2}}{ p ^{1/2} n^{1/4}},
\end{equation*}
with a high probability. 

While turning to the variance term $\sriskv$, we can take another $n$ i.i.d.~samples $x'_1, z'_2, \dots, x'_n$ from the same distribution as $x_1, \dots, x_n$ and denote $\nabla F(x') = [\nabla_w f_{NTK}(w_0, x'_1), \dots, \nabla_w f_{NTK}(w_0, x'_n)]^T$, further obtain
\begin{align*}
\sriskv &= \sigma^2 \mathbb{E}_{ x } \text{trace} \{ K^{-2} \nabla F \nabla_w f_{NTK}(w_0,x) \nabla_w f_{NTK}(w_0,x)^T \nabla F^T \}  \\
&=  \frac{\sigma^2}{n} \mathbb{E}_{ x'_i} \text{trace} \{ K^{-2} \nabla F \nabla F(x')^T \nabla F(x') \nabla F^T \},
\end{align*}
similar to Lemma \ref{lem:spectrum}, with a high probability, we could take the linearization procedure as:
\begin{align*}
& \|  \nabla F \nabla F(x')^T - \frac{1}{p} (\frac{1}{2 \pi} + \frac{3 r_0(\varSigma^2)}{4 \pi l p^2}) 11^T - \frac{1}{2 p} XX^{'T} \|_2 \le \frac{4 l}{p n^{1/16}},\\
&  \| \nabla F(x') \nabla F^T - \frac{1}{p} (\frac{1}{2 \pi} + \frac{3 r_0(\varSigma^2)}{4 \pi l p^2}) 11^T - \frac{1}{2 p} X' X^T \|_2 \le \frac{4 l}{ p n^{1/16}},
\end{align*}
then replacing all the matrix $K$, $\nabla F \nabla F(x')^T$ and $\nabla F(x') \nabla F^T$ by their linearized approximations respectively, we obtain
\begin{align*}
\sriskv / \sigma^2 & \le c'_{15} \frac{1}{p^2} 1^T \tilde{K} 1 + c'_{16} \frac{1}{p^3} \text{trace} \{ \tilde{K}^{-2} X \varSigma X^T \} + c'_{17} \frac{l^2}{p^2 n^{9/8}} \text{trace} \{ \tilde{K}^{-2} \} \\
& \le c'_{18} \left(  \frac{1}{n^{1/8}} + \frac{k^*}{n} + \frac{n \sum_{j > k^*} \lambda_j^2}{l^2} \right),
\end{align*}
where the first inequality is from the tiny error in matrix linearization, and the second inequality is from concentration bounds in Lemma \ref{lem_subspacenorm} and \ref{lem_ridgeeigen} with Conition~\ref{cond:ntk_benign} and \ref{cond:ntk_high-dim}, which is similar to the analysis in linear model.

\paragraph{Step 3: adversarial risk lower bound estimation.} As $\arisk(\hat{w})$ can be lower bounded as
\begin{align*}
& \quad \arisk(\hat{w})\\
&= \alpha^2 \mathbb{E}_{ x,\epsilon} \| \nabla_x f_{NTK} (\hat{w}, x) \|_2^2 = \alpha^2 \mathbb{E}_{ x,\epsilon} \| \nabla_x f_{NTK} (w_0, x) + \frac{\partial^2 f_{NTK}(w_0, x)}{\partial w \partial x} (\hat{w} - w_0) \|_2^2\\
& \ge \alpha^2 \left| \mathbb{E}_{ x,\epsilon} \| \frac{\partial^2 f_{NTK}(w_0, x)}{\partial w \partial x} (\hat{w} - w_0) \|_2^2 - \mathbb{E}_{ x,\epsilon} \| \nabla_x f_{NTK} (w_0, x)  \|_2^2 \right|,
\end{align*}
where the inequality is from triangular inequality. As the second term can be upper bounded by a constant, we just take a detailed analysis on the first term. While considering Condition~\ref{cond:ntk_high-dim}, we can obtain that
\begin{align*}
& \quad \mathbb{E}_{x,\epsilon} \| \frac{\partial^2 f_{NTK}(w_0, x)}{\partial w \partial x} (\hat{w} - w_0) \|_2^2 \\
&\ge \frac{1}{16 p^2} \mathbb{E}_{ \epsilon} \text{tr} \{ K^{-1} (\nabla F (w_* - w_0) + \epsilon) (\nabla F (w_* - w_0) + \epsilon)^T K^{-1} XX^T \}\\
&\ge \frac{\sigma^2}{16 p^2} \text{tr} \{ K^{-2} XX^T \} \ge \frac{\sigma^2}{32 p^2} \text{tr} \{\tilde{K}^{-2} XX^T \},
\end{align*}
where the first inequality is from the derivative calculation on each component of $\partial^2 f_{NTK}(w_0, x) / (\partial w \partial x)$, the second inequality is from ignoring the term related to $w_* - w_0$, and the last inequality is from linearizing kernel matrix $K$ to $\tilde{K}$. Then the following steps is similar to the analysis on linear model. To be specific,  with Lemma \ref{lem_subspacenorm}
 and \ref{lem_ridgeeigen}, with a high probability, we have
\begin{equation*}
     \mathbb{E}_{x,\epsilon} \| \frac{\partial^2 f_{NTK}(w_0, x)}{\partial w \partial x} (\hat{w} - w_0) \|_2^2 \ge c'_{19} \sigma^2 \frac{n \lambda_{k^*+1} r_{k^*}}{l^2},
\end{equation*}
which will lead to an exploded lower bound for $\arisk(\hat{w})$.

\section{Conclusion and discussion}

In this work, we studied benign overfitting settings where consistent estimation can be achieved even when we exactly fit the training data.
However, we show that in such scenarios, it is not possible to achieve good adversarial robustness, even if the ground truth model is robust to adversarial attacks. This reveals a fundamental trade-off  between standard risk and adversarial risk under suitable conditions.

There are still numerous interesting questions for further exploration.
Do overparameterized neural networks give rise to a deep neural tangent kernel matrix with ``slowly decaying'' eigenvalues that satisfy benign overfitting conditions?
Do the trade-offs between standard and adversarial risks exist when the adversarial budget is defined differently, such as in terms of $\ell_1$ or $\ell_0$ (pseudo)norms? Do the more complex models, such as Transformer, exhibit distinct performance in terms of adversarial robustness?
%another is about adversarial training, as our budget is defined on $\mathscr{L}_2$ norms, which do not restrict the sparsity on estimator, there may be a valuable exploration on restricting the adversarial budget on $\mathscr{L}_1$ norm or $\mathscr{L}_0$ norm. 

Finally, the issue of adversarial robustness has broader social impact in AI safety. This work tries to understand the fundamental reason why modern overparameterized machine learning methods lead to models that are not robust. A better theoretical understanding can be useful for developing safer AI models in real applications.

\section*{Acknowledgement}
We would like to thank Daniel Hsu, Difan Zou, Navid Ardeshir and Yong Lin for their helpful comments and suggestions.

\bibliography{ref.bib}
\bibliographystyle{apalike}

\newpage

\appendix
\section{Constant Notation}
Before the main proof process, we denote several corresponding constants in Table \ref{tab:const}:
\begin{table}[h]
    \centering
    \footnotesize
    \begin{tabular}{c|c}
    \toprule
         Symbol & Value\\
         \midrule
          $c'$ & $\max\{ 2, (1 + 16 \ln 3 \cdot \sigma_x^2 \cdot 54e) 32 \ln 3 \cdot \sigma_x^2 \cdot 54e\}$ \\
          \midrule
           $b$ & $> c^{'2}$\\
          \midrule
          $c$ & $> 256 \cdot (162e)^4 \sigma_x^4$ \\
          \midrule
          $c_1$ & $\max\{ c' + c' / b, (1 / c' - c' / b)^{-1} \}$  \\  
         \midrule
         $c_2$ & $8(162e)^2 \sigma_x^2$ \\
         \midrule
         $c_3$ & $2$\\
         \bottomrule
    \end{tabular}
    \caption{Constant List}
    \label{tab:const}
\end{table}

\section{Technical lemmas from prior works}

\begin{lemma}[Lemma 10 in \citealp{bartlett2020benign}]\label{lem_eigen}
 There are constants $b, c \ge 1$ such that, for any $k \ge 0$, with probability at least $1 - 2 e^{- \frac{n}{c}}$,
\begin{enumerate}
\item for all $i \ge 1$,
\begin{equation*}
    \mu_{k+1}(A_{-i}) \le \mu_{k+1}(A) \le \mu_{1}(A_{k}) \le c_1 (\sum_{j > k} \lambda_j + \lambda_{k+1} n);
\end{equation*}
\item for all $1 \le i \le k$,
\begin{equation*}
    \mu_{n}(A) \ge \mu_{n}(A_{-i}) \ge \mu_{n}(A_{k}) \ge \frac{1}{c_1} \sum_{j > k} \lambda_j - c_1 \lambda_{k+1} n;
\end{equation*}
\item if $r_k \ge bn$, then
\begin{equation*}
    \frac{1}{c_1} \lambda_{k+1} r_k \le \mu_n (A_k) \le \mu_1 (A_k) \le c_1 \lambda_{k+1} r_k,
\end{equation*}
\end{enumerate}
where $c_1 > 1$ is a constant only depending on $b, \sigma_x$.
\end{lemma}

\begin{lemma}[Corollary 24 in \citealp{bartlett2020benign}]\label{lem_subspacenorm}
 For any centered random vector $\bm{z} \in \mathbb{R}^n$ with independent $\sigma^2_x$ sub-Gaussian coordinates with unit variances, any $k$ dimensional random subspace $\mathscr{L}$ of $\mathbb{R}^n$ that is independent of $\bm{z}$, and any $t > 0$, with probability at least $1 - 3 e^{-t}$,
\begin{equation*}
\begin{aligned}
& \parallel z \parallel^2 \le n + 2 (162e)^2 \sigma_x^2(t + \sqrt{nt}),\\
& \parallel \bm{\Pi}_{\mathscr{L}} \bm{z} \parallel^2 \ge n - 2 (162e)^2 \sigma_x^2 (k + t + \sqrt{nt}),
\end{aligned}
\end{equation*}
where $\bm{\Pi}_{\mathscr{L}}$ is the orthogonal projection on $\mathscr{L}$.
\end{lemma}

\begin{lemma}\label{lem_ridgeeigen}
 There are constants $b, c \ge 1$ such that, for any $k \ge 0$, with probability at least $1 - 2 e^{- \frac{n}{c}}$:
 \begin{enumerate}
\item for all $i \ge 1$,
\begin{equation*}
    \mu_{k+1}(A_{-i} + \lambda n I) \le \mu_{k+1}(A + \lambda n I) \le \mu_{1}(A_{k} + \lambda n I) \le c_1 (\sum_{j > k} \lambda_j + \lambda_{k+1} n) + \lambda n;
\end{equation*}
\item for all $1 \le i \le k$,
\begin{equation*}
    \mu_{n}(A + \lambda n I) \ge \mu_{n}(A_{-i} + \lambda n I) \ge \mu_{n}(A_{k} + \lambda n I) \ge \frac{1}{c_1} \sum_{j > k} \lambda_j - c_1 \lambda_{k+1} n + \lambda n;
\end{equation*}
\item if $r_k \ge bn$, then
\begin{equation*}
    \frac{1}{c_1} \lambda_{k+1} r_k + n \lambda \le \mu_n (A_k + \lambda n I) \le \mu_1 (A_k + \lambda n I) \le c_1  \lambda_{k+1} r_k + n \lambda.
\end{equation*}   
\end{enumerate}
\end{lemma}

\begin{proof}
With Lemma \ref{lem_eigen}, the first two claims follow immediately. For the third claim: if $r_k(\Sigma) \ge bn$, we have that $bn \lambda_{k+1} \le \sum_{j > k} \lambda_j$, so
\begin{equation*}
 \begin{aligned}
& \mu_1 (A_k + \lambda n I) \le c_1 \lambda_{k+1} r_k (\Sigma) + \lambda n \le c_1  \lambda_{k+1} r_k + n \lambda \\
& \mu_n (A_k + \lambda n I) \ge \frac{1}{c_1} \lambda_{k+1} r_k (\Sigma) + \lambda n \ge \frac{1}{c_1} \lambda_{k+1} r_k + n \lambda,
 \end{aligned}   
\end{equation*}
for the same constant $c_1 > 1$ as in Lemma \ref{lem_eigen}.
\end{proof}

\begin{lemma}[Proposition 2.7.1 in \citealp{vershynin2018high}]\label{lem_sg_se}
 For any random variable $\xi$ that is centered, $\sigma^2$-subgaussian, and unit variance, $\xi^2 - 1$ is a centered $162e\sigma^2$-subexponential random variable, that is,
\begin{equation*}
    \mathbb{E}\exp(\lambda (\xi^2 - 1)) \le \exp((162e\lambda \sigma^2)^2), 
\end{equation*}
for all such $\lambda$ that $| \lambda | \le 1 / (162 e \sigma^2)$.
\end{lemma}

\begin{lemma}[Lemma 15 in \citealp{bartlett2020benign}]\label{lem_sum}
Suppose that $\{ \eta_i \}$ is a sequence of non-negative random variables, and that $\{ t_i \}$ is a sequence of non-negative real numbers (at least one of which is strictly positive) such that, for some $\delta \in (0, 1)$ and any $i \ge 1$, $\Pr(\eta_i > t_i) \ge 1 - \delta$. Then,
\begin{equation*}
    \Pr\left(\sum_i \eta_i \ge \frac{1}{2} \sum_i t_i\right) \ge 1 - 2 \delta.
\end{equation*}
\end{lemma}

\begin{lemma}[Lemma 2.7.6 in \citealp{vershynin2018high}]\label{lem_stnorm}
 For any non-increasing sequence $\{ \lambda_i \}_{i=1}^{\infty}$ of non-negative numbers such that $\sum_i \lambda_i < \infty$, and any independent, centered, $\sigma-$subexponential random variables $\{ \xi_i \}_{i=1}^{\infty}$, and any $x > 0$, with probability at least $1 - 2 e^{-x}$
\begin{equation*}
    | \sum_i\lambda_i \xi_i | \le 2 \sigma \max \left( x \lambda_1, \sqrt{x \sum_i \lambda_i^2} \right).
\end{equation*}
\end{lemma}

\begin{lemma}[Consequence of Theorem 5 in \citet{tsigler2020benign}]\label{lem_bart}
There is an absolute constant $c>1$ such that the following holds.
For any $k < \frac{n}{c}$, with probability at least $1 - c e^{- \frac{n}{c}}$, if $A_k$ is positive definite, then
\begin{equation*}
\begin{aligned}
{\trace \{ \varSigma [I - X^T (XX^T + \lambda n I)^{-1} X]^2 \}}
  & \le
  %\parallel \truth_{k : \infty} \parallel^2_{\Sigma_{k:\infty}}
  \left( \sum_{i>k} \lambda_i \right)
  \left( 1 + \frac{\mu_1 (A_k + \lambda n I)^2}{\mu_n (A_k + \lambda n I)^2} +  \frac{n \lambda_{k+1}}{\mu_n (A_k + \lambda n I)} \right)\\
& +
%\parallel \truth_{0:k} \parallel^2_{\Sigma_{1:k}^{-1}}
\left( \sum_{i\leq k} \frac1{\lambda_i} \right)
\left( \frac{\mu_1 (A_k + \lambda n I)^2}{n^2} + \frac{\lambda_{k+1}}{n} \cdot \frac{\mu_1 (A_k + \lambda n I)^2}{\mu_n (A_k + \lambda n I)} \right) ,
\\
{ \trace\{ X \varSigma X^T (XX^T + \lambda n I)^{-2} \}}
& \le
\frac{\mu_1(A_k+\lambda nI)^2}{\mu_n(A_k +\lambda nI)^2} \cdot \frac{k}{n}
+ \frac{n}{\mu_n(A_k + \lambda nI)^2}
\left( \sum_{i>k} \lambda_i^2 \right) .
\end{aligned}
\end{equation*}    
\end{lemma}

\begin{lemma}[Proposition 1 in \citealp{jacot2018neural}]\label{lem:ntk1}
For a network of depth L at initialization, with a Lipschitz nonlinearity $\sigma$, and in the limit as $n_1, \dots, n_{L-1} \to \infty$, the output functions $f_{\theta, k}$, for $k = 1, \dots, n_L$, tend (in law) to iid centered Gaussian processes of covariance $\Sigma^{(L)}$ is defined recursively by:
\begin{equation*}
\begin{aligned}
& \Sigma^{(1)} (x,x^{\prime}) = \frac{1}{n_0} x^T x^{\prime} + \beta^2,\\
& \Sigma^{(L+1)} (x,x^{\prime}) = \mathbb{E}_{f \sim N(0,\Sigma^{(L)})} [\sigma (f(x)) \sigma (f(x^{\prime}))] + \beta^2,
\end{aligned}
\end{equation*}
taking the expectation with respect to a centered Gaussian process f of covariance $\Sigma^{(L)}$.    
\end{lemma}

\begin{lemma}[Theorem 1 in \citealp{jacot2018neural}]\label{lem:ntk2}
For a network of depth L at initialization, with a Lipschitz nonlinearity $\sigma$, and in the limit as the layers width $n_1, \dots, n_{L-1} \to \infty$, the NTK $\Theta^{(L)}$ converges in probability to a deterministic limiting kernel:
\begin{equation*}
\Theta^{(L)} \to \Theta_{\infty}^{(L)} \otimes Id_{n_L}.
\end{equation*}
The scalar kernel $\Theta_{\infty}^{(L)}$ : $\mathbb{R}^{n_0} \times \mathbb{R}^{n_0} \to \mathbb{R}$ is defined recursively by
\begin{equation*}
\begin{aligned}
& \Theta_{\infty}^{(1)} (x,x^{\prime}) = \Sigma^{(1)} (x,x^{\prime}),\\
& \Theta_{\infty^{(L+1)} (x,x^{\prime})} = \Theta_{\infty}^{(L)}(x, x^{\prime}) \dot{\Sigma}^{(L+1)}(x,x^{\prime}) + \Sigma^{(L+1)}(x,x^{\prime}),
\end{aligned}
\end{equation*}
where
\begin{equation*}
\dot{\Sigma}^{(L+1)}(x,x^{\prime}) = \mathbb{E}_{f \sim N(0,\Sigma^{(L)})} [\dot{\sigma} (f(x)) \dot{\sigma} (f(x^{\prime}))]
\end{equation*}
taking the expectation with respect to a centered Gaussian process f of covariance $\Sigma^{(L)}$, and where $\dot{\sigma}$ denotes the derivative of $\sigma$.    
\end{lemma}

\section{Proof for Theorem \ref{thm_ridge} and Theorem \ref{thm_r2}}\label{pf:ridge}

Denoting $\tilde{\truth} = V^T \truth$ ($\Sigma = V^T \Lambda V$), based on data assumptions above, we have the following decomposition for excessive standard risk:
\begin{equation}\label{eq:std}
\begin{aligned}
\srisk (\ridge) &= \mathbb{E}_{\truth, x, \epsilon} [x^T(\truth - \ridge)]^2\\
&= \mathbb{E}_{\truth,x, \epsilon} \{ x^T [I - X^T (X X^T + \lambda n I)^{-1} X] \truth - x^T X^T (X X^T + \lambda n I)^{-1} \epsilon \}^2\\
&= \mathbb{E}_{\truth} \truth^T [I - X^T (X X^T + \lambda n I)^{-1} X] \Sigma [I - X^T (X X^T + \lambda n I)^{-1} X] \truth \\
& \quad + \mathbb{E}_{\epsilon} \epsilon^T (X X^T + \lambda n I)^{-1} X \Sigma X^T (X X^T + \lambda n I)^{-1} \epsilon\\
& = \underbrace{ \sum_i   \tilde{\truth}_i^2 ([I - X^T (XX^T + n \lambda I)^{-1} X] \Sigma [I - X^T (X X^T + \lambda n I)^{-1} X])_{i,i}  }_{\sriskb}\\
& \quad + \sigma^2 \underbrace{\trace \{ X \Sigma X^T (X X^T + \lambda n I)^{-2} \}}_{\sriskv}.
\end{aligned}
\end{equation}

And similarly, we can get an expression for parameter norm as
\begin{equation}\label{eq:para}
\begin{aligned}
\mathbb{E} \parallel \hat{\bm{\theta}} \parallel^2 &= \mathbb{E}_{X, \epsilon,\truth} (\epsilon^T + \truth^T X^T) (X X^T + \lambda n I)^{-1} X X^T (X X^T + \lambda n I)^{-1} (X \truth + \epsilon)\\
& = \underbrace{ \sum_i \tilde{\truth}_i^2 (X^T (XX^T + n \lambda I)^{-1} XX^T (XX^T + n \lambda I)^{-1} X)_{i,i} }_{\normb}\\
& \quad + \sigma^2 \underbrace{\trace \{ X X^T (X X^T + \lambda n I)^{-2} \}}_{\normv},
\end{aligned}
\end{equation}

\subsection{Standard Risk}

The upper bound in standard risk can be directly obtained from Lemma \ref{lem_bart} in \citet{tsigler2020benign}, which shows that with probability at least $1 - c e^{- \frac{n}{c}}$, we have
\begin{equation}\label{eq:stdu}
\begin{aligned}
\srisk / C_1 &\le \sum_{j > k^{*}} \lambda_j \truth_{*j}^2 + \sum_{j=1}^{k^{*}} \frac{\truth_{*j}^2}{\lambda_j} \frac{(n \lambda + \lambda_{k^{*} + 1} r_{k^{*}})^2}{n^2} + \sigma^2 \left( \frac{k^{*}}{n} + \frac{n \sum_{j > k^{*}} \lambda_j^2}{(n \lambda + \lambda_{k^{*}+1} r_{k^{*}})^2} \right),  
\end{aligned}
\end{equation}
where $C_1 > 0$ is a constant which only depends on $b, \sigma_x$. And here we just consider the lower bound in $\srisk$.

First, we will take estimation for term $\sriskv = \trace \{ X \Sigma X^T (X X^T + \lambda n I)^{-2} \}$. Considering $\Sigma = \sum_i \lambda_i v_i v_i^T$ in our model setting, where $v_i \in \mathbb{R}^p$, we can rewrite $X X^T$ as 
\begin{equation*}
    X X^T = \sum_i \lambda_i z_i z_i^T,
\end{equation*}
where $z_i$ are as defined in \eqref{eq:z_i}.

By defining
\begin{equation*}
    A = X X^T, \quad A_k = \sum_{i > k} \lambda_i z_i z_i^T, \quad A_{-k} = \sum_{i \ne k} \lambda_i z_i z_i^T,
\end{equation*}
with Woodbury identity, we have
\begin{equation}\label{eq:stdv_dec}
\begin{aligned}
\sriskv & = \trace \{ X \Sigma X^T (X X^T + \lambda n I)^{-2} \}  = \sum_i \lambda_i^2 z_i^T (\sum_j \lambda_j z_j z_j^T + n \lambda I)^{-2} z_i\\
& = \sum_i \frac{\lambda_i^2 z_i^T (A_{-i} + n \lambda I)^{-2} z_i}{[1 + \lambda_i z_i^T (A_{-i} + n \lambda I)^{-1} z_i]^2}.
\end{aligned}
\end{equation}
And considering to use Lemma \ref{lem_subspacenorm} by setting $t < n / c$, for each index $i$, denoting $\mathscr{L}_i$ is the subspace in $\mathbb{R}^n$, related to the $n - k^{*}$ eigenvalues of $A_{-i} + n \lambda I$, then with probability at least $1 - 3 e^{-n / c}$, we have
\begin{equation*}
\begin{aligned}
& \parallel z_i \parallel_2^2 \le n + 2(162e)^2 \sigma_x^2 (t  + \sqrt{nt}) \le c_2 n,\\
& \parallel \Pi_{\mathscr{L}_i} z_i \parallel_2^2 \ge n - 2 (162e)^2 \sigma_x^2 (k^{*} + t + \sqrt{n t}) \ge n/c_3,
\end{aligned}
\end{equation*}
where $c_2 = 8(162e)^2 \sigma_x^2$, $c_3 = 2$, (in our assumptions, $c > 1$ is a large enough constant to make $\sqrt{c} > 16 (162 e)^2 \sigma_x^2$, which leads to a positive $c_3$). 

As is mentioned in Lemma \ref{lem_ridgeeigen}, from Condition \ref{cond:benign}, we have $r_{k^{*}} \ge bn$ and $k^{*} \le n / c_0$ for some constant $c_0 > 0$, then with probability at least $1 - 2 e^{- n / c}$,
\begin{equation*}
    \mu_{k^{*}+1}(A_{- i} + \lambda n I) \le c_1 (\sum_{j > k^{*}} \lambda_j + n \lambda)
\end{equation*}
for any index $i = 1, \dots, \infty$, where $c_1 > 1$ only depends on $b, \sigma_x$. Then for any index $i$, we denote $\mathscr{L}_i$ is the subspace related to the $n - k$ eigenvalues of $A_{-i} + n \lambda I$, and obtain 
\begin{equation*}
    z_i^T (A_{-i} + n \lambda I)^{-1} z_i \ge (\bm{\Pi}_{\mathscr{L}_i}z_i)^T (A_{-i} + n \lambda I)^{-1} (\bm{\Pi}_{\mathscr{L}_i} z_i),
\end{equation*}
then by Lemma \ref{lem_subspacenorm}, with probability at least $1 - 5 e^{- n / c }$, we have
\begin{equation}\label{eq:l0}
    z_i^T (A_{-i} + n \lambda I)^{-1} z_i \ge \frac{\parallel \bm{\Pi}_{\mathscr{L}_i} z_i \parallel^2}{\mu_{k^{*} + 1} (A_{-i} + n \lambda)} \ge \frac{n}{c_3 c_1 (\sum_{j > k^{*}} \lambda_j + n \lambda)}, 
\end{equation}
in which $c_3$ is a constant just depending on $c, \sigma_x$, and $c_1$ just depends on $b, \sigma_x$. The first inequality is from $a^T A a^T \ge \parallel a \parallel_2^2 \mu_n(A)$; the second inequality is from the bounds for eigenvalues and vector norms in Lemma \ref{lem_subspacenorm} and \ref{lem_ridgeeigen}.
Due to this, 
\begin{equation}\label{eq:l1}
    1 + z_i^T (A_{-i} + n \lambda I)^{-1} z_i \le \big( \frac{c_1 c_3 (\sum_{j > k^{*}} \lambda_j + n \lambda)}{n \lambda_i} + 1 \big) \lambda_i z_i^T (A_{-i} + n \lambda I)^{-1} z_i.
\end{equation}
and on the other hand, 
\begin{equation}\label{eq:l2}
z_i^T (A_{-i} + n \lambda I)^{-2} z_i \ge \frac{1}{\parallel z_i \parallel^2} (z_i^T (A_{-i} + n \lambda I)^{-1} z_i)^2 \ge \frac{(z_i^T (A_{-i} + n \lambda I)^{-1} z_i)^2}{c_2 n},
\end{equation}
in which $c_2$ is a constant just depending on $\sigma_x$. The first inequality is from Cauthy-Schwarz, and the second inequality is from the upper bound of $\parallel z_i \parallel_2^2$ in Lemma \ref{lem_subspacenorm}.

Considering both Eq.\eqref{eq:l1} and \eqref{eq:l2}, for any index $i = 1, \dots, \infty$, with probability at least $1 - 5 e^{- n / c}$ we can get a lower bound as
\begin{equation}\label{eq:stdv_decl}
\begin{aligned}
\frac{\lambda_i^2 z_i^T (A_{-i} + n \lambda I)^{-2} z_i}{[1 + \lambda_i z_i^T (A_{-i} + n \lambda I)^{-1} z_i]^2} & \ge \big( \frac{c_1 c_3 (\sum_{j > k^{*}} \lambda_j + n \lambda)}{n \lambda_i} + 1 \big)^{-2} \frac{\lambda_i^2 z_i^T (A_{-i} + n \lambda I)^{-2} z_i}{(\lambda_i z_i^T (A_{-i} + n \lambda I)^{-1} z_i)^2}\\
& \ge \big( \frac{ (\sum_{j > k^{*}} \lambda_j + n \lambda)}{n \lambda_i} + 1 \big)^{-2} \frac{1}{c_1^2 c_3^2 c_2 n} > 0,
\end{aligned}
\end{equation}
Then we turn to the whole trace term \eqref{eq:stdv_dec}, due to Lemma \ref{lem_sum}, with probability at least $1 - 10 e^{- n / c }$, we have
\begin{equation*}
\begin{aligned}
\trace \{ X \Sigma X^T (X X^T + \lambda n I)^{-2} \} &= \sum_i \frac{\lambda_i^2 z_i^T (A_{-i} + n \lambda I)^{-2} z_i}{[1 + \lambda_i z_i^T (A_{-i} + n \lambda I)^{-1} z_i]^2}\\
&\ge \frac{1}{2 c_1^2 c_3^2 c_2 n} \sum_i \big( \frac{ (\sum_{j > k^{*}} \lambda_j + n \lambda)}{n \lambda_i} + 1 \big)^{-2}\\
& \ge \frac{1}{18 c_1^2 c_3^2 c_2 n} \sum_i \min \{ 1, \frac{n^2 \lambda_i^2}{(\sum_{j>k^{*}} \lambda_j)^2}, \frac{\lambda_i^2}{\lambda^2} \}\\
& \ge \frac{1}{18 c_1^2 c_3^2 c_2  b^2 n} \sum_i \min \{ 1, (\frac{bn}{r_{k^{*}}})^2 \frac{\lambda_i^2}{\lambda_{k^{*}+1}^2}, \frac{b^2 \lambda_i^2}{\lambda^2} \},
%& \ge \frac{1}{9 c_{11} b^2 n} \sum_{i > k^{*}} \min \{ 1, (\frac{bn}{r_{k^{*}}(\Sigma)})^2 \frac{\lambda_i^2}{\lambda_{k^{*}+1}^2}, \frac{b^2 \lambda_i^2}{\lambda^2} \},
\end{aligned}
\end{equation*}
in which the first inequality is from Eq.\eqref{eq:stdv_decl}; the second inequality is from
\begin{equation*}
  (a + b + c)^{-2} \ge (3 \max\{ a, b, c \})^{-2} = \frac{1}{9} \min \{ a^{-2}, b^{-2}, c^{-2} \},  
\end{equation*}
and the third inequality is just some bounds relaxing on constant level.  From Condition \ref{cond:benign}, we know that $bn / r_{k^{*}} \le 1$, as well as $\frac{\lambda_i}{\lambda_{k^{*} + 1}} \le 1$ for any index $i > k^{*}$, then we can further obtain
\begin{equation}\label{eq:stdvl}
\begin{aligned}
& \quad \trace \{ X \Sigma X^T (X X^T + \lambda n I)^{-2} \}\\
&\ge \frac{1}{18 c_1^2 c_3^2 c_2 b^2 n} \sum_{i =1}^{k^{*}} \min \{ 1, (\frac{bn}{r_{k^{*}}})^2 \frac{\lambda_i^2}{\lambda_{k^{*}+1}^2}, \frac{b^2 \lambda_i^2}{\lambda^2} \} + \frac{1}{18 c_1^2 c_3^2 c_2 b^2 n} \sum_{i > k^{*}} \min \{ (\frac{bn}{r_{k^{*}}})^2 \frac{\lambda_i^2}{\lambda_{k^{*}+1}^2}, \frac{b^2 \lambda_i^2}{\lambda^2} \}\\
& \ge \frac{1}{18 c_1^2 c_3^2 c_2 b^2 n} \sum_{i =1}^{k^{*}} \min \{ 1, \frac{b^2 n^2 \lambda_i^2}{(\lambda_{k^{*}+1} r_{k^{*}})^2 + n^2 \lambda^2} \} + \frac{1}{18 c_1^2 c_3^2 c_2} \sum_{i > k^{*}}  \frac{ n \lambda_i^2}{(\lambda_{k^{*}+1} r_{k^{*}})^2 + n^2 \lambda^2} \\
& = \frac{1}{18 c_1^2 c_3^2 c_2 b^2 n} \sum_{i =1}^{k^{*}} \min \{ 1, \frac{b^2 n^2 \lambda_i^2}{(\lambda_{k^{*}+1} r_{k^{*}})^2 + n^2 \lambda^2} \}
+ \frac{n \sum_{i > k^{*}} \lambda_i^2}{18 c_1^2 c_3^2 c_2 [(\lambda_{k^{*}+1} r_{k^{*}})^2 + n^2 \lambda^2]}\\
& \ge \frac{1}{18 c_1^2 c_3^2 c_2 b^2 n} \sum_{i =1}^{k^{*}} \min \{ 1, \frac{b^2 n^2 \lambda_i^2}{(\lambda_{k^{*}+1} r_{k^{*}})^2 + n^2 \lambda^2} \} + \frac{n \sum_{i > k^{*}} \lambda_i^2}{18 c_1^2 c_3^2 c_2 (\lambda_{k^{*}+1} r_{k^{*}} + n \lambda)^2}.
\end{aligned}
\end{equation}
The second inequality is from $\min\{ 1/a, 1/b \} \ge 1/(a + b)$, and the last inequality is from the fact $a^2 + b^2 \le (a + b)^2$ for positive $a, b$.

More specifically, if we consider $n \lambda \ge \lambda_{k^{*}+1} r_{k^{*}}$, then it is not harmful to take lower bounds as
\begin{equation*}
\frac{1}{(\lambda_{k^{*}+1} r_{k^{*}})^2 + n^2 \lambda^2} \ge \frac{1}{2 n^2 \lambda^2},
\end{equation*}
Based on this, we have
\begin{equation}\label{eq:stdvl2}
\begin{aligned}
& \quad \trace \{ X \Sigma X^T (X X^T + \lambda n I)^{-2} \}\\
& \ge \frac{1}{18 c_1^2 c_3^3 c_2 b^2 n} \sum_{i =1}^{k^{*}} \min \{ 1, \frac{b^2 n^2 \lambda_i^2}{(\lambda_{k^{*}+1} r_{k^{*}})^2 + n^2 \lambda^2} \} + \frac{n \sum_{i > k^{*}} \lambda_i^2}{18 c_1^2 c_3^2 c_2 (\lambda_{k^{*}+1} r_{k^{*}} + n \lambda)^2}\\
& \ge \frac{1}{36 c_1^2 c_3^2 c_2 b^2 n} \sum_{i =1}^{k^{*}} \min \{ 1, \frac{b^2 n^2 \lambda_i^2}{ n^2 \lambda^2} \} + \frac{n \sum_{i > k^{*}} \lambda_i^2}{36 c_1^2 c_3^2 c_2  n^2 \lambda^2}\\
& \ge \frac{1}{36 c_1^2 c_3^2 c_2 b^2 n} \sum_{i=1}^{k^{*}} \min\{ 1, \frac{\lambda_i^2}{\lambda^2} \} + \frac{ \sum_{i > k^{*}} \lambda_i^2}{36 c_1^2 c_3^2 c_2 n \lambda^2}\\
& \ge \frac{1}{36 c_1^2 c_3^2 c_2 b^2 n} \left( \sum_{\lambda_i > \lambda} 1 + \sum_{\lambda_i \le \lambda} \frac{\lambda_i^2}{\lambda^2} \right).
\end{aligned}
\end{equation}

Then as for term $\sriskb$, we can still take the same decomposition for $\Sigma$ and $X X^T$ as
\begin{equation*}
    \Sigma = \bm{V} \bm{\Lambda} \bm{V}^T, \quad X = \bm{Z} \bm{\Lambda}^{1/2} \bm{V}^T, 
\end{equation*}
in which $\bm{Z} \in \mathbb{R}^{n \times p}$ takes i.i.d. elements, and for convenience, we denote $\tilde{\truth} = V^T \truth$. So we can get
\begin{equation*}
\begin{aligned}
\sriskb &= \mathbb{E} \truth^T [I - X^T (XX^T + n \lambda I)^{-1} X] \Sigma [I - X^T (XX^T + n \lambda I)^{-1} X] \truth \\
&= \mathbb{E} \truth^T V V^T [I - V \Lambda^{1/2} Z^T (XX^T + n \lambda I)^{-1} X] \Sigma [I - X^T (XX^T + n \lambda I)^{-1} Z \Lambda^{1/2} V] V^T V \truth \\
& = \sum_i \tilde{\truth}_i^2 \left( \lambda_i (1 - \lambda_i z_i^T (XX^T + n \lambda I)^{-1}z_i )^2 + \sum_{j \ne i} \lambda_i \lambda_j^2 (z_i^T (XX^T + n\lambda I)^{-1} z_j)^2 \right)\\
& \ge  \sum_i \tilde{\truth}_i^2 \lambda_i (1 - \lambda_i z_i^T (XX^T + n \lambda I)^{-1}z_i )^2\\
&=  \sum_i \tilde{\truth}_i^2 \frac{\lambda_i}{(1 + \lambda_i z_i^T (A_{-i} + n \lambda I)^{-1} z_i)^2} \ge \mathbb{E} \sum_i \tilde{\truth}_i^2 \frac{\lambda_i}{(1 + \frac{\lambda_i \parallel z_i \parallel_2^2}{\mu_n(A_{-i}) + n \lambda})^2},
\end{aligned}
\end{equation*}
the first inequality is from ignoring the second part on the line above it, the second inequality  is from $a^T A a \le \mu_1(A) \parallel a \parallel_2^2$, and the equality on the last line is from Woodbury identity. 

Still considering the case $\lambda \ge \lambda_{k^{*} + 1} r_{k^{*}}(\Sigma) / n$,  for each index $i = 1, \dots, p$, we can take an lower bound as
\begin{equation*}
    \mu_n(A_{-i}) + n \lambda \ge n \lambda,
\end{equation*}
which implies that with probability at least $1 - 5 e^{- n / c}$, we have
\begin{equation*}
    \frac{\lambda_i}{(1 + \frac{\lambda_i \parallel z_i \parallel_2^2}{\mu_n(A_{-i}) + n \lambda})^2} \ge \frac{\lambda_i }{(1 + \frac{c_1 c_2 \lambda_i n}{ n \lambda})^2} \ge 
\frac{1}{c_1^2 c_2^2} \frac{\lambda_i }{(1 + \frac{\lambda_i}{\lambda})^2} \ge  \frac{1}{4 c_1^2 c_2^2} \min \{ \lambda_i, \frac{\lambda^2}{\lambda_i} \},
\end{equation*}
the last inequality is from 
\begin{equation*}
  (a + b )^{-2} \ge (2 \max\{ a, b \})^{-2} = \frac{1}{4} \min \{ a^{-2}, b^{-2} \}.  
\end{equation*}
So with probability at least $1 - 10 e^{- n / c}$, term $\sriskb$ can be lower bounded as
\begin{equation}\label{eq:stdbl2}
\sriskb \ge \sum_i \tilde{\truth}_i^2 \frac{\lambda_i}{(1 + \frac{\lambda_i \parallel z_i \parallel_2^2}{\mu_n(A_{-i}) + n \lambda})^2} \ge 
\frac{1}{4 c_1^2 c_2^2} \left( \sum_{\lambda_i > \lambda} \tilde{\truth}_i^2 \frac{\lambda^2}{\lambda_i} + \sum_{\lambda_i \le \lambda} \tilde{\truth}_i^2 \lambda_i \right),
\end{equation}
in which $c_2$ only depends on $\sigma_x$.
Combining the result in Eq.\eqref{eq:stdvl2} and \eqref{eq:stdbl2}, we can get the lower bound for excessive standard risk while $n \lambda \ge \lambda_{k^{*} + 1} r_{k^{*}}$.

\subsection{Parameter Norm}

From \eqref{eq:arisk_bound},
the gap between excessive standard risk and adversarial risk can be bounded as
\begin{equation*}
    r^2 \mathbb{E} \parallel \ridge \parallel_2^2 \le \arisk - \srisk \le \srisk + 2r^2 \mathbb{E} \parallel \ridge \parallel_2^2,
\end{equation*}
so the estimation for parameter norm measures the adversarial robustness essentially.
The method to dealing with parameter norm is similar to the process in stanadrd risk estimation process. To be specific, still reviewing the bounds in Eq.\eqref{eq:l1}, \eqref{eq:l2} and \eqref{eq:stdv_decl}, i.e, for any index $i$, with a high probability, we have
\begin{equation*}
\begin{aligned}
& 1 + z_i^T (A_{-i} + n \lambda I)^{-1} z_i \le \big( \frac{c_1 c_3 (\sum_{j > k} \lambda_j + n \lambda_{k+1} + n \lambda)}{n \lambda_i} + 1 \big) \lambda_i z_i^T (A_{-i} + n \lambda I)^{-1} z_i,\\
& z_i^T (A_{-i} + n \lambda I)^{-2} z_i \ge \frac{1}{\parallel z_i \parallel^2} (z_i^T (A_{-i} + n \lambda I)^{-1} z_i)^2 \ge \frac{(z_i^T (A_{-i} + n \lambda I)^{-1} z_i)^2}{c_2 n},\\
& \frac{\lambda_i^2 z_i^T(A_{-i} +\lambda n I)^{-2} z_i}{(1 + \lambda_i z_i^T(A_{-i} + \lambda n I)^{-1}z_i)^2} \ge \frac{1}{c_1^2 c_3^2 c_2 n} \big( \frac{ (\sum_{j > k^{*}} \lambda_j + n \lambda)}{n \lambda_i} + 1 \big)^{-2},
\end{aligned} 
\end{equation*}
also, with Lemma \ref{lem_subspacenorm}, for each index $i$, with probability at least $1 - e^{- n/c}$, we have
\begin{equation*}
    \parallel z_i \parallel_2^2 \ge \parallel \Pi_{\mathscr{L}_i} z_i \parallel_2^2 \ge n /c_3,
\end{equation*}
then focusing on $\normv$, according to Lemma~\ref{lem_sum}, with probability at least $1 - 10 e^{- \frac{n}{c}}$, we will obtain
\begin{equation*}
\begin{aligned}
\trace \{ X X^T (X X^T + \lambda n I)^{-2} \} &= \sum_i \lambda_i z_i^T (A + \lambda n I)^{-2} z_i\\
& = \sum_i \frac{\lambda_i z_i^T(A_{-i} +\lambda n I)^{-2} z_i}{(1 + \lambda_i z_i^T(A_{-i} + \lambda n I)^{-1}z_i)^2}\\
&= \sum_i \frac{1}{\lambda_i} \frac{\lambda_i^2 z_i^T(A_{-i} +\lambda n I)^{-2} z_i}{(1 + \lambda_i z_i^T(A_{-i} + \lambda n I)^{-1}z_i)^2}\\
&\ge \frac{1}{2 c_1^2 c_3^2 c_2 n} \sum_i \frac{1}{\lambda_i} \big( \frac{ (\sum_{j > k^{*}} \lambda_j + n \lambda)}{n \lambda_i} + 1 \big)^{-2}\\
& \ge \frac{1}{18 c_1^2 c_3^2 c_2 n} \sum_i \frac{1}{\lambda_i} \min \{ 1, \frac{n^2 \lambda_i^2}{(\sum_{j>k^{*}} \lambda_j)^2}, \frac{\lambda_i^2}{\lambda^2} \}\\
& \ge \frac{1}{18 c_1^2 c_3^2 c_2 b^2 n} \sum_i \frac{1}{\lambda_i} \min \{ 1, (\frac{bn}{r_{k^{*}}})^2 \frac{\lambda_i^2}{\lambda_{k^{*}+1}^2}, \frac{b^2 \lambda_i^2}{\lambda^2} \}.
\end{aligned}
\end{equation*}
The second inequality is from
\begin{equation*}
  (a + b + c)^{-2} \ge (3 \max\{ a, b, c \})^{-2} = \frac{1}{9} \min \{ a^{-2}, b^{-2}, c^{-2} \},  
\end{equation*}
and the third inequality is just some bounds relaxing on constant level.

As $r_{k^{*}} \ge b n$ and for $i > k^{*}$, $\frac{\lambda_i}{\lambda_{k^{*} + 1}} \le 1$, we further obtain
\begin{equation}\label{eq:normvl}
\begin{aligned}
&\trace \{ X X^T (X X^T + \lambda n I)^{-2} \} \\
& \ge \frac{1}{18 c_1^2 c_3^2 c_2 b^2 n} \sum_{i \le k^{*}} \frac{1}{\lambda_i} \min \{ 1, (\frac{bn}{r_{k^{*}}})^2 \frac{\lambda_i^2}{\lambda_{k+1}^2}, \frac{b^2 \lambda_i^2}{\lambda^2} \} + \frac{1}{18 c_1^2 c_3^2 c_2 b^2 n} \sum_{i >  k^{*}} \frac{1}{\lambda_i} \min \{ (\frac{bn}{r_{k^{*}}})^2 \frac{\lambda_i^2}{\lambda_{k+1}^2}, \frac{b^2 \lambda_i^2}{\lambda^2} \}\\
& \ge \frac{1}{18 c_1^2 c_3^2 c_2 b^2} \sum_{i \le k^{*}} \frac{1}{\lambda_i} \min \{ \frac{1}{n}, \frac{b^2 n \lambda_i^2}{r_{k^{*}}^2 \lambda_{k^{*} + 1}^2 + n^2 \lambda^2}  \} + \frac{1}{18 c_1^2 c_3^2 c_2 b^2 n} \sum_{i > k^{*}} \frac{1}{\lambda_i} \frac{b^2 n^2 \lambda_i^2}{r_{k^{*}}^2 \lambda_{k^{*} + 1}^2 + n^2 \lambda^2}\\
& = \frac{1}{18 c_1^2 c_3^2 c_2 b^2} \sum_{i \le k^{*}} \frac{1}{\lambda_i} \min \{ \frac{1}{n}, \frac{b^2 n \lambda_i^2}{r_{k^{*}}^2 \lambda_{k^{*} + 1}^2 + n^2 \lambda^2}  \} + \frac{1}{18 c_1^2 c_3^2 c_2} \frac{n \lambda_{k^{*}+1} r_{k^{*}}}{r_{k^{*}}^2 \lambda_{k^{*} + 1}^2 + n^2 \lambda^2}\\
& \ge \frac{1}{18 c_1^2 c_3^2 c_2 b^2} \sum_{i \le k^{*}} \frac{1}{\lambda_i} \min \{ \frac{1}{n}, \frac{b^2 n \lambda_i^2}{(r_{k^{*}} \lambda_{k^{*} + 1} + n \lambda)^2}  \} + \frac{1}{18 c_1^2 c_3^2 c_2} \frac{n \lambda_{k^{*}+1} r_{k^{*}}}{(r_{k^{*}} \lambda_{k^{*} + 1} + n \lambda)^2},
%& \ge \frac{\tilde{c}}{n} \min\{ 1, \frac{1}{\lambda^2} \} + \frac{n \lambda_{k^{*}+1} r_{k^{*}}(\Sigma)}{c_4 (r_k(\Sigma)^2 \lambda_{k^{*} + 1}^2 + n^2 \lambda^2)}\\
%& \ge \frac{\tilde{c}}{n} \min\{ 1, \frac{1}{\lambda^2} \} + \frac{n \lambda_{k^{*}+1} r_{k^{*}}(\Sigma)}{c_4 (r_k(\Sigma) \lambda_{k^{*} + 1} + n \lambda)^2},
\end{aligned}
\end{equation}
The first inequality is due to the fact that as $i \ge k^{*}$, we have
\begin{equation*}
    (\frac{bn}{r_{k^{*}}})^2 \frac{\lambda_i^2}{\lambda_{k^{*}+1}^2} \le 1,
\end{equation*}
and the second inequality is using $a^2 + b^2 \le (a + b)^2$ as $a, b > 0$.

Specifically, we still consider two cases. First, if $n \lambda \ge \lambda_{k^{*}+1} r_{k^{*}}$, we can take a lower bound as
\begin{equation*}
    \frac{1}{(n\lambda + \lambda_{k^{*} + 1} r_{k^{*}})^2} \ge \frac{1}{4 n^2 \lambda^2},
\end{equation*}
which implies a lower bound for the term $\normv$ as
\begin{equation}\label{eq:normvl2}
\begin{aligned}
\normv &\ge \frac{1}{18 c^2 c_3^2 c_2 b^2} \sum_{i \le k^{*}} \frac{1}{\lambda_i} \min \{ \frac{1}{n}, \frac{b^2 n \lambda_i^2}{(r_{k^{*}} \lambda_{k^{*} + 1} + n \lambda)^2}  \} + \frac{1}{18 c^2 c_3^2 c_2} \frac{n \lambda_{k^{*}+1} r_{k^{*}}}{(r_{k^{*}} \lambda_{k^{*} + 1} + n \lambda)^2}\\
&\ge \frac{1}{72 c^2 c_3^2 c_2 b^2} \sum_{i \le k^{*}} \frac{1}{\lambda_i} \min\{ \frac{1}{n}, \frac{b^2 \lambda_i^2}{n \lambda^2} \} + \frac{1}{72 c^2 c_3^2 c_2} \frac{\lambda_{k^{*} + 1} r_{k^{*}}}{n \lambda^2}\\
&\ge \frac{1}{72 c^2 c_3^2 c_2 b^2} \left( \sum_{\lambda_i > \lambda} \frac{1}{n \lambda_i} + \sum_{\lambda_i \le \lambda} \frac{\lambda_i}{n \lambda^2} \right).
\end{aligned}
\end{equation}
The first inequality is from Eq.\eqref{eq:normvl}, the inequality above Eq.\eqref{eq:normvl2} implies the second inequality, and the third inequality is from the choice of minimum value for each index.

On the other hand, if we take the regularization parameter $\lambda$ small enough, i.e, $n \lambda \le \lambda_{k^{*} + 1} r_{k^{*}}$, we can take a similar lower bound as
\begin{equation*}
    \frac{1}{(n\lambda + \lambda_{k^{*} + 1} r_{k^{*}})^2} \ge \frac{1}{4 \lambda_{k^{*}+1}^2 r_{k^{*}}^2},
\end{equation*}
which implies the related lower bound for the term $\normv$ as
\begin{equation}\label{eq:normvl1}
\begin{aligned}
\normv &\ge \frac{1}{18 c_1^2 c_3^2 c_2 b^2} \sum_{i \le k^{*}} \frac{1}{\lambda_i} \min \{ \frac{1}{n}, \frac{b^2 n \lambda_i^2}{(r_{k^{*}} \lambda_{k^{*} + 1} + n \lambda)^2}  \} + \frac{1}{18 c_1^2 c_3^2 c_2} \frac{n \lambda_{k^{*}+1} r_{k^{*}}}{(r_{k^{*}} \lambda_{k^{*} + 1} + n \lambda)^2}\\
&\ge \frac{1}{72 c_1^2 c_3^2 c_2 b^2} \sum_{i \le k^{*}} \frac{1}{\lambda_i} \min\{ \frac{1}{n}, \frac{b^2 n \lambda_i^2}{(\lambda_{k^{*}+1} r_{k^{*}})^2} \} + \frac{1}{72 c_1^2 c_3^2 c_2} \frac{n \lambda_{k^{*} + 1} r_{k^{*}}}{(\lambda_{k^{*}+1} r_{k^{*}}))^2}\\
&\ge \frac{1}{72 c_1^2 c_3^2 c_2 b^2} \sum_{i=1}^{k^{*}} \frac{1}{n \lambda_i} + \frac{1}{72 c_1^2 c_3^2 c_2} \frac{n}{\lambda_{k^{*}+1} r_{k^{*}}},
\end{aligned}
\end{equation}
the analysis is similar to \eqref{eq:normvl2} above, and the last inequality is from the definition of $k^{*}$ in \ref{eq:defk}.

Then similarly, we turn to the estimation for term $\normb$,
\begin{equation*}
\begin{aligned}
\normb &= \mathbb{E} \truth^T X^T (XX^T + n \lambda I)^{-1} XX^T (XX^T + n \lambda I)^{-1} X \truth\\
&= \mathbb{E} \truth V V^T V \Lambda^{1/2} Z^T (XX^T + n \lambda I)^{-1} XX^T (XX^T + n \lambda I)^{-1} Z \Lambda^{1/2} V^T V \truth \\
&= \sum_i \tilde{\truth}_i^2 \lambda_i z_i^T (XX^T + n \lambda I)^{-1} XX^T (XX^T + n \lambda I)^{-1} z_i\\
&= \sum_i \tilde{\truth}_i^2 \lambda_i z_i^T (XX^T + n \lambda I)^{-1} (\sum_j \lambda_j z_j z_j^T ) (XX^T + n \lambda I)^{-1} z_i\\
& = \sum_i \tilde{\truth}_i^2 \lambda_i z_i^T (XX^T + n \lambda I)^{-1} ( \lambda_i z_i z_i^T +  \sum_{j \ne i} \lambda_j z_j z_j^T ) (XX^T + n \lambda I)^{-1} z_i\\
&\ge \sum_i \tilde{\truth}_i^2 \lambda_i^2 ( z_i^T (XX^T + n \lambda I)^{-1} z_i )^2\\
&= \sum_i \tilde{\truth}_i^2 \left( \frac{\lambda_i z_i^T (A_{-i} + n \lambda I)^{-1} z_i}{1 + \lambda_i z_i^T (A_{-i} + n \lambda I)^{-1} z_i}\right)^2,
\end{aligned}
\end{equation*}
in which the inequality is from ignoring the terms with index $j \ne i$ on the line above it, and the equality on the last line is from Woodbury identity.
And considering Eq.\eqref{eq:l0}, for each index $i$, with probability at least $1 - 5 e^{- n / c }$, we have
\begin{equation*}
    z_i^T (A_{-i} + n \lambda I)^{-1} z_i \ge \frac{n}{c_3 c_1 (\sum_{j > k^{*}} \lambda_j + n \lambda)},
\end{equation*}
which implies that
\begin{equation*}
    \left( \frac{\lambda_i z_i^T (A_{-i} + n \lambda I)^{-1} z_i}{1 + \lambda_i z_i^T (A_{-i} + n \lambda I)^{-1} z_i}\right)^2 \ge \frac{1}{c_3^2 c_1^2} \left( \frac{n \lambda_i}{n \lambda_i + \sum_{j > k^{*}} \lambda_j + n \lambda} \right)^2,
\end{equation*}
Then according to Lemma \ref{lem_sum}, with probability at least $1 - 10 e^{- n / c}$,  we can estimate the lower bound for $\normb$ as
\begin{equation*}
\begin{aligned}
\normb &\ge \sum_i \tilde{\truth}_i^2 \left( \frac{\lambda_i z_i^T (A_{-i} + n \lambda I)^{-1} z_i}{1 + \lambda_i z_i^T (A_{-i} + n \lambda I)^{-1} z_i}\right)^2\\
&\ge \sum_i \tilde{\truth}_i^2 \frac{1}{2 c_3^2 c_1^2} \left( \frac{n \lambda_i}{n \lambda_i + \sum_{j > k^{*}} \lambda_j + n \lambda} \right)^2\\
&\ge \frac{1}{8 c_3^2 c_1^2} \sum_i \tilde{\truth}_i^2 \min\{ 1, \frac{n^2 b^2 \lambda_i^2}{( \sum_{j > k^*} \lambda_j + n \lambda)^2} \}\\
&= \frac{1}{8 c_3^2 c_1^2 b^2} \sum_{i=1}^{k^{*}} \tilde{\truth}_i^2 \min\{ 1, \frac{n^2 b^2  \lambda_i^2}{( \lambda_{k^{*} + 1} r_{k^{*}} + n \lambda)^2} \} + \frac{1}{8 c_3^2 c_1^2} \frac{n^2 \sum_{j > k^{*}} \tilde{\truth}_j^2 \lambda_j^2 }{(\lambda_{k^{*} + 1} r_{k^{*}} + n \lambda)^2},
\end{aligned}
\end{equation*}
in which the third inequality is from 
\begin{equation*}
    \frac{1}{(a + b)^2} \ge \frac{1}{4} \min\{ \frac{1}{a^2}, \frac{1}{b^2} \},
\end{equation*}
and the equality on the last line is from the fact that $r_{k^{*}} \ge bn$ and $\lambda_j \le \lambda_{k^{*} + 1}$ for any $j > k^{*}$.

So we can also consider two situations. First, if $n \lambda \le \lambda_{k^{*} + 1} r_{k^{*}}$, we have a lower bound as
\begin{equation*}
    \frac{1}{(n \lambda + \lambda_{k^{*} + 1} r_{k^{*}})^2} \ge \frac{1}{4 \lambda_{k^{*} + 1}^2 r_{k^{*}}^2},
\end{equation*}
then we can obtain
\begin{equation}\label{eq:normbl1}
\begin{aligned}
\normb &\ge \frac{1}{8 c_3^2 c_1^2 b^2} \sum_{i=1}^{k^{*}} \tilde{\truth}_i^2 \min\{ 1, \frac{n^2 b^2  \lambda_i^2}{( \lambda_{k^{*} + 1} r_{k^{*}} + n \lambda)^2} \} + \frac{1}{8 c_3^2 c_1^2} \frac{n^2 \sum_{j > k^{*}} \tilde{\truth}_j^2 \lambda_j^2 }{(\lambda_{k^{*} + 1} r_{k^{*}} + n \lambda)^2}\\
&\ge \frac{1}{32 c_3^2 c_1^2 b^2} \sum_{i=1}^{k^{*}} \tilde{\truth}_i^2 \min\{ 1, \frac{n^2 b^2  \lambda_i^2}{( \lambda_{k^{*} + 1} r_{k^{*}})^2} \} + \frac{1}{32 c_3^2 c_1^2} \frac{n^2 \sum_{j > k^{*}} \tilde{\truth}_j^2 \lambda_j^2 }{(\lambda_{k^{*} + 1} r_{k^{*}} )^2}\\
&= \frac{1}{32 c_3^2 c_1^2 b^2} \sum_{i=1}^{k^{*}} \tilde{\truth}_i^2  + \frac{1}{32 c_3^2 c_1^2} \frac{n^2 \sum_{j > k^{*}} \tilde{\truth}_j^2 \lambda_j^2 }{(\lambda_{k^{*} + 1} r_{k^{*}} )^2},
\end{aligned}
\end{equation}
where the last equality is from the definition of $k^{*}$ in Eq.\eqref{eq:defk}.

Similarly, if $n \lambda \ge \lambda_{k^{*} + 1} r_{k^{*}}$, we have
\begin{equation}\label{eq:normbl2}
\begin{aligned}
\normb &\ge \frac{1}{8 c_3^2 c_1^2 b^2} \sum_{i=1}^{k^{*}} \tilde{\truth}_i^2 \min\{ 1, \frac{n^2 b^2  \lambda_i^2}{( \lambda_{k^{*} + 1} r_{k^{*}} + n \lambda)^2} \} + \frac{1}{8 c_3^2 c_1^2} \frac{n^2 \sum_{j > k^{*}} \tilde{\truth}_j^2 \lambda_j^2 }{(\lambda_{k^{*} + 1} r_{k^{*}} + n \lambda)^2}\\
&\ge \frac{1}{32 c_3^2 c_1^2 b^2} \sum_{i=1}^{k^{*}} \tilde{\truth}_i^2 \min\{ 1, \frac{  \lambda_i^2}{\lambda^2} \} + \frac{1}{32 c_3^2 c_1^2} \frac{ \sum_{j > k^{*}} \tilde{\truth}_j^2 \lambda_j^2 }{\lambda^2}\\
&\ge \frac{1}{32 c_3^2 c_1^2 b^2} \left( \sum_{\lambda_i > \lambda} \tilde{\truth}_i^2 + \sum_{\lambda_i \le \lambda} \frac{\tilde{\truth}_i^2 \lambda_i^2}{\lambda^2} \right),
\end{aligned}
\end{equation}
where the last inequality is due to $b > 1$.

\section{Proof for Corollary \ref{cor_1} and Theorem \ref{thm:main_tradeoff}}\label{pf:trade}

In this part, we would explore the impact of $\lambda$ for both standard risk and parameter norm. And we will begin with the small regularization regime:

(1). \textbf{Small Regularization}: $\lambda \le \lambda_{k^{*}+1} r_{k^{*}} / n$.

In this regime, the regularization parameter $\lambda$ is too small to cause obvious impact for both standard risk and parameter norm, while comparing with min-norm estimator. And with the analysis above, we have the upper bound for $\srisk$ as in \eqref{eq:stdu}:
\begin{equation*}
\begin{aligned}
\srisk / C_1 &\le \sum_{j > k^{*}} \lambda_j \truth_{*j}^2 + \sum_{j=1}^{k^{*}} \frac{\truth_{*j}^2}{\lambda_j} \frac{(n \lambda + \lambda_{k^{*} + 1} r_{k^{*}})^2}{n^2} + \sigma^2 \left( \frac{k^{*}}{n} + \frac{n \sum_{j > k^{*}} \lambda_j^2}{(n \lambda + \lambda_{k^{*}+1} r_{k^{*}})^2} \right)\\
&\le \sum_{j > k^{*}} \lambda_j \truth_{*j}^2 + \sum_{j=1}^{k^{*}} \frac{\truth_{*j}^2}{\lambda_j} \frac{4 \lambda_{k^{*} + 1}^2 r_{k^{*}}^2}{n^2} + \sigma^2 \left( \frac{k^{*}}{n} + \frac{n \sum_{j > k^{*}} \lambda_j^2}{\lambda_{k^{*}+1}^2 r_{k^{*}}^2} \right),
\end{aligned}
\end{equation*}
under Condition \ref{cond:benign}, it tends to zero, which implies that it is a near optimal estimator with respect to $\srisk$.
But as for parameter norm, with the results shown in \eqref{eq:normvl1} and \eqref{eq:normbl1}, we have
\begin{equation*}
\begin{aligned}
\mathbb{E} \parallel \ridge \parallel_2^2 &\ge \frac{1}{32 c_3^2 c_1^2 b^2} \sum_{i=1}^{k^{*}} \tilde{\truth}_i^2  + \frac{1}{32 c_3^2 c_1^2} \frac{n^2 \sum_{j > k^{*}} \tilde{\truth}_j^2 \lambda_j^2 }{(\lambda_{k^{*} + 1} r_{k^{*}} )^2} + \frac{\sigma^2}{72 c_1^2 c_3^2 c_2 b^2} \sum_{i=1}^{k^{*}} \frac{1}{n \lambda_i} + \frac{\sigma^2}{72 c_1^2 c_3^2 c_2} \frac{n}{\lambda_{k^{*}+1} r_{k^{*}}}\\
&\ge \frac{\sigma^2}{72 c_1^2 c_3^2 c_2} \frac{n}{\lambda_{k^{*}+1} r_{k^{*}}},
\end{aligned}
\end{equation*}
 with $\sigma^2 = \omega(\lambda_{k^* + 1} r_{k^*} / n)$, the parameter norm would be large, which leads to a non-robust estimator for adversarial attacks.

(2). \textbf{Large Regularization}: $\lambda \ge \lambda_1$.

In this situation, considering the standard risk, with \eqref{eq:stdvl2} and \eqref{eq:stdbl2}, we have
\begin{equation*}
\begin{aligned}
\srisk &\ge \frac{1}{4 c_1^2 c_2^2} \left( \sum_{\lambda_i > \lambda_1} \tilde{\truth}_i^2 \frac{\lambda^2}{\lambda_i} + \sum_{\lambda_i \le \lambda_1} \tilde{\truth}_i^2 \lambda_i \right) + \frac{\sigma^2}{36 c_1^2 c_3^2 c_2 b^2 n} \left( \sum_{\lambda_i > \lambda_1} 1 + \sum_{\lambda_i \le \lambda_1} \frac{\lambda_i^2}{\lambda^2} \right)\\
&= \frac{1}{4 c_1^2 c_2^2} \sum_i \tilde{\truth}_i^2 \lambda_i + \frac{\sigma^2 \sum_i \lambda_i^2 }{36 \lambda_1^2 c_1^2 c_3^2 c_2 b^2 n} \ge \frac{1}{4 c_1^2 c_2^2} \| \truth \|^2_{\Sigma},
\end{aligned}
\end{equation*}
which implies that the large regularization will induce a standard risk which can not converge to zero. It means that general ridge regression methods with constant level regularization $\lambda$ will not take the estimator effective enough.

(3). \textbf{Intermediate Regularization}: $\lambda_{k^{*}+1} r_{k^{*}} / n \le \lambda \le \lambda_1$.

In this regime, we focus on a special case, in which the norm of parameter $\truth$ has a slow decreasing rate, and the signal-to-noise ratio is not very large (as is mentioned in Condition~\ref{cond:trade} and constrain on $\sigma^2 = \omega(\lambda_{k^* + 1} r_{k^*} / n)$). 

To be specific, the upper bound of $\srisk(\ridge)$ for min-norm estimator is
\begin{equation*}
\srisk / C_1 \le \sum_{j > k^{*}} \lambda_j \truth_{*j}^2 + \sum_{j=1}^{k^{*}} \frac{\truth_{*j}^2}{\lambda_j} \frac{4 \lambda_{k^{*} + 1}^2 r_{k^{*}}^2}{n^2} + \sigma^2 \left( \frac{k^{*}}{n} + \frac{n \sum_{j > k^{*}} \lambda_j^2}{\lambda_{k^{*}+1}^2 r_{k^{*}}^2} \right).
\end{equation*}
Then we turn to the estimator with an intermediate regularization. As is shown in \eqref{eq:normvl2} and \eqref{eq:normbl2}, the lower bound of parameter norm is
\begin{equation*}
    \mathbb{E} \| \ridge \|_2^2 \ge \frac{1}{32 c_3^2 c_1^2 b^2} \left( \sum_{\lambda_i > \lambda} \tilde{\truth}_i^2 + \sum_{\lambda_i \le \lambda} \frac{\tilde{\truth}_i^2 \lambda_i^2}{\lambda^2} \right) + \frac{\sigma^2}{72 n c_1^2 c_3^2 c_2 b^2} \left( \sum_{\lambda_i > \lambda} \frac{1}{\lambda_i} + \sum_{\lambda_i \le \lambda} \frac{\lambda_i}{\lambda^2} \right),
\end{equation*}
and with \eqref{eq:stdvl2} and \eqref{eq:stdbl2}, the lower bound for standard risk is
\begin{equation*}
\srisk \ge \frac{1}{4 c_1^2 c_2^2} \left( \sum_{\lambda_i > \lambda} \tilde{\truth}_i^2 \frac{\lambda^2}{\lambda_i} + \sum_{\lambda_i \le \lambda} \tilde{\truth}_i^2 \lambda_i \right) + \frac{\sigma^2}{36 c_1^2 c_3^2 c_2 b^2 n} \left( \sum_{\lambda_i > \lambda} 1 + \sum_{\lambda_i \le \lambda} \frac{\lambda_i^2}{\lambda^2} \right),
\end{equation*}

With Condition~\ref{cond:trade}, if we have $(\lambda_{k^*+1} r_{k^*})/n \le \lambda \le \lambda_{w^*}$, then
\begin{equation*}
\begin{aligned}
 \srisk(\ridge) \mathbb{E}\| \ridge \|_2^2 &\ge \frac{1}{288 c_1^4 c_2^3 c_3^2 b^2} \frac{\sigma^2}{n \lambda^2} \sum_{\lambda_i \le \lambda} \tilde{\truth}_i^2 \lambda_i \sum_{\lambda_i \le \lambda} \lambda_i \ge \frac{1}{288 c_1^4 c_2^3 c_3^2 b^2} \sigma^2 \| \truth \|_2^2 \sqrt{\max \{ \frac{k^*}{n}, \frac{n}{R_{k^*}} \}},  
\end{aligned}
\end{equation*}
comparing this term with the upper bound of $\srisk(\ridge)$ in min-norm estimator, we can obtain the corresponding result in Corollary~\ref{cor_1}.

Then before the following analysis, we claim a useful lemma first:
\begin{lemma}\label{lem:stable_index}
For $\lambda = \tilde{\lambda}$ be the smallest regularization parameter leading to a stable parameter norm ($\tilde{\lambda}$ can change with the increasing in sample size $n$), in which $\tilde{\lambda} < \lambda_{k^* + 1}$, we can always get the result that
\begin{equation*}
    \lim_{n \to \infty} \frac{\tilde{\lambda}}{\lambda_k^*+1} = \infty.
\end{equation*}
\end{lemma}
\begin{proof}
While we consider $\tilde{\lambda}$ induces a stable parameter norm, we have
\begin{equation}\label{eq:con_stable}
    \lim_{n \to \infty} \frac{\sigma^2}{n \tilde{\lambda}^2} \sum_{\lambda_j \le \tilde{\lambda}} \lambda_j \ne \infty, \quad \lim_{n \to \infty} \frac{\sigma^2}{n \lambda_{k^*+1}^2} \sum_{j > k^*} \lambda_j = \infty.
\end{equation}
Then if the condition $\tilde{\lambda} / \lambda_{k^* + 1} \to \infty$ does not meet, there exists a constant $C > 0$ satisfying 
\begin{equation*}
    \lim_{n \to \infty} \frac{\tilde{\lambda}}{\lambda_{k^*+1}} \le C,
\end{equation*}
we can obtain
\begin{equation*}
   \lim_{n \to \infty} \frac{\sigma^2}{n \tilde{\lambda}^2} \sum_{\lambda_j \ge \tilde{\lambda}} \lambda_j \ge \lim_{n \to \infty} \frac{\sigma^2}{n C^2 \lambda_{k^*+1}^2} \sum_{j \ge k^*+1} \lambda_j = \infty,
\end{equation*}
which contradicts the first equation in Eq.\eqref{eq:con_stable}, so we can draw a conclusion that
\begin{equation*}
    \lim_{n \to \infty} \frac{\tilde{\lambda}}{\lambda_{k^* + 1}} = \infty.
\end{equation*}
\end{proof}
Then from the second condition in Condition~\ref{cond:trade}, $\lambda = \lambda_{w^*}$ can always leads to a stable $\mathbb{E} \| \ridge \|_2^2$, combing with Lemma \ref{lem:stable_index}, $\lambda_{w^*} / \lambda_{k^*+1}$ tends to infinity. Then considering the first condition in Condition~\ref{cond:trade}, as well as the fact that $\sriskb$ will increase with $\lambda$, based on the result
\begin{equation*}
\frac{\sriskb(\ridge |_{\lambda = \lambda_{w^*}})}{\| \truth \|_2^2 \srisk(\ridge |_{\lambda=0})} \ge c_4 \min \left\{ \frac{\lambda_{w^*}^2 \sum_{\lambda_i > \lambda_{w^*}} \tilde{\truth}_i^2 / \lambda_i + \sum_{\lambda_i \le \lambda_{w^*}} \tilde{\truth}_i^2 \lambda_i}{\| \truth \|_2^2 (\lambda_{k^*+1}^2 \| \truth_{0:k^*} \|^2_{\Sigma_{0:k^*}^{-1}} + \| \truth_{k^*:\infty} \|^2_{\Sigma_{k^*:\infty}} )} , \frac{1}{\mathbb{E} \| \ridge \|_2^2 \sqrt{\max\{ k^*/n, n/R_{k^*} \}}} \right\},
\end{equation*}
in which $c_4 = \min\{ C_1 / (4 c_1^2 c_2^2), C_1/ (288 c_1^4 c_2^3 c_3^2 b^2) \}$ is a constant only depending on $b, \sigma_x$, we can draw a conclusion that with enough sample size $n$, while $\lambda_{w^*} \le \lambda \le \lambda_1$, $\srisk \ge \| \truth \|_2^2 \srisk(\ridge |_{\lambda=0})$, as both two terms on the right hand side tends to infinity. So in this regime, we reveal that under large enough sample size $n$, with a high probability, the near optimal standard risk convergence rate and stable adversarial risk can not be obtained at the same time.

By considering the results for all regimes together, we obtain the conclusion stated in Theorem~\ref{thm:main_tradeoff},  which implies that to get a stable adversarial risk, there must be corresponding loss in convergence rate in standard risk.

\section{Proof for Theorem \ref{thm_ntk}}\label{pf:ntk}

The proof consists of three steps. First, we take a linear approximation for NTK kernel $K = \nabla F \nabla F^T$ as $m \to \infty$; next, we take asymptotic expressions for standard risk $\srisk (\hat{w})$ and Lipschitz norm $ \mathbb{E} \| \nabla_x f_{NTK}(\hat{w}, x) \|^2$; finally, the upper and lower bounds of $\srisk (\hat{w})$ and  $\mathbb{E} \| \nabla_x f_{NTK}(\hat{w}, x) \|^2$ are calculated respectively.

\textbf{Step 1: kernel matrix linearization.} Recalling Lemma \ref{lem:ntk1} and \ref{lem:ntk2} in \citet{jacot2018neural}, with Condition \ref{cond:ntk_high-dim}, the kernel matrix $K = \nabla F \nabla F^T \in \rR^{n \times n}$ will have components as 
\begin{equation}\label{eq:kernel_orig}
\begin{aligned}
    K_{i,j} = K(x_i, x_j) &= \nabla_w f_{NTK}(w_0, x_i)^T \nabla_w f_{NTK} (w_0, x_j)\\
    &= \frac{x_i^T x_j}{\pi p} \arccos\left( - \frac{x_i^T x_j}{\| x_i \| \| x_j \|} \right) + \frac{\| x_i \| \| x_j \|}{2 \pi p} \sqrt{1 - \left( \frac{x_i^T x_j}{\| x_i \| \| x_j \|} \right)^2} + o_p(\frac{1}{\sqrt{m}}),
\end{aligned}
\end{equation}
here we define a temporary function $t_{ij}(z)$ as:
\begin{equation*}
    t_{i,j}(z) := \frac{x_i^T x_j}{\pi l} \arccos \left( - \frac{x_i^T x_j}{l z} \right) + \frac{z}{2 \pi} \sqrt{1 - \left( \frac{x_i^T x_j}{l z} \right)^2},
\end{equation*}
which has a uniformal bounded Lipschitz as:
\begin{equation*}
| t'_{i,j}(z) = | \frac{3}{2 \pi} \frac{(x_i^T x_j / l z)^2}{\sqrt{1 - (x_i^T x_j /l z)^2}} + \frac{1}{2 \pi} \sqrt{1 - (x_i^T x_j /l z)^2} | \le \frac{2}{\pi}, 
\end{equation*}
and the kernel matrix $K$ can be approximated by a new kernel $K'$ which has components $K'_{i,j} = (l / p) t_{i,j}(1)$, due to the following fact
\begin{align*}
\| p/l K - p/l K' \|_2 &= \max_{\beta \in \mathbb{S}^{n-1}} \beta^T (p/l K - p/l K') \beta = \max_{\beta \in \mathbb{S}^{n-1}} \sum_{i,j} \beta_i \beta_j  (t_{i,j}(\frac{\| x_i \| \| x_j \|}{l}) - t_{i,j}(1) + o_p(\frac{p}{l \sqrt{m}}))\\
&\le \frac{2}{\pi} \max_{\beta \in \mathbb{S}^{n-1}} \sum_{i,j} \beta_i \beta_j \left| \frac{\| x_i \| \| x_j \|}{l} - 1 \right| + o_p(\frac{np}{l\sqrt{m}})\\
& \le \frac{2}{\pi} \max_{i,j} \left| \frac{\| x_i \| \| x_j \|}{l} - 1 \right| \cdot  \max_{\beta \in \mathbb{S}^{n-1}} \sum_{i,j} \beta_i \beta_j + o_p(\frac{np}{l\sqrt{m}}) \\
& = \frac{2}{\pi} \max_{i} \left| \frac{\| x_i \|_2^2}{l} - 1 \right| \cdot  \max_{\beta \in \mathbb{S}^{n-1}} \sum_{i,j} \beta_i \beta_j + o_p(\frac{np}{l \sqrt{m}}) \\
& \le \frac{2 n}{\pi} \max_{i} \left| \frac{\| x_i \|_2^2}{l} - 1 \right| + o_p(\frac{np}{l \sqrt{m}}), 
\end{align*}
where the first inequality is due to the bounded Lipschitz norm of $t_{i,j}(z)$ and the fact that $\beta \in \mathbb{S}^{n-1}$, and the last inequality is from Cauthy-Schwarz inequality:
\begin{equation*}
    \sum_{i,j} \beta_i \beta_j \le \sqrt{\sum_{i,j} \beta_i^2 } \sqrt{\sum_{i,j} \beta_j^2} = n \sum_i \beta_i^2 = n. 
\end{equation*}
Then under Condition \ref{cond:ntk_benign} and \ref{cond:ntk_high-dim}, considering the concentration inequality for input data, for any fixed index $i = 1, \dots, n$, with probability at least $1 - 2 n e^{- t^2 l^2 / 2 r_0(\varSigma^2) }$, we could obtain that
\begin{equation*}
  \max_{i = 1, \dots, n} \left| \frac{\| x_i \|_2^2}{l} - 1 \right| \le t,
\end{equation*}
under Condition~\ref{cond:ntk_high-dim}, as $r_0(\varSigma^2) \le r_0(\varSigma) = l$, choosing $t = n^{- 5/16}$, we have $t^2 l^2/ r_0(\varSigma^2) \ge l n^{-5/8} \ge n^{1/8}$, so with probability at least $1 - 2n e^{- n^{1/8} / 2}$, we can get
\begin{equation*}
    \| K - K' \|_2 \le \frac{2n^{11/16} }{p \pi} + o_p(\frac{n}{\sqrt{m}}) = o(\frac{l}{p}),
\end{equation*}
where the last equality is from Condition \ref{cond:ntk_high-dim}.
So it doesn't matter to replace kernel matrix $K$ by $K'$. Further, if we denote a function $g: \rR \to \rR$ as:
\begin{equation*}
    g(z) := \frac{z}{\pi l} \arccos ( - \frac{z}{l} ) + \frac{1}{2 \pi} \sqrt{1 - (\frac{z}{l})^2},
\end{equation*}
the components of matrix $K'$ can be expressed as $K'_{i,j} = \frac{l}{p} g(x_i^T x_j)$,
then with a refinement of \citet{el2010spectrum} in Lemma \ref{lem:spectrum} , with probability at least $1 - 4 n^2 e^{- n^{1/8} / 2}$, we have the following approximation:
\begin{equation*}
    \| K' - \tilde{K} \|_2 = o(\frac{l}{p n^{1/16}}),
\end{equation*}
in which 
\begin{equation}\label{eq:est_kernel}
    \tilde{K} = \frac{l}{p} (\frac{1}{2 \pi} + \frac{3 r_0(\varSigma^2)}{4 \pi l^2}) 11^T + \frac{1}{2 p} XX^T + \frac{l}{p} (\frac{1}{2} - \frac{1}{2 \pi}) I_n,
\end{equation}
as $\| \tilde{K} \|_2 \ge \frac{l}{p}(\frac{1}{2} - \frac{1}{2 \pi})$, we can approximate $K$ by $\tilde{K}$ in following calculations.

\textbf{Step 2: asymptotic standard risk estimation. } With the solution in Eq.\eqref{eq:ntk_para}, the excessive standard risk can be expressed as
\begin{align*}
& \quad  \srisk(\hat{w})\\
&= \mathbb{E}_{x, \epsilon} [\nabla_w f_{NTK}(w_0, x)^T (\hat{w} - w_*)]^2\\
& = \mathbb{E}_{ x, \epsilon} \{ \nabla_w f_{NTK}(w_0, x)^T [ (\nabla F^T (\nabla F \nabla F^T)^{-1} \nabla F - I)(w_* - w_0) + \nabla F^T (\nabla F \nabla F^T)^{-1} \epsilon ] \}^2\\
& = \mathbb{E}_{x}  (w_* - w_0)^T (I - \nabla F^T (\nabla F \nabla F^T)^{-1} \nabla F ) (\nabla_w f_{NTK} (w_0,x) \nabla_w f_{NTK} (w_0, x)^T - \frac{1}{n} \nabla F^T \nabla F ) \\
& \quad \quad \quad \quad \quad \quad \quad \quad ( I - \nabla F^T (\nabla F \nabla F^T)^{-1} \nabla F ) (w_* - w_0) \\
& \quad + \sigma^2 \mathbb{E}_{ x} \text{tr} \{ (\nabla F \nabla F^T)^{-1} \nabla F \nabla_w f_{NTK}(w_0, x) \nabla_w f_{NTK} (w_0, x)^T \nabla F^T (\nabla F \nabla F^T)^{-1} \}\\
& \le \mathbb{E}_{ x}  \|  w_* - w_0 \|_2^2  \| (I - \nabla F^T (\nabla F \nabla F^T)^{-1} \nabla F) \|_2^2 \| \nabla_w f_{NTK} (w_0,x) \nabla_w f_{NTK} (w_0, x)^T \|_2 \\
& \quad + \sigma^2 \mathbb{E}_{ x} \text{tr} \{ (\nabla F \nabla F^T)^{-1} \nabla F \nabla_w f_{NTK}(w_0, x) \nabla_w f_{NTK} (w_0, x)^T \nabla F^T (\nabla F \nabla F^T)^{-1} \}\\
& \le \underbrace {R^2 \mathbb{E}_{ x} \| \nabla_w f_{NTK} (w_0,x) \nabla_w f_{NTK} (w_0, x)^T - \frac{1}{n} \nabla F^T \nabla F \|_2}_{\sriskb}\\
& \quad + \underbrace{\sigma^2 \mathbb{E}_{ x } \text{tr} \{ K^{-2} \nabla F \nabla_w f_{NTK}(w_0,x) \nabla_w f_{NTK}(w_0,x)^T \nabla F^T \}  }_{\sriskv},
\end{align*}
where we denote $\nabla F(x') = [\nabla_w f_{NTK}(w_0,x'_1), \dots, \nabla_w f_{NTK} (w_0, x'_n) ]^T \in \rR^{n \times m(p+1)}$, and the last inequality is induced from the facts:
\begin{align*}
 & \| w_* - w_0 \|_2^2 = \| [\Theta_*, U_*] - [\Theta_0, U_0] \|_F^2    \le R^2,\\
 & \| I - \nabla F^T (\nabla F \nabla F^T)^{-1} \nabla F \|_2 \le 1.
 \end{align*}

For the first term $\sriskb$, we first prove the random variable $\nabla_w f_{NTK}(w_0, x)$ is sub-gaussian with respect to $x$. Take derivative for $\nabla_w f_{NTK}(w_0, x)$ on each dimension of $x$:
\begin{equation}\label{eq:deriv}
\begin{aligned}
& \| \frac{\partial^2 f_{NTK}(w_0,x)}{\partial u_i \partial x} \|_2 = \| \frac{1}{\sqrt{m p}} h'(\theta_{0,i}^T x) \theta_{0,i} \|_2 \le \frac{\| \theta_{0,i} \|_2}{\sqrt{m p}}, \quad i = 1, \dots, m,\\ 
& \| \frac{\partial^2 f_{NTK}(w_0,x)}{\partial \theta_{i,j} \partial x} \|_2 = \| \frac{u_{0,i}}{\sqrt{m p}} h'(\theta_{0,i}^T x) e_j \|_2 \le \frac{| u_{0,i} |}{\sqrt{m p}}, \quad i = 1, \dots, m, j = 1, \dots, p,
\end{aligned}    
\end{equation}
it implies that for any vector $\gamma \in \mathbb{R}^{m(p+1)}$, the function $\gamma^T \nabla_w f_{NTK}(w_0,x)$ has a bounded Lipschitz as
\begin{equation*}
    \| \frac{\partial \gamma^T \nabla_w f_{NTK}(w_0,x)}{\partial x} \|_2 \le \frac{1}{\sqrt{m p}} \left( \sum_{i=1}^m | \gamma_i | \| \theta_{0,i} \|_2 + \sum_{i=1}^m \sum_{j=1}^p | u_{0,i} | |\gamma_{i m + p } | \right) \le \frac{1}{\sqrt{m p}} \| \gamma \|_2 \sqrt{ \sum_{i=1}^m \| \theta_{0,i} \|_2^2 + p u_{0,i}^2  },  
\end{equation*}
where the first inequality is due to the derivative results in Eq.\eqref{eq:deriv}, and the second inequality is from Cauthy-Schwarz inequality. Then by Lemma \ref{lem:lip_gaussian}, we can obtain
\begin{equation*}
    \mathbb{E} e^{\lambda \gamma^T \nabla_w f_{NTK}(w_0,x)} \le 
\text{exp} \left(\frac{\lambda^2 \| \gamma \|_2^2 (\sum_{i=1}^m \| \theta_{0,i} \|_2^2 + p u_{0,i}^2 )}{2 m p} \right), 
\end{equation*}
which implies that $\nabla_w f_{NTK}(w_0,x)$ is a $\sqrt{(\sum_{i=1}^m \| \theta_{0,i} \|_2^2 + p u_{0,i}^2) / m p}$-subgaussian random vector, and $\nabla F$ can be regarded as $n$ i.i.d.~samples from the distribution of $\nabla_w f_{NTK}(w_0, x)$, corresponding to data $x_1, \dots, x_n$. Then calculate the mean value of $\nabla_w f_{NTK}(w_0, x)$ on each dimension, we obtain
\begin{align*}
& \mathbb{E}_x \frac{\partial f_{NTK}(w_0,x)}{\partial u_i} = \mathbb{E}_x \frac{1}{\sqrt{m p}} h(\theta_{0,i}^T x) = \frac{\| 
\varSigma^{1/2} \theta_{0,i} \|_2}{\sqrt{2 \pi m p}}, i = 1, \dots, m,\\
& \mathbb{E}_x \frac{\partial f_{NTK}(w_0 x)}{\partial \theta_{i,j}} = \mathbb{E}_x \frac{u_{0,i}}{\sqrt{m p}} h'(\theta_{0,i}^T x) x_j = \frac{u_{0,i}}{\sqrt{2 \pi m p}} \frac{\theta_{0,i}^T \varSigma e_j}{\| 
\varSigma^{1/2} \theta_{0,i} \|_2}, i = 1, \dots, m, j = 1, \dots, p, 
\end{align*}
which implies that the $L_2$ norm of its mean value is
\begin{equation*}
\begin{aligned}
\| \mathbb{E}_x \nabla_w f_{NTK}(w_0, x) \|_2^2 &= \frac{1}{2 \pi m p} \left( \sum_{i=1}^m \theta_{0,i}^T \varSigma \theta_{0,i} + \sum_{i=1}^m \sum_{j=1}^p \frac{u_{0,i}^2 (\varSigma^{1/2} \theta_{0,i})_j^2}{\| \varSigma^{1/2} \theta_{0,i} \|_2^2} \right)\\
& = \frac{1}{2 \pi m p} \left( \sum_{i=1}^m \theta_{0,i}^T \varSigma \theta_{0,i} + u_{0,i}^2 \right),
\end{aligned}
\end{equation*}
then using Lemma \ref{lem:matrix_mean}, with probability at least $1 - 4 e^{- \sqrt{n}}$, we can get
\begin{equation}\label{eq:ntk_bias_o}
\begin{aligned}
& \quad \| \underbrace{ \mathbb{E}_x \nabla_w f_{NTK}(w_0, x) \nabla_w f_{NTK}(w_0, x)^T}_{S_f} - \frac{1}{n} \nabla F^T \nabla F \|_2\\
& \le \| S_f \|_2 \max \{ \sqrt{\frac{\text{tr}(S_f)}{n}}, \frac{\text{tr}(S_f)}{n}, \frac{1}{n^{1/4}} \} + 2 \sqrt{2} \frac{\sqrt{(\sum_{i=1}^m \| \theta_{0,i} \|_2^2 + p u_{0,i}^2)(\sum_{i=1}^m  \theta_{0,i}^T \varSigma \theta_{0,i} + u_{0,i}^2) } }{\sqrt{2 \pi} m p n^{1/4}},
\end{aligned} 
\end{equation}
with some constant $C > 0$. Under Condition~\ref{cond:ntk_high-dim}, we have
\begin{align*}
 \| S_f \|_2 \le \text{tr} (S_f) &= \left( 1 + O_p(\frac{1}{\sqrt{m}}) \right) \mathbb{E}_x \nabla_w f_{NTK}(w_0, x)^T \nabla_w f_{NTK}(w_0, x)\\
 &=\left( 1 + O_p(\frac{1}{\sqrt{m}}) \right)\mathbb{E}_x K(x, x) =\left( 1 + O_p(\frac{1}{\sqrt{m}}) \right)\mathbb{E}_x \frac{\| x \|_2^2}{p} = \frac{l}{p} + O_p(\frac{l}{p \sqrt{m}}),\\
 \frac{\sum_{i=1}^m \| \theta_{0,i} \|_2^2 + p u_{0,i}^2}{m} &=\left( 1 + O_p(\frac{1}{\sqrt{m}}) \right) \left( \mathbb{E}_{w_0} \| \theta_0 \|_2^2 + p u_0^2 \right) = 2 p + O_p(\frac{p}{\sqrt{m}}) ,\\
\frac{\sum_{i=1}^m  \theta_{0,i}^T \varSigma \theta_{0,i} + u_{0,i}^2}{m} &=\left( 1 + O_p(\frac{1}{\sqrt{m}}) \right) \left( \mathbb{E}_{w_0} \text{tr} [\varSigma \theta_{0} \theta_0^T] + u_0^2 \right) = l + 1 + O_p(\frac{p}{\sqrt{m}}),
\end{align*}
take the results above into Eq.\eqref{eq:ntk_bias_o}, then we can obtain that with probability at least $1 - 4 e^{- \sqrt{n}}$,
\begin{equation}\label{eq:ntk_bias}
\begin{aligned}
 & \quad \| \mathbb{E}_x \nabla_w f_{NTK}(w_0, x) \nabla_w f_{NTK}(w_0, x)^T - \frac{1}{n} \nabla F^T \nabla F \|_2\\
 & \le \frac{l}{p} \left( \frac{1}{n^{1/4}} + \frac{l}{n p} \right) + 4 \sqrt{2} \frac{\sqrt{2p(l+1)}}{\sqrt{2 \pi} p n^{1/4}} \le \frac{8}{\sqrt{\pi}} \sqrt{\frac{l}{p}} \frac{1}{n^{1/4}}.
\end{aligned}
\end{equation}
Then we turn to the variance term $\sriskv$, 
\begin{align*}
& \quad \sigma^2 \mathbb{E}_{ x } \text{tr} \{ K^{-2} \nabla F \nabla_w f_{NTK}(w_0,x) \nabla_w f_{NTK}(w_0,x)^T \nabla F^T \}  \\
&= \frac{\sigma^2}{n} \mathbb{E}_{x'_i} \sum_{i=1}^n \text{tr} \{ (\nabla F \nabla F^T)^{-1} \nabla F \nabla_w f_{NTK}(w_0, x'_i) \nabla_w f_{NTK} (w_0, x'_i)^T \nabla F^T (\nabla F \nabla F^T)^{-1} \}\\
&= \frac{\sigma^2}{n} \mathbb{E}_{ x'_i} \sum_{i=1}^n \text{tr} \{ (\nabla F \nabla F^T)^{-1} \nabla F \nabla_w f_{NTK}(w_0, x'_i) \nabla_w f_{NTK} (w_0, x'_i)^T \nabla F^T (\nabla F \nabla F^T)^{-1} \}\\
&=  \frac{\sigma^2}{n} \mathbb{E}_{ x'_i} \text{tr} \{ (\nabla F \nabla F^T)^{-1} \nabla F \nabla F(x')^T \nabla F(x') \nabla F^T (\nabla F \nabla F^T)^{-1} \}\\
&=  \frac{\sigma^2}{n} \mathbb{E}_{ x'_i} \text{tr} \{ K^{-2} \nabla F \nabla F(x')^T \nabla F(x') \nabla F^T \},
\end{align*}
where we denote $x'_1, \dots, x'_n$ are i.i.d.~samples from the same distribution as $x_1, \dots,x_n$, And the last equality  is from the fact that $\nabla F \nabla F^T = K$. For the matrix $\nabla F \nabla F(x')^T $ and $\nabla F(x') \nabla F^T$, with probability at least $ 1 - 4 n^2 e^{- n^{1/4} / 2}$, we can take similar procedure as  Lemma \ref{lem:spectrum} to linearize them respectively:
\begin{equation}\label{eq:non_asym}
\begin{aligned}
& \|  \nabla F \nabla F(x')^T - \frac{l}{p} (\frac{1}{2 \pi} + \frac{3 r_0(\varSigma^2)}{4 \pi l^2}) 11^T - \frac{1}{2 p} XX^{'T} \|_2 \le \frac{4l}{p n^{1/16}},\\
& \| \nabla F(x') \nabla F^T - \frac{l}{p} (\frac{1}{2 \pi} + \frac{3 r_0(\varSigma^2)}{4 \pi l^2}) 11^T - \frac{1}{2 p} X' X^T \|_2 \le \frac{4 l}{p n^{1/16}},
\end{aligned}    
\end{equation}
as the samples $x'_i, i = 1, \dots, n$ are independent of $x_i, i = 1, \dots, n$, we can take Eq.~\eqref{eq:non_asym} into $\sriskv$ to take expecation as:
\begin{equation}\label{eq:ntk_var_o}
\begin{aligned}
& \quad  \sriskv\\
& \le 2 \frac{\sigma^2}{n} \mathbb{E}_{ x'_i} \text{tr} \left\{ \tilde{K}^{-2} \left( (\frac{l}{p} (\frac{1}{2 \pi} + \frac{3 r_0(\varSigma^2)}{4 \pi l^2}) 11^T + \frac{1}{2 p} XX^{'T}) (\frac{l}{p} (\frac{1}{2 \pi} + \frac{3 r_0(\varSigma^2)}{4 \pi l^2})11^T + \frac{1}{2 p}X' X^T) + \frac{16l^2}{p^2 n^{1/8}} I \right) \right\} ,\\
&= 2 \sigma^2 \text{tr} \{ \tilde{K}^{-2}(\frac{l^2}{p^2} (\frac{1}{4 \pi^2} + o(1)) 11^T + \frac{1}{4 p^2} X \varSigma X^T + \frac{16 l^2}{p^2 n^{9/8}} I_n) \},
\end{aligned}
\end{equation}
where the inequality is from linearizing the matrix $K$, $\nabla F \nabla F(x')^T $ and $\nabla F(x') \nabla F^T$. By Woodbury identity, denoting
\begin{equation*}
    \tilde{R} = \frac{1}{2 p} XX^T + \frac{l}{p}(\frac{1}{2} - \frac{1}{2 \pi}) I_n,
\end{equation*}
we can get
\begin{equation*}
\begin{aligned}
& \quad \frac{l^2}{p^2} (\frac{1}{4 \pi^2} + o(1)) 1^T \tilde{K}^{-2} 1 = \frac{l^2}{p^2} (\frac{1}{4 \pi^2} + o(1)) 1^T (\frac{l}{p} (\frac{1}{2 \pi} + \frac{3 r_0(\varSigma^2)}{4 \pi l^2}) 11^T + \tilde{R})^{-2} 1\\
& = \frac{\frac{l^2}{p^2} (\frac{1}{4 \pi^2} + o(1)) 1^T \tilde{R}^{-2} 1}{(1 + \frac{l}{p} (\frac{1}{2 \pi} + \frac{3 r_0(\varSigma^2)}{4 \pi l^2}) 1^T \tilde{R}^{-1} 1)^2}  \le \frac{1^T \tilde{R}^{-2} 1}{(1^T \tilde{R}^{-1} 1)^2} \le \frac{n / \lambda_n(\tilde{R})^2}{n^2 / \lambda_1(\tilde{R})^2},   
\end{aligned}    
\end{equation*}
where the first inequality is from ignoring the constant term $1$ on denominator, and the second inequality is due to the fact
\begin{align*}
    & 1^T \tilde{R}^{-2} 1 \le n \lambda_1(\tilde{R}^{-1})^2 = n / \lambda_n(\tilde{R})^2,\\
    & 1^T \tilde{R}^{-1} 1 \ge n \lambda_n(\tilde{R}^{-1}) = n / \lambda_1(\tilde{R}),
\end{align*}
as recalling Lemma \ref{lem_ridgeeigen}, with a high probability, we have
\begin{equation*}
    \lambda_n(\tilde{R}) \ge \frac{l}{p}(\frac{1}{2} - \frac{1}{2 \pi}) +  \frac{1}{c_1 p} \lambda_{k^* + 1} r_{k^*} \ge \frac{l}{4 p}, \quad \lambda_1 (\tilde{R}) \le \frac{l}{p}(\frac{1}{2} - \frac{1}{2 \pi}) + \frac{c_1}{p}(n \lambda_1 + l) \le \frac{2 l (1 + c_1)}{p} \le \frac{(1 + c_1)(l + n)}{p},
\end{equation*}
we can further obtain that
\begin{equation*}
\frac{l^2}{p^2} (\frac{1}{4 \pi^2} + o(1)) 1^T \tilde{K}^{-2} 1 \le \frac{n}{n^2} \frac{4 (l + n)^2 (1 + c_1)^2 / p^2}{l^2 / (4 p^2)} \le 32(1 + c_1)^2 \left( \frac{1}{n} + \frac{n}{l^2} \right),
\end{equation*}
further due to Condition \ref{cond:ntk_high-dim}, we have
\begin{equation}\label{eq:ntk_var1}
\frac{l^2}{p^2} (\frac{1}{4 \pi^2} + o(1)) 1^T \tilde{K}^{-2} 1 \le  32(1 + c_1)^2 \left( \frac{1}{n} + \frac{n}{l^2} \right) \le \frac{64 (1 + c_1)^2}{n^{1/2}}.
\end{equation}
For the second term, based on Lemma \ref{lem_ridgeeigen}, with probability at least $1 - c e^{- n /c}$, we can obtain that
\begin{equation}\label{eq:ntk_var2}
\begin{aligned}
\frac{\sigma^2}{2 p^2} \text{tr} \{ \tilde{K}^{-2} X \varSigma X \} &= \frac{\sigma^2}{2 p^2} \text{tr} \{ (\frac{l}{p} (\frac{1}{2 \pi} + \frac{3 r_0(\varSigma^2)}{4 \pi l^2}) 11^T + \frac{1}{2p} XX^T + \frac{l}{p}(\frac{1}{2} - \frac{1}{2 \pi}) I_n )^{-2} X \varSigma X^T \}\\
&\le \frac{\sigma^2}{ p^2} \text{tr} \{ (\frac{1}{2p} XX^T + \frac{l}{p} (\frac{1}{2} - \frac{1}{ 2 \pi}) I_n )^{-2} X \varSigma X^T \}\\
&= \sigma^2 \text{tr} \{ (XX^T + 2 l (\frac{1}{2} - \frac{1}{2 \pi}) I_n )^{-2} X \varSigma X^T \}\\
& \le \sigma^2 \left(  \frac{k^*}{n} \frac{(c_1 \lambda_{k^* + 1} r_{k^*} + l (1 - 1/\pi))^2}{( 1/c_1 \lambda_{k^* + 1} r_{k^*} + l (1 - 1/\pi))^2} + \frac{n \sum_{i > k^*} \lambda_i^2}{(\lambda_{k^* + 1} r_{k^*} + l(1 - 1/\pi) )^2} \right)\\
&\le \sigma^2 \left( c_1^4 \frac{k^*}{n} + \frac{n \sum_{i > k^*} \lambda_i^2}{(l(1 - 1/\pi) )^2} \right),
\end{aligned}    
\end{equation}
where the first inequality is from relaxing the unimportant term $11^T$ (see Lemma 2.2 in \citet{bai2008methodologies}), the second inequality is based on Lemma \ref{lem_bart}, and the last inequality is from the fact that
\begin{align*}
 & \frac{(c_1 \lambda_{k^* + 1} r_{k^*} + l (1 - 1/\pi))^2}{( 1/c_1 \lambda_{k^* + 1} r_{k^*} + l (1 - 1/\pi))^2} \le c_1^4,\\
 & \lambda_{k^* + 1} r_{k^*} + l(1 - 1/\pi) \ge l (1 - 1/\pi).
 \end{align*}
And for the third term,
\begin{equation}\label{eq:ntk_var3}
 \begin{aligned}
2 \sigma^2 \text{tr} \{ \tilde{K}^{-2} \frac{16 l^2}{p^2 n^{9/8}} \} &= \frac{16 l^2}{p^2 n^{9/8}} \sigma^2 \text{tr} \{ (\frac{l}{p} (\frac{1}{2 \pi} + \frac{3 r_0(\varSigma^2)}{4 \pi l^2}) 11^T + \frac{1}{2p} XX^T + \frac{l}{p}(\frac{1}{2} - \frac{1}{2 \pi}) I_n )^{-2} \}\\
&\le \frac{32 l^2}{p^2 n^{9/8}} \sigma^2  \text{tr} \{ (\frac{1}{2p} XX^T + \frac{l}{p}(\frac{1}{2} - \frac{1}{2 \pi}) I_n )^{-2} \}\\
&= \frac{128 l^2}{n^{9/8}} \sigma^2 \text{tr} \{ (XX^T + l (1 - 1/\pi) I_n)^{-2} \}\\
& \le \frac{128}{(1 - 1/\pi)^2} \frac{1}{n^{1/8}},
\end{aligned}   
\end{equation}
where the last inequality is from the fact that $\mu_{n}(XX^T + l (1 - 1/\pi)) \ge l(1 - 1/\pi)$.
So combing Eq.\eqref{eq:ntk_bias}, \eqref{eq:ntk_var_o}, \eqref{eq:ntk_var1}, \eqref{eq:ntk_var2} and \eqref{eq:ntk_var3}, with a high probability, $\srisk(\hat{w})$ can be upper bounded as
\begin{equation}\label{eq:risk_final}
     \srisk(\hat{w}) \le r^2 \frac{8}{\sqrt{\pi}} \sqrt{\frac{l}{p}} \frac{1}{n^{1/4}} + 128 (1 + c_1)62 \sigma^2 \left( \frac{1}{n^{1/8}} + c_1^4 \frac{k^*}{n} + \frac{n \sum_{i> k^*} \lambda_i^2}{l^2 (1 - 1/ \pi)^2} \right).
\end{equation}

\textbf{Step 3: asymptotic Lipschitz norm estimation. } The final step is to lower bound the excessive adversarial risk $\arisk(\hat{w})$:
\begin{equation}\label{eq:adv_lower}
\begin{aligned}
& \quad  \arisk(\hat{w})\\
&= \alpha^2 \mathbb{E}_{ x,\epsilon} \| \nabla_x f_{NTK} (\hat{w}, x) \|_2^2 = \alpha^2 \mathbb{E}_{ x,\epsilon} \| \nabla_x f_{NTK} (w_0, x) + \frac{\partial^2 f_{NTK}(w_0, x)}{\partial w \partial x} (\hat{w} - w_0) \|_2^2\\
& \ge \alpha^2 \left| \mathbb{E}_{ x,\epsilon} \| \frac{\partial^2 f_{NTK}(w_0, x)}{\partial w \partial x} (\hat{w} - w_0) \|_2^2 - \mathbb{E}_{x,\epsilon} \| \nabla_x f_{NTK} (w_0, x)  \|_2^2 \right|,
\end{aligned}
\end{equation}
as the term $\mathbb{E}_{x,\epsilon} \| \nabla_x f_{NTK} (w_0, x)  \|_2^2$ can be calculated as
\begin{equation}\label{eq:adv_rt}
\mathbb{E}_{x,\epsilon} \| \nabla_x f_{NTK} (w_0, x)  \|_2^2 = \mathbb{E}_{ x,\epsilon} \| \frac{1}{\sqrt{mp}} \sum_{j=1}^m u_{0,j} h'(\theta_{0,j}^T x) \theta_{0,j}  \|_2^2 = \frac{1}{2} < \infty,
\end{equation}
the adversarial robustness is just measured by the term $\mathbb{E}_{x,\epsilon} \| \frac{\partial^2 f_{NTK}(w_0, x)}{\partial w \partial x} (\hat{w} - w_0) \|_2^2$ in Eq.\eqref{eq:adv_lower}. By Jensen's inequality, we have
\begin{align*}
& \quad \mathbb{E}_{ x,\epsilon} \| \frac{\partial^2 f_{NTK}(w_0, x)}{\partial w \partial x} (\hat{w} - w_0) \|_2^2 \\
&\ge \mathbb{E}_{ \epsilon} \| \mathbb{E}_x \frac{\partial^2 f_{NTK}(w_0, x)}{\partial w \partial x} (\hat{w} - w_0) \|_2^2\\
&=\mathbb{E}_{\epsilon} \sum_{d=1}^p \left( \frac{1}{\sqrt{mp}} \mathbb{E}_x [ \sum_{j=1}^m u_{0,j} h'(\theta_{0,j}^T x) (\hat{\theta}_{j,d} - \theta_{0,j,d}) + \sum_{j=1}^m \theta_{0,j,d} h'(\theta_{0,j}^T x) (\hat{u}_j - u_{0,j}) ] \right)^2\\
&= \mathbb{E}_{ \epsilon} \sum_{d=1}^p \left( \frac{1}{2 \sqrt{mp}}  [ \sum_{j=1}^m u_{0,j}  (\hat{\theta}_{j,d} - \theta_{0,j,d}) + \sum_{j=1}^m \theta_{0,j,d}  (\hat{u}_j - u_{0,j}) ] \right)^2,
\end{align*}
the equalities are from the  direct expansion of $\frac{\partial^2 f_{NTK}(w_0, x)}{\partial w \partial x} (\hat{w} - w_0)$.
Recalling the expression of $\hat{w}$ in Eq.\eqref{eq:ntk_para}, if we denote two types of vectors as
\begin{equation*}
\begin{aligned}
& \beta_{j, d} = [ h'(\theta_{0,j}^T x_1) x_{1,d}, \dots, h'(\theta_{0,j}^T x_n) x_{n,d}]^T \in \rR^n,\\
& \gamma_j = [ h(\theta_{0,j}^T x_1) , \dots, h(\theta_{0,j}^T x_n) ]^T \in \rR^n,
\end{aligned}
\end{equation*}
the estimated parameters can be expressed as
\begin{equation}\label{eq:para_est}
\hat{\theta}_{j,d} - \theta_{0,j,d} = \frac{u_{0,j}}{\sqrt{mp}} \beta_{j,d}^T K^{-1} (\nabla F (w_* - w_0) + \epsilon), \quad \hat{u}_j - u_{0,j} = \frac{1}{\sqrt{mp}} \gamma_j^T K^{-1} (\nabla F (w_* - w_0) + \epsilon).
\end{equation}
Take Eq.\eqref{eq:para_est} into the expression above, we could further obtain
\begin{equation}\label{eq:adv1}
\begin{aligned}
& \quad \mathbb{E}_{ x,\epsilon} \| \frac{\partial^2 f_{NTK}(w_0, x)}{\partial w \partial x} (\hat{w} - w_0) \|_2^2 \\
&\ge \mathbb{E}_{ \epsilon} \| \mathbb{E}_x \frac{\partial^2 f_{NTK}(w_0, x)}{\partial w \partial x} (\hat{w} - w_0) \|_2^2\\
&= \mathbb{E}_{ \epsilon} \sum_{d=1}^p \left( \frac{1}{2 m p} [ \sum_{j=1}^m u_{0,j}^2 \beta_{j,d}^T K^{-1} (\nabla F (w_* - w_0) + \epsilon) + \sum_{j=1}^m \theta_{0,j,d} \gamma_j^T K^{-1}  (\nabla F (w_* - w_0) + \epsilon)] \right)^2,
\end{aligned}    
\end{equation}
considering Condition~\ref{cond:ntk_high-dim}, we can get
\begin{align*}
& \quad \frac{1}{2 m p}  \left( \sum_{j=1}^m u_{0,j}^2 \beta_{j,d}^T K^{-1} (\nabla F (w_* - w_0) + \epsilon) + \sum_{j=1}^m \theta_{0,j,d} \gamma_j^T K^{-1}  (\nabla F (w_* - w_0) + \epsilon) \right)\\
& =\left( 1 + O_p(\frac{1}{\sqrt{m}}) \right) \frac{1}{2 p} \mathbb{E}_{w_0} \left[  u_{0,1}^2 \beta_{1,d}^T K^{-1} (\nabla F (w_* - w_0) + \epsilon) + \theta_{0,1,d} \gamma_1^T K^{-1}  (\nabla F (w_* - w_0) + \epsilon)\right]\\
& =\left( 1 + O_p(\frac{1}{\sqrt{m}}) \right) \frac{1}{2 p}  [x_{1,d}, \dots, x_{n,d}] K^{-1} (\nabla F (w_* - w_0) + \epsilon),
\end{align*}
take this result into Eq.\eqref{eq:adv1}, we could obtain that
\begin{equation}\label{eq:adv_o}
\begin{aligned}
& \quad \mathbb{E}_{ x,\epsilon} \| \frac{\partial^2 f_{NTK}(w_0, x)}{\partial w \partial x} (\hat{w} - w_0) \|_2^2 \\
& \ge \mathbb{E}_{ \epsilon} \sum_{d=1}^p \left( \frac{1}{4 p} \mathbb{E}_{ \epsilon} \{[x_{1,d}, \dots, x_{n,d}] K^{-1} (\nabla F (w_* - w_0) + \epsilon)\} \right)^2\\
&= \frac{1}{16 p^2} \mathbb{E}_{ \epsilon} \text{tr} \{ K^{-1} (\nabla F (w_* - w_0) + \epsilon) (\nabla F (w_* - w_0) + \epsilon)^T K^{-1} XX^T \}\\
&\ge \frac{\sigma^2}{16 p^2} \text{tr} \{ K^{-2} XX^T \} \ge \frac{\sigma^2}{32 p^2} \text{tr} \{\tilde{K}^{-2} XX^T \}.
\end{aligned}    
\end{equation}
where the second inequality is from ignoring the term related to $w_* - w_0$, and the last inequality is from linearizing kernel matrix $K$ to $\tilde{K}$.
Recalling Eq.\eqref{eq:normvl}, we have
\begin{align*}
& \quad \mathbb{E}_{ x,\epsilon} \| \frac{\partial^2 f_{NTK}(w_0, x)}{\partial w \partial x} (\hat{w} - w_0) \|_2^2 \\
&\ge \frac{\sigma^2}{32 p^2} \text{tr} \{\tilde{K}^{-2} XX^T \} \\
&= \frac{\sigma^2}{32 p^2} \text{tr} \{ (\frac{l}{p} (\frac{1}{2 \pi} + \frac{3 r_0(\varSigma^2)}{4 \pi l^2}) 11^T + \frac{1}{2 p} XX^T + \frac{l}{p}(\frac{1}{2} - \frac{1}{2 \pi}) I_n)^{-2} XX^T \}\\
&\ge \frac{\sigma^2}{64 p^2} \text{tr} \{ (\frac{1}{2 p} XX^T + \frac{l}{p}(\frac{1}{2} - \frac{1}{2 \pi}) I_n)^{-2} XX^T \}\\
& = \frac{\sigma^2}{16} \text{tr} \{ (XX^T + l (1 - 1 / \pi) I_n)^{-2} XX^T \}\\
& \ge \frac{\sigma^2}{288 c^2 c_3^2 c_2} \frac{n \lambda_{k^*+1} r_{k^*} }{(\lambda_{k^*+1} r_{k^*} +  l (1 - 1 / \pi))^2} \ge \frac{\sigma^2}{1152 c^2 c_3^2 c_2 } \frac{n \lambda_{k^*+1} r_{k^*} }{l^2},
\end{align*}
where the second inequality is from relaxing the term $11^T$ (see Lemma 2.2 in \citet{bai2008methodologies}), the third inequality is based on Eq.\eqref{eq:normvl}, and the last inequality is from the fact that
\begin{equation*}
     \lambda_{k^* + 1} r_{k^*} + l (1 - 1/\pi) \le 2 l, 
\end{equation*}
With Condition~\ref{cond:ntk_benign}, we have the fact that
\begin{equation}\label{eq:ntk_adv}
\mathbb{E}_{ x,\epsilon} \| \frac{\partial^2 f_{NTK}(w_0, x)}{\partial w \partial x} (\hat{w} - w_0) \|_2^2 \ge \frac{\sigma^2}{1152 c^2 c_3^2 c_2 } \frac{n \lambda_{k^*+1} r_{k^*} }{l^2},
\end{equation}
will exploded while $n$ increases.

\section{Proof for Remark \ref{rmk_widerclass}}\label{pf:widerclass}
Due to the analysis above, $w_0 - w_*$ just influence the bias term $\sriskb$ in $\srisk$, so we just need to consider this term.

First, with the solution in Eq.\eqref{eq:ntk_para}, $\sriskb$ can be expressed as
\begin{align*}
& \quad \mathbb{E}_{ x, \epsilon} \{ \nabla_w f_{NTK}(w_0, x)^T [ (\nabla F^T (\nabla F \nabla F^T)^{-1} \nabla F - I)(w_* - w_0) ] \}^2\\
& = \mathbb{E}_{ x}  (w_* - w_0)^T (I - \nabla F^T (\nabla F \nabla F^T)^{-1} \nabla F ) (\nabla_w f_{NTK} (w_0,x) \nabla_w f_{NTK} (w_0, x)^T - \frac{1}{n} \nabla F^T \nabla F ) \\
& \quad \quad \quad \quad \quad \quad \quad \quad ( I - \nabla F^T (\nabla F \nabla F^T)^{-1} \nabla F ) (w_* - w_0) \\
& = (v_* - v_0)^T \mathbb{E}_{x} \left[ \nabla_w f_{NTK} (w_0,x) \nabla_w f_{NTK} (w_0, x)^T - \frac{1}{n} \nabla F^T \nabla F \right] (v_* - v_0) ,
\end{align*}
where we denote $\nabla F(x') = [\nabla_w f_{NTK}(w_0,x'_1), \dots, \nabla_w f_{NTK} (w_0, x'_n) ]^T \in \rR^{n \times m(p+1)}$, and 
\begin{equation*}
    v_* - v_0 = ( I - \nabla F^T (\nabla F \nabla F^T)^{-1} \nabla F ) (w_* - w_0).
\end{equation*}
Then reviewing Eq~\eqref{eq:deriv} as:
\begin{align*}
& \| \frac{\partial^2 f_{NTK}(w_0,x)}{\partial u_i \partial x} \|_2 = \| \frac{1}{\sqrt{m p}} h'(\theta_{0,i}^T x) \theta_{0,i} \|_2 \le \frac{\| \theta_{0,i} \|_2}{\sqrt{m p}}, \quad i = 1, \dots, m,\\ 
& \| \frac{\partial^2 f_{NTK}(w_0,x)}{\partial \theta_{i,j} \partial x} \|_2 = \| \frac{u_{0,i}}{\sqrt{m p}} h'(\theta_{0,i}^T x) e_j \|_2 \le \frac{| u_{0,i} |}{\sqrt{m p}}, \quad i = 1, \dots, m, j = 1, \dots, p,
\end{align*}    
we could obtain that the function $(w_* - w_0)^T \nabla_w f_{NTK}(w_0,x)$ has a bounded Lipschitz with respect to $x$:
\begin{equation*}
    \| \frac{\partial (v_* - v_0)^T \nabla_w f_{NTK}(w_0,x)}{\partial x} \|_2 \le \frac{1}{\sqrt{m p}} \left( \sum_{i=1}^m | v_{*,i} - v_{0,i} | \| \theta_{0,i} \|_2 + \sum_{i=1}^m \sum_{j=1}^p | v_{*,mi+j} - v_{0,mi+j} | |u_{0,i} | \right) =: \text{lip},
\end{equation*}
where the inequality is due to the derivative results in Eq.\eqref{eq:deriv}. By Lemma \ref{lem:lip_gaussian}, we can obtain
\begin{equation*}
    \mathbb{E} e^{\lambda (v_* - v_0)^T \nabla_w f_{NTK}(w_0,x)} \le \text{exp} \left(\frac{\lambda^2 \text{lip}^2}{2} \right), 
\end{equation*}
which implies that $(v_* - v_0)^T \nabla_w f_{NTK}(w_0,x)$ is a $\text{lip}$-subgaussian random variable. Then with Lemma~\ref{lem_sg_se}, $(v_* - v_0)^T \nabla_w f_{NTK}(w_0,x) \nabla_w f_{NTK}(w_0, x)^T (v_* - v_0)$ is a $162e \text{lip}^2$-subgaussioan random variable. As $(v_* - v_0)^T \nabla F \nabla F^T (v_* - v_0)$ can be regarded as $n$ i.i.d.~samples from the same distribution as the random variable $(v_* - v_0)^T \nabla_w f_{NTK}(w_0, x) \nabla_w f_{NTK}(w_0, x)^T (v_* - v_0)$, corresponding to data $x_1, \dots, x_n$. Then with probability at least $1 - \text{exp}(- n t^2 / (162e)^2 \text{lip}^4 )$, the bias term could be upper bounded as
\begin{equation*}
    \sriskb \le t,
\end{equation*}
choosing $t = 162e \text{lip}^2 / \sqrt{n}$, we could further obtain 
\begin{equation}\label{eq:ntk_biasf}
    \sriskb \le \frac{162e \text{lip}^2}{\sqrt{n}},
\end{equation}
with probability at least $1 - e^{-\sqrt{n}}$. And the remaining problem is to estimate $\text{lip} / \sqrt{p}$. With the conditions on initial parameter $w_0$ are ground truth parameter $w_*$, we have
\begin{align*}
 \frac{\text{lip}}{\sqrt{p}} &= \frac{1}{p \sqrt{m}} \left( \sum_{i=1}^m | v_{*,i} - v_{0,i} | \| \theta_{0,i} \|_2 + \sum_{i=1}^m \sum_{j=1}^p | v_{*,im+j} - v_{0,im+j} | |u_{0,i} | \right)\\
 &\le \frac{1}{p \sqrt{m}} \| v_* - v_0 \|_{\infty} \left( \sum_{i=1}^m  \| \theta_{0,i} \|_2 + \sum_{i=1}^m \sum_{j=1}^p  |u_{0,i} | \right)\\
 & \le \frac{\epsilon_p^{1/2}}{p \sqrt{m}}  \left( \sum_{i=1}^m  \| \theta_{0,i} \|_2 + \sum_{i=1}^m \sum_{j=1}^p  |u_{0,i} | \right)\\
 &= \epsilon_p^{1/2} \left( \frac{1}{\sqrt{m}} \sum_{i=1}^m | u_{0,i} | + \frac{1}{\sqrt{m}} \sum_{i=1}^m \frac{\| \theta_{0,i} \|_2}{p} \right)\\
 &\le \epsilon_p^{1/2} \left( 1 + O(\frac{1}{\sqrt{m}}) \right) \left( 1 + O(\frac{1}{\sqrt{p}}) \right) \le 2 \| w_* - w_0 \|_{\infty},
\end{align*}
where the first inequality is from 
\begin{equation*}
   | \sum_s a_s b_k | \le \max_s | a_s | \cdot \sum_s | b_s |,
\end{equation*}
the second inequality is due to the fact that $I - \nabla F^T (\nabla F \nabla F^T)^{-1} \nabla F$ is a projection matrix which spans on $(mp + m - n)$-dim space and $\| \mathbb{E} (v_* - v_0) (v_* - v_0)^T \|_2 \le \| \mathbb{E} (w_* - w_0) (w_* - w_0)^T \|_2 \le \epsilon_1 $,
the third inequality is induced by Condition~\ref{cond:ntk_high-dim}. So considering Eq.~\eqref{eq:ntk_biasf}, we could further obtain that
\begin{equation*}
    \sriskb \le \frac{162 e \cdot p}{\sqrt{n}} \| w_* - w_0 \|_{\infty}^2,
\end{equation*}
with a probability at least $1 - e^{- \sqrt{n}}$. To be specific, as $w_* - w_0$ is a random vector satisfying that $\| w_* - w_0 \|_{\infty}^2 = o_p(1/p)$, the bias term $\sriskb$ converges to zero with a rate at least $o(1 / \sqrt{n})$.

\section{Auxiliary Lemmas}

\begin{lemma}[Refinement of Theorem 2.1 in \citealp{el2010spectrum}]\label{lem:spectrum}
Let we assume that we observe n i.i.d.~random vectors, $x_i \in \rR^p$. Let us consider the kernel matrix $K$ with entries
\begin{equation*}
    K_{i,j} = f(\frac{x_i^T x_j}{l}).
\end{equation*}
We assume that:
\begin{enumerate}
\item  $n,l,p$ satisfy Condition~\ref{cond:ntk_high-dim};
\item  $\varSigma$ is a positive-define $p \times p$ matrix, and $\| \varSigma \|_2 = \lambda_{\max}(\varSigma)$ remains bounded (without loss of generality, here we suppose $\lambda_{\max}(\varSigma) = 1$);
\item $\varSigma / l$ has a finite limit, that is, there exists $\tau \in \rR$ such that $\lim_{p \to \infty} \text{trace}(\varSigma) / l = \tau$;
\item $x_i = \varSigma^{1/2} \eta_i$, in which $\eta_i, i = 1, \dots, n$ are $\sigma$-subgaussian i.i.d.~random vectors with $\mathbb{E} \eta_i = 0$ and $\mathbb{E} \eta_i \eta_i^T = I_p$;
\item $f$ is a $C^1$ function in a neighborhood of $\tau = \lim_{p \to \infty} \text{trace}(\varSigma) / l $ and a $C^3$ function in a neighborhood of $0$.
\end{enumerate}
Under these assumptions, the kernel matrix $K$ can in probability be approximated consistently in operator norm, when $p$ and $n$ tend to $\infty$, by the kernel $\tilde{k}$, where
\begin{align*}
& \tilde{K} = \left( f(0) + f''(0) \frac{\text{trace}(\varSigma^2)}{2 l^2} \right) 1 1^T + f'(0) \frac{XX^T}{l} + v_p I_n,\\
& v_p = f\left( \frac{\text{trace}(\varSigma)}{l} \right) - f(0) - f'(0) \frac{\text{trace}(\varSigma)}{l}.
\end{align*}
In other words, with probability at least $1 - 4 n^2 e^{- n^{1/8} / (2 \tau)}$,
\begin{equation*}
    \| K - \tilde{K} \|_2 \le o(n^{- 1 / 16}). 
\end{equation*}
\end{lemma}

\begin{proof}
The proof is quite similar to Theorem 2.1 in \citet{el2010spectrum}, and the only difference is we change the bounded $4 + \epsilon$ absolute moment assumption to sub-gaussian assumption on data $x_i$, so obtain a faster convergence rate. 

First, using Taylor expansions, we can rewrite the kernel matrix $K$ sa
\begin{align*}
& f(x_i^T x_j / l) = f(0) + f'(0) \frac{x_i^T x_j}{l} + \frac{f''(0)}{2} \left( \frac{x_i^T x_j}{l} \right)^2 + \frac{f^{(3)}(\xi_{i,j})}{6} \left( \frac{x_i^T x_j}{l} \right)^3, i \ne j,\\
& f(\| x_i \|_2^2 / l) = f(\tau) + f'(\xi_{i,i}) \left( \frac{\| x_i \|_2^2}{l} - \tau \right), \text{on the diagonal,}
\end{align*}
in which $\tau = \text{trace}(\varSigma) / l$. Then we could deal with these terms separately.

For the second-order off-diagonal term, as the concentration inequality shows that 
\begin{equation}\label{eq:c1}
    \mathbb{P} \left( \max_{i,j} | \frac{x_i^T x_j}{l} - \delta_{i,j} \frac{\text{trace}(\varSigma)}{l} | \le t \right) \ge 1 - 2 n^2 e^{- \frac{l^2 t^2}{2 r_0(\varSigma^2)}},
\end{equation}
with Lemma \ref{lem_sg_se}, we can obtain that
\begin{equation}\label{eq:c2}
    \mathbb{P} \left( \max_{i \ne j} | \frac{(x_i^T x_j)^2}{l^2} - \mathbb{E} \frac{(x_i^T x_j)^2}{l^2} | \le t \right) \ge 1 - 2 n^2 e^{- \frac{l^4 t^2}{2 (162e)^2 r_0(\varSigma^4)}},
\end{equation}
in which 
\begin{equation*}
    \mathbb{E} \left( \frac{x_i^T x_j}{l} \right)^2 = \frac{1}{l^2} \mathbb{E} [x_i^T x_j x_j^T x_i] = \frac{1}{l^2} \mathbb{E} \text{trace} \{ x_j x_j^T x_i x_i^T \} = \frac{\text{trace}(\varSigma^2)}{l^2}. 
\end{equation*}
Denoting a new matrix $W$ as
\begin{equation*}
 W_{i,j} = \left\{
 \begin{aligned}
    & \frac{(x_i^T x_j)^2}{l^2}, i \ne j,\\
    & 0, i = j,
 \end{aligned}
 \right.
\end{equation*}
then considering $r_0(\varSigma^4) / l \le r_0(\varSigma) / l = \tau$ is bounded, choosing $t = n^{- 17 / 16}$, under Condition~\ref{cond:ntk_high-dim}, we have $l^3 n^{-17/8} \ge n^{21/32}$, so with probability at least $1 - 2 n^2 e^{- \frac{n^{1/8}}{2 (162e)^2 \tau^2}}$, we have
\begin{equation*}
    \| W - \frac{\text{trace}(\varSigma^2)}{l^2} (11^T - I_n) \|_2 \le \| W - \frac{\text{trace}(\varSigma^2)}{l^2} (11^T - I_n) \|_F \le \frac{1}{n^{1/16}}.
\end{equation*}

For the third-order off-diagonal term, as is mentioned in Eq.\eqref{eq:c1}, choosing $t = n^{- 1 / 4}$, with probability at least $1 - 2 n^2 e^{- \frac{n^{1/4}}{2 \tau }}$, we have
\begin{equation*}
    \max_{i \ne j} | \frac{x_i^T x_j}{l} | \le \frac{1}{n^{1/4}}.
\end{equation*}
Denote the matrix $E$ has entries $E_{i,j} = f^{(3)}(\xi_{i,j}) x_i^T x_j / l$ off the diagonal and $0$ on the diagonal, the third-order off-diagonal term can be upper bounded as
\begin{equation*}
    \| E \circ W \|_2 \le \max_{i,j} | E_{i,j} | \| W \|_2 \le o(n^{- 1 / 4}), 
\end{equation*}
where the last inequality is from the bounded norm of $W$.

For the diagonal term, still recalling Eq.\eqref{eq:c1}, while we have
\begin{equation*}
    \max_i | \frac{\| x_i \|_2^2}{l} - \tau | \le \frac{1}{n^{1 / 4}}, 
\end{equation*}
with probability at least $1 - 2 n^2 e^{- \frac{n^{1/4}}{2 \tau}}$, we can further get 
\begin{equation*}
    \max_i | f(\frac{\| x_i \|_2^2}{l}) - f(\tau) | \le o(n^{- 1 / 4}), 
\end{equation*}
which implies that 
\begin{equation*}
    \| \text{diag}[f(\frac{\| x_i \|_2^2}{l}), i = 1, \dots, n] - f(\tau) I_n \|_2 \le o(n^{- 1 / 4}).
\end{equation*}
Combing all the results above, we can obtain that 
\begin{equation*}
    \| K - \tilde{K} \|_2 \le o(n^{- 1 / 16}),
\end{equation*}
with probability at least $1 - 4 n^2 e^{- n^{1/8} / (2 \tau)}$.
\end{proof}

\begin{lemma}\label{lem:lip_gaussian}
If $x \sim \mathcal{N}(0, \sigma_x^2 I_d)$, and function $f: \mathbb{R}^d \to \mathbb{R}$ is L-Lipschitz, the random variable $f(x)$ is still sub-gaussian with parameter $L \sigma_x$. To be specific,
\begin{equation*}
    \mathbb{E} e^{ \lambda f(x)} \le e^{\frac{\lambda^2 L^2 \sigma_x^2}{2}}.
\end{equation*}
\end{lemma}

\begin{lemma}\label{lem:matrix_mean}
Assume $x \in \mathbb{R}^q$ is a $q$-dim sub-gaussian random vector with parameter $\sigma$, and $\mathbb{E}[x] = \mu$. Here are $n$ i.i.d.~samples $x_1, \dots, x_n$, which have the same distribution as $x$, then we can obtain that with probability at least $1 - 4 e^{- \sqrt{n}}$, 
\begin{equation*}
    \| \mathbb{E} xx^T - \frac{1}{n} \sum_{i=1}^n x_i x_i^T \|_2 \le \| \mathbb{E} zz^T \|_2 \max \{ \sqrt{\frac{\text{trace}(\mathbb{E} zz^T)}{n}}, \frac{\text{trace}(\mathbb{E} zz^T)}{n}, \frac{1}{n^{1/4}} \} + 2 \sqrt{2} \frac{\sigma \| \mu \|_2}{n^{1/4}}.
\end{equation*}
\end{lemma}

\begin{proof}
First, we denote $z = x - \mu$ is a ramdom vector with zero mean, correspondingly, there are $n$ i.i.d.~samples, $z_1, \dots, z_n$. Then we can obtain that
\begin{equation*}
    \mathbb{E} xx^T = \mathbb{E} (z + \mu) (z + \mu)^T = \mathbb{E} zz^T + \mu \mu^T,
\end{equation*}
and for the samples, 
\begin{equation*}
    \frac{1}{n} \sum_{i=1}^n x_i x_i^T = \frac{1}{n} \sum_{i=1}^n z_i z_i^T + \frac{2}{n} \sum_{i=1}^n \mu z_i^T + \mu \mu^T,
\end{equation*}
which implies that
\begin{align*}
\| \mathbb{E} xx^T - \frac{1}{n} \sum_{i=1}^n x_i x_i^T \|_2  &= \| \mathbb{E} zz^T + \mu \mu^T - \frac{1}{n} \sum_{i=1}^n z_i z_i^T - \mu \mu^T - \frac{2}{n} \sum_{i=1}^n \mu z_i^T \|_2 \\
&=  \| \mathbb{E} zz^T - \frac{1}{n} \sum_{i=1}^n z_i z_i^T - \frac{2}{n} \sum_{i=1}^n \mu z_i^T \|_2 \\
&\le \| \mathbb{E} z z^T - \frac{1}{n} \sum_{i=1}^n z_i z_i^T \|_2 + 2 \| \frac{1}{n} \sum_{i=1}^n \mu z_i^T \|_2\\
&= \| \mathbb{E} z z^T - \frac{1}{n} \sum_{i=1}^n z_i z_i^T \|_2 + 2 | \frac{1}{n} \sum_{i=1}^n \mu^T z_i |,
\end{align*}
where the inequality is from triangular inequality. So we can estimate the two terms respectively.

For the first term, as $z$ is $\sigma$-subgaussian random variable, by Theorem 9 in \citet{koltchinskii2017concentration}, with probability at least $1 - 2 e^{-t}$,
\begin{equation}\label{eq:lem1}
    \| \mathbb{E} z z^T - \frac{1}{n} \sum_{i=1}^n z_i z_i^T \|_2 \le  \| \mathbb{E} zz^T \|_2 \max \{ \sqrt{\frac{\text{trace}(\mathbb{E} zz^T)}{n}}, \frac{\text{trace}(\mathbb{E} zz^T)}{n}, \sqrt{\frac{t}{n}}, \frac{t}{n} \},
\end{equation}
And for the second term, by general concentration inequality, we can obtain that with probability at least $1 - 2 e^{- n t^2 / (2 \sigma^2 \| \mu \|_2^2 )}$,
\begin{equation}\label{eq:lem2}
    | \frac{1}{n} \sum_{i=1}^n z_i^T \mu | \le t.
\end{equation}
Choosing $t = \sqrt{n}$ in Eq.\eqref{eq:lem1} and $t = \sqrt{2} \sigma \| \mu \|_2 n^{- 1/4}$ in Eq.\eqref{eq:lem2}, with probability at least $1 - 4 e^{- \sqrt{n}}$,
\begin{align*}
\| \mathbb{E} xx^T - \frac{1}{n} \sum_{i=1}^n x_i x_i^T \|_2 &\le \| \mathbb{E} zz^T - \frac{1}{n} \sum_{i=1}^n z_i z_i^T \|_2 + 2 \| \frac{1}{n} \sum_{i=1}^n z_i^T \mu \|_2\\
& \le \| \mathbb{E} zz^T \|_2 \max \{ \sqrt{\frac{\text{trace}(\mathbb{E} zz^T)}{n}}, \frac{\text{trace}(\mathbb{E} zz^T)}{n}, \frac{1}{n^{1/4}} \} + 2 \sqrt{2} \frac{\sigma \| \mu \|_2}{n^{1/4}}.
\end{align*} 
\end{proof}

\end{document}